\documentclass[twoside]{article}

\usepackage[accepted]{aistats2025}

\usepackage[utf8]{inputenc}
\usepackage[english]{babel}
\usepackage{amsfonts}
\usepackage{dsfont}
\usepackage{amsthm}
\usepackage{amssymb}
\usepackage{hyperref}
\usepackage{booktabs}
\usepackage{multirow}
\usepackage{graphicx}
\usepackage[dvipsnames]{xcolor}
\usepackage{subcaption}
\usepackage[normalem]{ulem}
\usepackage{float}

\newtheorem{theorem}{Theorem}[section]
\newtheorem{definition}{Definition}[section]
\newtheorem{proposition}{Proposition}[section]
\newtheorem{corollary}{Corollary}[theorem]
\newtheorem{lemma}[theorem]{Lemma}

\renewcommand{\P}{\mathbb{P}}
\newcommand{\E}{\mathbb{E}}
\newcommand{\ind}{\mathds{1}}
\newcommand{\TV}{\mathrm{TV}}
\newcommand{\R}{\mathbb{R}}
\newcommand{\N}{\mathbb{N}}
\newcommand{\test}{\mathrm{test}}
\newcommand{\std}{\mathrm{std}}
\newcommand{\strat}{\mathrm{strat}}
\newcommand{\stratm}{\mathrm{strat}\text{-}}

\newcommand{\leb}{\mathrm{leb}}
\newcommand{\relleb}{\mathrm{leb}}

\newcommand{\rank}{\mathrm{rank}}
\newcommand{\Unif}{\mathrm{Unif}}
\newcommand{\dif}{\mathrm{d}}

\DeclareMathOperator*{\argmax}{arg\,max}

\newcommand{\DeltaSet}{\mathcal{D}}

\begin{document}

\twocolumn[

\aistatstitle{Strategic Conformal Prediction}

\aistatsauthor{
    Daniel Csillag
    \And
    Claudio José Struchiner
    \And
    Guilherme Tegoni Goedert
}

\aistatsaddress{
    \texttt{daniel.csillag@fgv.br} \\ School of Applied Mathematics \\ Fundação Getulio Vargas
    \And
    \texttt{claudio.struchiner@fgv.br} \\ School of Applied Mathematics \\ Fundação Getulio Vargas
    \And
    \texttt{guilherme.goedert@fgv.br} \\ School of Applied Mathematics \\ Fundação Getulio Vargas
} ]

\begin{abstract}
    When a machine learning model is deployed, its predictions can alter its environment, as better informed agents strategize to suit their own interests.
    With such alterations in mind, existing approaches to uncertainty quantification break.
    In this work we propose a new framework, Strategic Conformal Prediction, which is capable of robust uncertainty quantification in such a setting.
    Strategic Conformal Prediction
    is backed by a series of theoretical guarantees spanning marginal coverage, training-conditional coverage, tightness and robustness to misspecification that hold in a distribution-free manner.
    Experimental analysis further validates our method,
    showing its remarkable effectiveness in face of arbitrary strategic alterations, whereas other methods break.
\end{abstract}

\section{Introduction}

Machine learning models are widely used to manage risks associated with the actions of individuals.
However, these individuals can leverage their knowledge about the model to strategize and align the model's predictions to their own interests, causing a distribution shift.
This is typically referred to as the \emph{strategic} (or \emph{performative}) setting, and is ubiquitous in machine learning practice, such as in credit scoring, autonomous driving, spam mail classification, and chat-based interactions.

Strategic machine learning is a topic with deep connections to not only statistics, but also game theory, economics and causality, and has been the topic of significant prior research (e.g., \cite{perdomo-performative-prediction,nir-rosenfeld-regularization,levanon-practical,performative-sgd}; see also the survey of \cite{performative-survey}).

However, to the best of our knowledge, there has been no prior work on uncertainty quantification for the strategic setting.
This is essential to the safe use of machine learning methods, and especially so in safety-critical systems, where strategic issues can be found aplenty.

In this paper, we present the first method for uncertainty quantification in the strategic setting, based on the ideas of split conformal prediction \cite{vovk-conformal,split-conformal}.
By considering the underlying ML model and conformity score fixed and leveraging the model-agnostic nature of conformal prediction, our method is able to deal with strategic alterations on arbitrary tasks (e.g.
regression, classification, structured prediction)
while achieving a strategic marginal guarantee as well as a corresponding training-conditional guarantee~\cite{group-conditional-conformal,bian-training-conditional-conformal}.

We also present extensions of our method akin to the ones of standard conformal prediction, in particular ensuring group-conditional and label-conditional coverage.
Furthermore, by leveraging change-of-measure inequalities and select modeling assumptions, we prove a series of additional theoretical results on the robustness and efficiency/tightness of our method. Finally, we explore our method empirically in experiments on various datasets, validating the efficacy of our proposed solution.

Our main contributions are:
\vspace{-0.4cm}
\begin{enumerate}
    \item A new method for uncertainty quantification based on conformal prediction that is valid under strategic alterations, with guarantees akin to type-I error control and the usual validity property of conformal prediction. To our knowledge, this is the first instance of uncertainty quantification for the strategic setting, and is directly applicable in diverse settings including regression, classification, structured prediction, etc.
    \item Extensions of our proposed method for group-conditional and label-conditional coverage, showcasing that our method is versatile and can be extended similarly to standard conformal prediction methods;
    \item We develop a rich supporting theory for our method and framework, including guarantees of precision under model misspecification, as the strategic alterations converge (e.g. due to an equilibrium state) and on the tightness of our predictive intervals, further justifying our framework.
\end{enumerate}

\section{Framework and Method}

Suppose that we have a random sample $(Z_i)_{i=1}^n = (X_i, Y_i)_{i=1}^n \subset \mathcal{Z} = \mathcal{X} \times \mathcal{Y}$, and consider the problem of learning to predict $Y_i$ using $X_i$.

In standard conformal prediction, we seek to find a set-predictor $C_t : \mathcal{X} \to 2^{\mathcal{Y}}$ (indexed by a threshold $t \in \R$) such that, for any significance level $\alpha$, there is a corresponding threshold $t^\std_\alpha$ such that
\[ \P[Y \in C_{t^\std_\alpha} (X)] \geq 1 - \alpha, \]
with $C_t$, per split conformal prediction, taking the form
\begin{equation}\label{eq:conformal-set}
    C_t (X) = \left\{ y \in \mathcal{Y} : s(x, y) \leq t \right\}
\end{equation}
where $s : \mathcal{X} \times \mathcal{Y} \to \R$ is an arbitrary function (typically referred to as the `conformity score' in the conformal prediction literature) and $t^\std_\alpha$ is given by
\begin{align*}
    t^\std_\alpha = \inf \biggl\{ t \in \R : \frac{1}{n+1} \sum_{i=1}^n \ind[s(X_i, Y_i) > t] \quad\ %
    \\ + \frac{1}{n+1} \leq \alpha \biggr\}.
\end{align*}
In the strategic setting, we must consider that the distribution of the covariates $X$ may suffer alterations in response to our model, as better informed agents strategize to suit their own interests. Such alterations can be denoted as stochastic functions $\Delta \in \DeltaSet$, with $\Delta : \mathcal{X} \to \mathcal{X}$. We thus seek guarantees of the form
\begin{equation}\label{eq:goal}
    \P[\forall \Delta \in \DeltaSet,\ Y \in C_\alpha (\Delta (X))] \geq 1 - \alpha,
\end{equation}
where the covariates $X$ are altered by these functions $\Delta : \mathcal{X} \to \mathcal{X}$, while the response $Y$ does not change.\footnote{In the scope of this paper, we consider that the outcome $Y$ does not change. Though restrictive, this is a typical assumption in the strategic ML literature, and not too unreasonable, though perhaps a bit too conservative. Our setup can also be extended to the case where the $\Delta$ alters the outcomes $Y$, but this has additional subtleties and we leave to future work.}
By considering the probability uniform over $\Delta \in \DeltaSet$ we can obtain simultaneous validity with regards to multiple plausible alterations, in particular including the case of no alteration ($\Delta(x) = x$).

These alterations will typically be in the efforts of making the algorithm confident of the some outcome, which corresponds to ensuring that some region $\Omega \subset \mathcal{Y}$ of the outcome space is \emph{not} contained in the predicted $C_\alpha (\Delta (x))$ (as it corresponds to the set of labels that, with high probability, \emph{may} be the true one). I.e., by considering the construction in Equation~\ref{eq:conformal-set}, $\Delta$ seeks to maximize $\inf_{y \in \Omega} s(\Delta (X), y)$. Indeed, one can even think of the many $\Delta \in \DeltaSet$ as ranging from no alteration to an `optimal' alteration, interpolating in terms of the cost of performing these alterations.

Crucially,
these alteration functions $\Delta$ will generally depend arbitrarily on the conformity score $s$ and on the calibrated set predictor $C_t$ (which is fully specified by threshold $t$ itself).

For tractability purposes, however, we assume that the $\Delta$ do not depend on the threshold $t$ (but can depend freely on the conformity score $s$).
Thanks to the for-all over $\Delta \in \DeltaSet$, this is not too harsh of an assumption: if we have a $\Delta^\star_t$ as the ``true'' $\Delta$ as a function of the learned threshold $t$, we can consider $\DeltaSet = \{\Delta^*_{t'} : t' \in \R\}$, and our guarantee will certainly hold for $\Delta^*_t$.
Moreover, this sort of construction is robust: we prove that as long as the choice of $\DeltaSet$ is able to approximately cover the range of the ``true'' $\Delta^\star_t$, we still get strong coverage guarantees (see Proposition~\ref{thm:robustness-2}).

With this in mind, our goal becomes to attain some set-predictor $C_\alpha : \mathcal{X} \to 2^\mathcal{Y}$ satisfying Equation~\ref{eq:goal}, where the $\Delta \in \DeltaSet$ depend freely on $s$ but are independent from $C_\alpha$.
We show that this can be achieved simply by modifying the computation of $t^\std_\alpha$, incorporating suprema over the $\Delta \in \DeltaSet$:
\begin{align*}
    t^\strat_\alpha = \inf \Biggl\{ t \in \R &: \frac{1}{n+1} \sum_{i=1}^n \ind\left[ \sup_{\Delta \in \DeltaSet} s(\Delta(X_i), Y_i) > t \right] \\ & \qquad\qquad\qquad\qquad + \frac{1}{n+1} \leq \alpha \Biggr\}.
\end{align*}
\begin{theorem}\label{thm:conformal-marginal-guarantee}
    Let $Z_1, \ldots, Z_n, Z^\test$ be $n + 1$ exchangeable random variables in $\mathcal{X} \times \mathcal{Y}$. Let $t^\strat_\alpha$ be as above. Then, for any $\alpha \in (0, 1)$,
    \[ \P\left[ \forall \Delta \in \mathcal{D}, \ Y^\test \in C_{t^\strat_\alpha} (\Delta (X^\test)) \right] \geq 1 - \alpha. \]
\end{theorem}
We can also make a training-conditional PAC guarantee for the batch setting, under an i.i.d. assumption:
\begin{theorem}\label{thm:training-conditional-guarantee}
    For any $\delta \in (0, 1)$, with probability of at least $1 - \delta$ over the draw of i.i.d. random variables $Z_1, \ldots, Z_n$, it holds that
    \begin{align*}
        & \P\left[ \forall \Delta \in \DeltaSet, \ Y^\test \in C_{t^\strat_\alpha} (\Delta (X^\test)) \ \middle|\ t^\strat_\alpha \right]
        \\ &\quad \geq 1 - \alpha - \sqrt{\frac{\log 1/\delta}{2n}}.
    \end{align*}
\end{theorem}

\subsection{Construction of the $\Delta$}

By means of Theorems~\ref{thm:conformal-marginal-guarantee} and \ref{thm:training-conditional-guarantee}, once we have our set $\DeltaSet$ of alteration functions $\Delta$, we are able to provably control coverage error in face of these specified alterations.
However, modeling these $\Delta$s can be nontrivial.

In this section, we outline some practical and pragmatic ways of constructing the $\Delta$s to be used, and exhibit how they relate to more usual constructions in the strategic classification and economics literature.

\subsubsection{From an utility-cost decomposition}\label{sec:delta-utility-cost}

In the economic and game-theoric literature it is typical to model actions as the solution of an optimization problem expressed in terms of utility and cost functions.
An utility function $u : \mathcal{X} \to \R$ in which higher values of $u(X)$ correspond to more `desired/rewarding' values of $X \in \mathcal{X}$, and a cost function $c : \mathcal{X} \times \mathcal{X} \to \R_{\geq 0}$ (which should satisfy $c(X, X) = 0$ for all $X \in \mathcal{X}$) in which a cost $c(X, X')$ corresponds to how `costly' it would be for an individual to alter themselves from covariates $X$ to covariates $X'$ (higher values indicate higher cost).

The effective actions are then modelled as the maximization\footnote{Assuming that there is an unique maximizer. If there isn't one, we can consider a random tie-breaking rule, making use of the fact that we allow for stochastic $\Delta$s.} of utility `regularized' by the cost:
\[ \Delta^{(u,c)} (X) = \argmax_{X' \in \mathcal{X}}\, \bigl[ u(X') - c(X, X') \bigr]. \]
So one can directly recover the usual case in the strategic classification and performative optimization literature by considering a singleton $\DeltaSet$ containing just $\Delta^{(u,c)}$ for some fixed known choices of $u$ and $c$.

However, it is advisable to consider not just the `fully rational' alteration $\Delta^{(u,c)}$ but rather a whole range of potential alterations ranging from no alteration at all to the fully rational one.
This can be pragmatically done by altering the cost function:
\begin{align*}
    \Delta^{(u,\lambda^{-1} c)} (X) = \argmax_{X' \in \mathcal{X}}\, \bigl[ u(X') - \lambda^{-1} c(X, X') \bigr],
    \\ \text{for } \lambda \in [0, 1],
\end{align*}
with the corner case $\lambda = 0$ (being a bit lax with infinities) corresponding to alterations that have zero cost:
\begin{align*}
    \Delta^{(u,0^{-1} c)} (X) &= \argmax_{X' \in \mathcal{X}}\, \bigl[ u(X') - \infty \ind[c(X, X') = 0] \bigr]
    \\ &= \argmax_{\substack{X' \in \mathcal{X} \\ c(X, X') = 0}}\, u(X').
\end{align*}
Taking $\lambda = 1$ recovers $\Delta^{(u,c)}$ exactly; meanwhile $\lambda = 0$ (considering that $0^{-1} = +\infty$) gives us zero-cost alterations, which typically amounts to the identity alteration (as the cost $c(X, X')$ is usually taken to be positive unless $X = X'$). Any $\lambda$ between the two effectively interpolates between them.

It is also worth noting that the Strategic Conformal Prediction setting admits a particularly natural form of utility: since the individual's goal is to fool the predicted set into not containing some particular subset of outcomes $\Omega \subset \mathcal{Y}$ -- i.e., maximizing $\inf_{y \in \Omega} s(\Delta (X), y)$, we can consider exactly that as our utility:
\[ u(X') = \inf_{y \in \Omega} s(X', y). \]
Nevertheless, our procedure is more general and can work with arbitrary utilities (e.g., ones that take into account the individuals' perceptions), not just this fully-rational one.

\subsubsection{As a stochastic simulator}\label{sec:delta-stochastic-simulator}

Another way of constructing the $\Delta$s is to leverage stochastic simulators.
There is a wide literature on modeling populations with agent-based models and stochastic simulators, ranging from more parsimonious models (e.g., \cite{parsimonious-agent-based-model-1,parsimonious-agent-based-model-2}) to more complex ones (e.g., \cite{complex-agent-based-model-1,complex-agent-based-model-2}), sometimes even leveraging learning procedures underneath.

Let $M : \mathcal{X} \to \mathcal{X}$ represent one call to an underlying stochastic simulator, which gives us the alteration proposed by it.
Note that $M$ is generally not idempotent: if we were to compute $M(M(X))$, this would generally not be the same as just $M(X)$ (even in lack of stochasticity).
This leads us to define a sequence of altered covariates starting from the initial covariates $X$:
\[ X^{(0)} = X, \qquad\qquad X^{(k+1)} = M(X^{(k)}). \]
As we consider $k \to +\infty$, $X^{(k)}$ corresponds to a ``higher level of effort'' and thus a ``higher cost'' than the previous ones. This naturally leads to a set $\DeltaSet^{(M)}$ of $\Delta$s:
\begin{align*}
    \DeltaSet^{(M,k_{\max})} = \{ (X \mapsto  X^{(k)}) : k = 0, \ldots, k_{\max} \}, \qquad
    \\ \text{for } k_{\max} \in \N.
\end{align*}
It is necessary to consider the sequence only up to $k_{\max}$ for computability purposes -- if we consider $M$ to be a black-box stochastic function, then we would not be able to compute suprema over all $k \in \N$ (as required for our Strategic Conformal calibration).

One particularly interesting instance of this construction is to use a generative foundation model with agentic capabilities (such as recent LLMs) to simulate actions of the individuals.
Such models can not only be prompted to simulate particular characteristics of subpopulations but also \emph{calibrated} to do so while better aligning with the actual populations.

\subsubsection{Via iterative random search}\label{sec:delta-proximal}

An interesting variant of the stochastic simulator setting of Section~\ref{sec:delta-stochastic-simulator} is that in which we formulate the stochastic step $M$ as a random search in a local neighborhood. In particular, consider that we have some utility function $u : \mathcal{X} \to \R$ and a cost $c : \mathcal{X} \times \mathcal{X} \to \R_{\geq 0}$ that we further assume to be a metric over $\mathcal{X}$.

Instead of considering the regularized utility maximization problem presented in Section~\ref{sec:delta-utility-cost}, we can consider an iterative procedure where in each step we evaluate the utility on a few points of the neighborhood of the current point (as determined by the metric $c$) and take the best one. I.e.,
\begin{align*}
    X^{(0)} = X, \qquad X^{(k+1)} &= \argmax_{j = 1, \ldots, m} u(X + \delta^{(k)}_j),
    \\
    \delta^{(k)}_1, \ldots, \delta^{(k)}_m &\overset{\mathrm{iid}}{\sim} \mathcal{N}_c(0, \sigma^2 I) \quad \forall k \in \N,
\end{align*}
for some $\sigma$ hyperparameter (and anisotropy given by the choice of metric $c$).\footnote{This notation suggests the common case that $\mathcal{X} = \R^d$, but the idea is more general; a normal distribution/law can be defined over any Banach space (see, e.g., \cite{probability-in-banach-spaces}), and everything still follows.}
Intuitively, this corresponds to each agent considering $m$ different actions and opting for the one they find best.
It is also reasonable to include an additional $\delta^{(k)}_0 = 0$ to allow the agent to do nothing at each step.

This approach can be seen as a middle term between the typical utility-cost formulations (Section~\ref{sec:delta-utility-cost}) and stochastic simulator formulations (Section~\ref{sec:delta-stochastic-simulator}) -- it allows one to leverage the pragmatic nature of the decomposition of actions into utility and cost while being much easier to implement than the optimization-based approach of Section~\ref{sec:delta-utility-cost} (which would require solving many optimization problems for every data point in the calibration set, and these problems could be quite tricky [e.g., highly nonconvex]).

\subsection{Extensions of Strategic Conformal Prediction}

In this section, we present extensions and variants of our base method for group-conditional and label-conditional coverage~\cite{group-conditional-conformal,label-conditional-conformal},
demonstrating that the same strategy used to adapt standard split conformal prediction to the strategic setting can also be applied to other conformal prediction algorithms.

For the sake of brevity, we present here only the extension for group-conditional coverage; the label-conditional case can be found in the supplementary material.

For group-conditional guarantees, suppose that we have some (possibly intersecting) subsets $G_1, \ldots, G_l \subset \mathcal{X}$ (the groups) such that $\bigcup_{i=1}^l G_j = \mathcal{X}$. We want to find some $C^{\stratm(G_*)}_\alpha$ such that
\begin{align*}
    \P[\forall \Delta \in \DeltaSet,\ Y \in C^{\stratm(G_*)}_\alpha (\Delta (X)) | X \in G_j] \geq 1 - \alpha,
    \\ \text{for all } j = 1, \ldots, l. \quad
\end{align*}
Similar to standard Group-Conditional Conformal Prediction~\cite{group-conditional-conformal}, it suffices to do our strategic calibration separately over each group -- thus obtaining a threshold for each group -- and then using the threshold corresponding to the group the test covariate is in (and, in the case of overlap between groups, taking the highest threshold):
\[ C^{\stratm(G_*)}_\alpha (X) = \left\{ y \in \mathcal{Y} : s(X, y) \leq \max_{j; X \in G_j} t^{\stratm(G_j)}_\alpha \right\} \]
where
\begin{align*}
    t^{\stratm(G_j)}_\alpha = \inf \Biggl\{ t \in \R : &\sum_{\substack{i=1 \\ X_i \in G_j}}^n \ind\left[ \sup_{\Delta \in \DeltaSet} s(\Delta(X_i), Y_i) > t \right]
        \\ &\qquad\qquad\quad + 1 \leq \alpha (n_{G_j}+1) \Biggr\}
\end{align*}
and $n_{G_j} = \lvert \{i = 1, \ldots, n : X_i \in G_j\} \rvert$.
\begin{theorem}\label{thm:conformal-marginal-guarantee-groupconditional}
    Let $Z_1, \ldots, Z_n, Z^\test$ be $n + 1$ exchangeable random variables in $\mathcal{X} \times \mathcal{Y}$. Let $C^{\stratm(G_j)}_\alpha$ be as above (for $j = 1, \ldots, l$). Then, for any $\alpha \in (0, 1)$,
    \begin{align*}
        \P\left[ \forall \Delta \in \DeltaSet, \ Y^\test \in C^{\stratm(G_j)}_{\alpha} (\Delta (X^\test)) \middle| X^\test \in G_j \right] \\ \geq 1 - \alpha, \quad \forall j = 1, \ldots, l. \quad
    \end{align*}
\end{theorem}

\subsection{Additional Guarantees under Modeling Assumptions}

In this section, we will present additional theoretical properties enjoyed by our procedure.
Throughout, these guarantees will either hold under specific modeling assumptions or involve some quantity that is generally inestimable in practice, but enlightening from a theoretical point of view.

\subsubsection{Under misspecification of the $\Delta$s}

Key to our procedure is the choice of the alteration functions $\Delta \subset \DeltaSet$.
As these are specified by the practitioner, there is a significant possibility of misspecification.
This is further accentuated by our constraint that the $\Delta$ do not depend on the learned threshold $t$.

Fortunately, our model is reasonably robust to such violations.
Suppose that there is a ``true'' alteration function $\Delta^\star$ -- i.e., what we are really interested in is coverage when the covariates have been modified by $\Delta^\star$.
In particular, this $\Delta^\star$ can be a function not only of $s$ (as the $\Delta \in \DeltaSet$ can) but also on the learned threshold $t$.
First, we prove a bound on the difference between the coverage of a set predictor with the true alteration function $\Delta^*$ versus with any of our $\Delta$s in $\DeltaSet$ in terms of how much the two deviate on average over $X$ (per total variation):

\begin{proposition}\label{thm:robustness-1}
    Let $C : \mathcal{X} \to 2^\mathcal{Y}$ be a set predictor.
    For any $\Delta \in \DeltaSet$,
    \begin{align*}
        &\left\lvert \P_{X,Y}[Y \in C_\alpha (\Delta (X))] - \P_{X,Y}[Y \in C_\alpha (\Delta^\star (X))] \right\rvert
        \\ &\quad \leq \E_X\left[\TV( \Delta (X) | X \ \Vert\ \Delta^\star (X) | X )\right].
    \end{align*}
\end{proposition}

This shows that as long as the two $\Delta$s are sufficiently aligned in high-probability regions of $X$, there will be little difference in the coverage of the two, showcasing a certain robustness to misspecification of our model.

Proposition~\ref{thm:robustness-1} bounds the difference in coverage between $\Delta^*$ and some particular $\Delta \in \DeltaSet$. However, we are more interested in the \emph{set} $\DeltaSet$, not just a single one of its elements. The following proposition considers just this and, for the specific case of our method, bounding the coverage with $\Delta^*$ in terms of the infimum of the average deviation terms from Proposition~\ref{thm:robustness-1}.

\begin{proposition}\label{thm:robustness-2}
    Let $C : \mathcal{X} \to 2^\mathcal{Y}$ be a set predictor satisfying Equation~\ref{eq:goal}. Then
    \begin{align*}
        &\P_{X,Y}[Y \in C_\alpha (\Delta^\star (X))]
        \\ &\quad \geq 1 - \alpha - \inf_{\Delta \in \DeltaSet} \E_X\left[\TV( \Delta (X) | X \ \Vert\ \Delta^\star (X) | X )\right].
    \end{align*}
\end{proposition}

Note also that the randomness of the $\Delta$ is absolutely essential for these results to be useful, and suggests that it is even advisable to \emph{always} have them be somewhat stochastic.
If the $\Delta \in \DeltaSet$ are fully deterministic, then the total variation distances will equal 1 if there is even a slight misspecification. However, if there is randomness, then even under significant misspecification the total variation distances will be well below 1, meaning a smaller deviation from the true alterations $\Delta^*$ and more robustness.

\subsubsection{As the $\Delta$ converge}

Consider that we have a sequence of alteration functions $\Delta_k : \mathcal{X} \to \mathcal{X}$, $k = 0, 1, \ldots$, and that $\DeltaSet = \{\Delta_k : k = 0, \ldots, k_{\max}\}$ for some $k_{\max} \in \N$ (e.g., as is the case in the stochastic simulation and iterative random search constructions in Sections \ref{sec:delta-stochastic-simulator} and \ref{sec:delta-proximal}).
As mentioned in the sections prior, this $k_{\max}$ is necessary from a computational point-of-view -- for all we know, the underlying stochastic simulators process may change behavior arbitrarily and we would be unable to compute the suprema in the calibration procedure for the $t^\strat_\alpha$.
But what happens for $\Delta_k$ beyond this $k_{\max}$?

We show that if this sequence of $\Delta_k$s converges to some $\overline{\Delta}$, then as the $k_{\max}$ used for calibration increases, the coverage under $\overline{\Delta}$ improves, achieving the nominal coverage at the limit.

\begin{proposition}\label{thm:as-delta-converges}
    Let $(\Delta_k)_{k=0}^\infty$ be a sequence of alteration functions $\Delta_k : \mathcal{X} \to \mathcal{X}$, such that $\lim_{k \to \infty} \Delta_k (X) = \overline{\Delta}(X)$ according to the total variation metric.
    Then, for any $k_{\max} > 0$, consider the set predictor $C_{t^\strat_\alpha}$ generated by $\DeltaSet = \{ \Delta_0, \ldots, \Delta_{k_{\max}} \}$. There exists some $\epsilon_{k_{\max}}$ such that,
    \[ \P_{X,Y}[Y \in C_{t^\strat_\alpha} (\overline{\Delta} (X))] \geq 1 - \alpha - \epsilon, \]
    and $\lim_{k_{\max} \to \infty} \epsilon_{k_{\max}} = 0$.
\end{proposition}

In fact, as $\epsilon$ decreases, it can become smaller than the gap between the actual coverage $\P[Y \in C_\alpha (\overline{\Delta}(X))]$ and the nominal coverage $1 - \alpha$ -- at which point we actually attain the nominal coverage even for $\Delta_k$ with $k \gg k_{\max}$.
Indeed, this remarkable behavior does happen in practice, as can be seen in Figure~\ref{fig:curves}.

\subsubsection{Tightness of the intervals under stability assumptions}

\begin{figure*}[h]
    \centering
    \includegraphics[width=.9\textwidth]{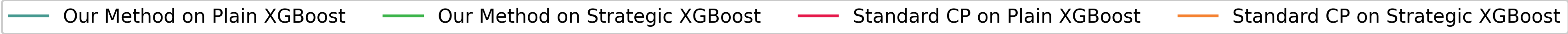}

    \begin{subfigure}[t]{0.48\textwidth}
        \centering
        \includegraphics[width=\columnwidth]{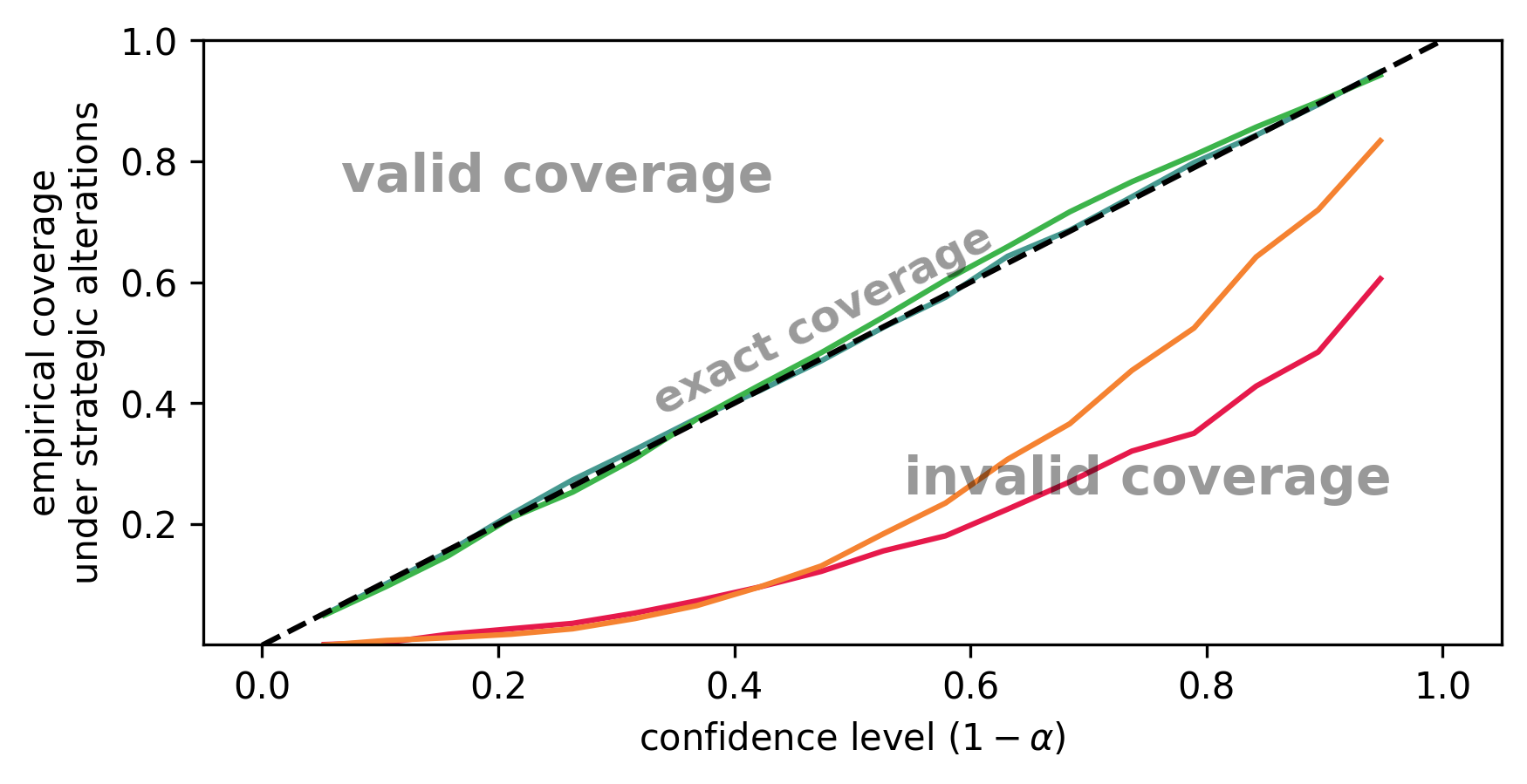}
        \caption{Coverage of standard CP \& our method under strategic alterations over varying values of $\alpha$.}
        \label{fig:coverages}
    \end{subfigure}\ \quad
    \begin{subfigure}[t]{0.48\textwidth}
        \centering
        \includegraphics[width=\columnwidth]{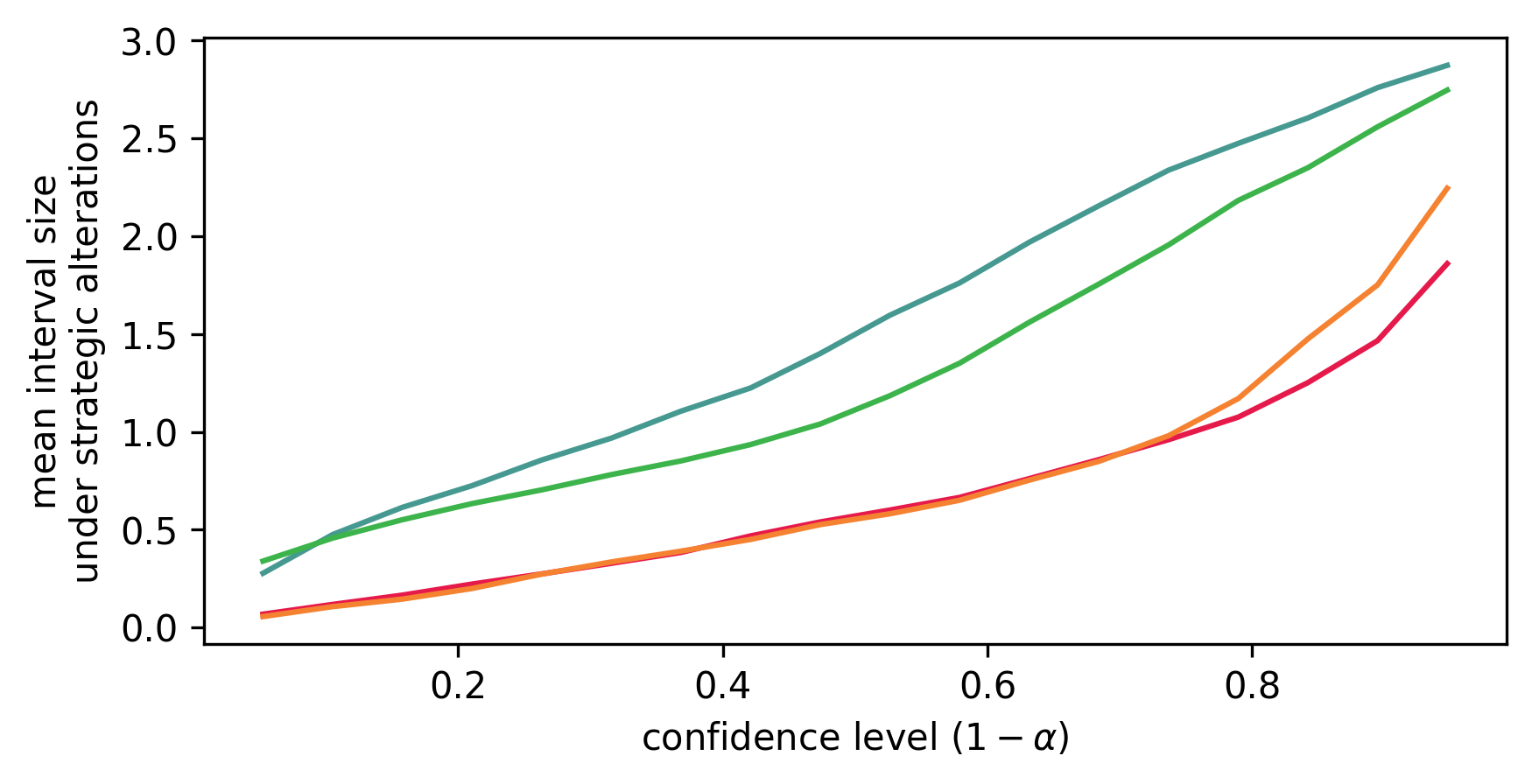}
        \caption{Average sizes of the predictive sets of standard CP \& our method under strategic alterations.}
        \label{fig:sizes}
    \end{subfigure}
    \caption{\textbf{Coverage and interval sizes for our method versus standard CP.} Figures~\ref{fig:coverages} and \ref{fig:sizes} evaluate the coverage and interval sizes under strategic alterations over varying confidence levels $1-\alpha$ on the \texttt{academic-dropout} dataset. Note that our method attains exact coverage, while the baselines have (very) invalid coverage. This comes at a cost of larger predictive sets, but using a better model (Strategic XGBoost) reduces this. For more details and more datasets, see the supplementary material.}
    \label{fig:coverages-and-sizes}
    \vspace{-1em}
\end{figure*}

So far all of our results have been only about the coverage of our predictions in the face of strategic alterations, $\P[\forall \Delta \in \DeltaSet,\ Y \in C_{t^\strat_\alpha} (\Delta (X))]$ (and variants).
However, just as important are results about \emph{how large} the sets outputted by $C_\alpha$ will be; after all, one trivial way to attain all the coverage guarantees would have been to simply produce the whole output space $\mathcal{Y}$, but that is hardly useful.

On one hand, we desire short predictive sets $C_{t^\strat_\alpha}(X)$. However, if the underlying model is itself vulnerable to strategic perturbations, then these sets will necessarily have to increase accordingly; and of course, if there is uncertainty about the outcome itself, this also adds to the size of the sets.
But the question remains: if neither of these are the case, then are our sets tight?

The answer is, fortunately, yes, as the two following propositions show.
Such matters fundamentally depend on the choice of conformity score; for the sake of concreteness, we consider here only the case of the simple conformity score $s(x, y) = \lvert \mu(x) - y \rvert$, where $\mu : \mathcal{X} \to \mathcal{Y}$ for $\mathcal{Y} = \R$ is a base regression model, but the same argument can be used for other conformity scores.
Additionally, focusing on such a conformity score and the regression setting gives us a natural way to quantify the size of the predicted sets, in the form of their Lebesgue measure.

In what follows, $\leb(A)$ corresponds to the Lebesgue measure of the set $A$, $\relleb_A(B) = \leb(B) / \leb(A)$ and $(W)_{[i]}$ corresponds to the $i$-th order statistic of the random variable $W$.
First, we show that if the model is already robust against strategic perturbations, then the interval size will be approximately equal that of standard conformal prediction:

\begin{proposition}\label{thm:tightness-1}
    Let $\mu : \mathcal{X} \to \R$ be a possibly-stochastic function representing a base model.
    Consider the predictive sets $C_{t^\strat_\alpha}$ with conformity score $s : \mathcal{X} \times \R \to \R$ given by $s(x, y) = \lvert \mu(x) - y \rvert$ with $\mathcal{Y}$ being a compact subset of $\mathcal{Y}$ and alteration functions $\Delta \in \DeltaSet$, and contrast it with the corresponding standard conformal predictive sets $C_{t^\std_\alpha}$. Then, almost surely over $X$,
    \begin{align*}
        &\E_{C_{t^\strat_\alpha}}[\relleb_{\mathcal{Y}}\left(C_{t^\strat_\alpha}(X)\right)] \leq \E_{C_{t^\std_\alpha}}[\relleb_{\mathcal{Y}}\left(C_{t^\std_\alpha}(X)\right)]
        \\ &+ 2\, \TV\left(\Bigl( s(X, Y) \Bigr)_{[j_{n,\alpha}]} \ \middle\|\ \Bigl( \sup_{\Delta \in \DeltaSet} s(\Delta (X), Y) \Bigr)_{[j_{n,\alpha}]} \right),
    \end{align*}
    for $j_{n,\alpha} = \lceil (1 - \alpha) (1 + n) \rceil$.
\end{proposition}

And, in turn, we can show that if there is little noise in the predictions of the underlying model $\mu$, then the produced intervals are near singletons.

\begin{corollary}\label{thm:tightness-2}
    In the same setting as Proposition~\ref{thm:tightness-1}, further assume that $\mu(X) = Y + \epsilon$ for some random $\epsilon$ (arbitrarily dependent on $X$ and $Y$), with $\epsilon \in [-M, M]$ almost surely. Then, almost surely over $X$,
    \begin{align*}
        &\E_{C_{t^\strat_\alpha}}[\leb_{\mathcal{Y}}\left(C_{t^\strat_\alpha}(X)\right)] \leq 2 M
        \\ &\qquad + 2\, \TV\left((\epsilon)_{[j_{n,\alpha}]} \ \middle\|\ \Bigl( \sup_{\Delta \in \DeltaSet} s(\Delta (X), Y) \Bigr)_{[j_{n,\alpha}]} \right),
    \end{align*}
    for $j_{n,\alpha} = \lceil (1 - \alpha) (1 + n) \rceil$.
\end{corollary}

Note that as $M \to 0$ and the strategic alterations do not cause the scores to deviate much from the $\epsilon$, the size of the strategic conformal predictive sets approaches zero.
\vspace{-0.2cm}
\section{Experiments}
\vspace{-0.2cm}

In order to fully assess our method's efficacy and further verify our theoretical guarantees, we conduct an empirical investigation of our method.

\begin{figure*}[h]
    \centering
    \begin{subfigure}[t]{0.48\textwidth}
        \centering
        \includegraphics[width=\columnwidth]{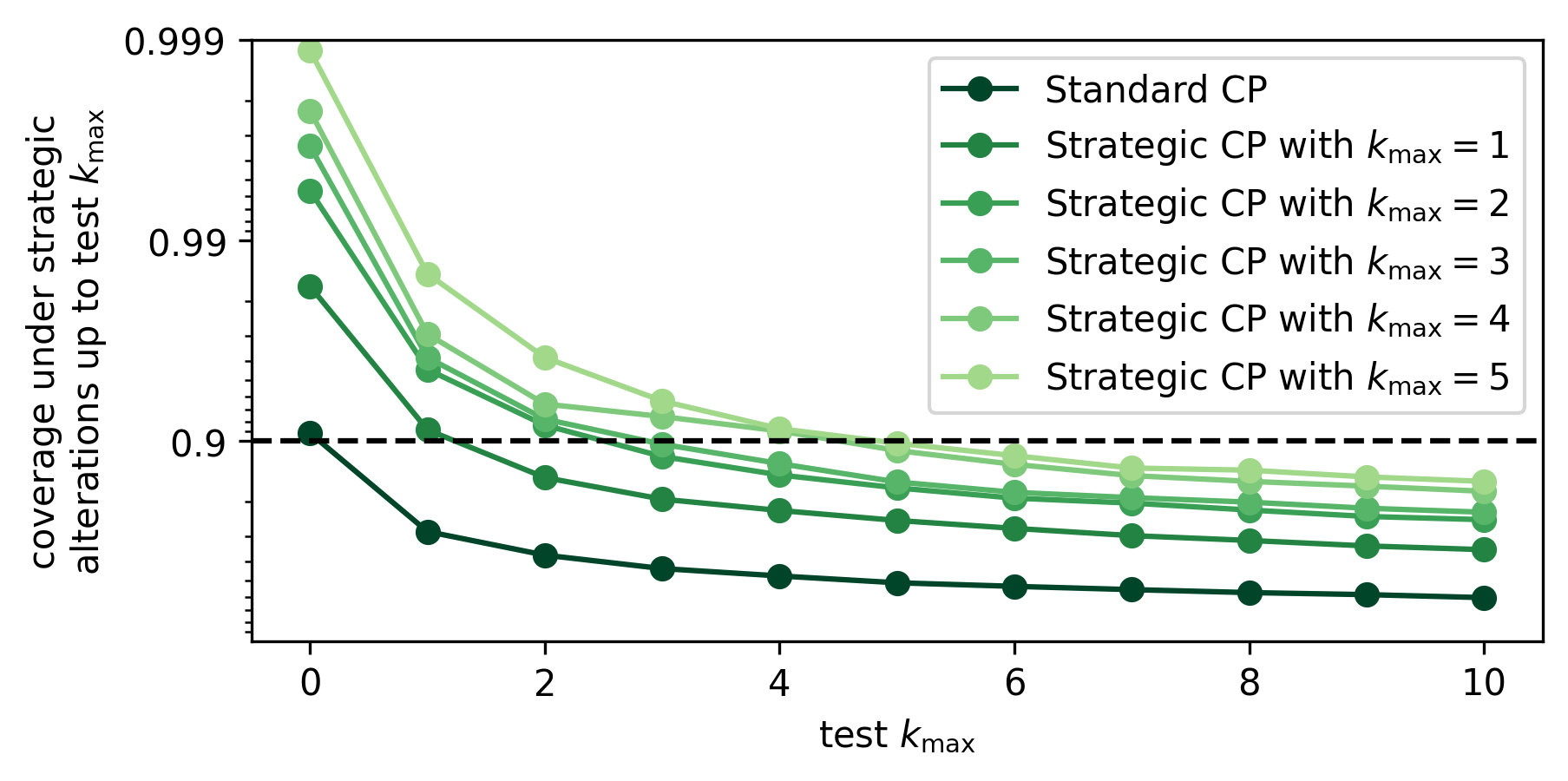}
        \caption{Strategic Conformal Prediction on top of a Plain \linebreak XGBoost Model}
    \end{subfigure}\ \quad
    \begin{subfigure}[t]{0.48\textwidth}
        \centering
        \includegraphics[width=\columnwidth]{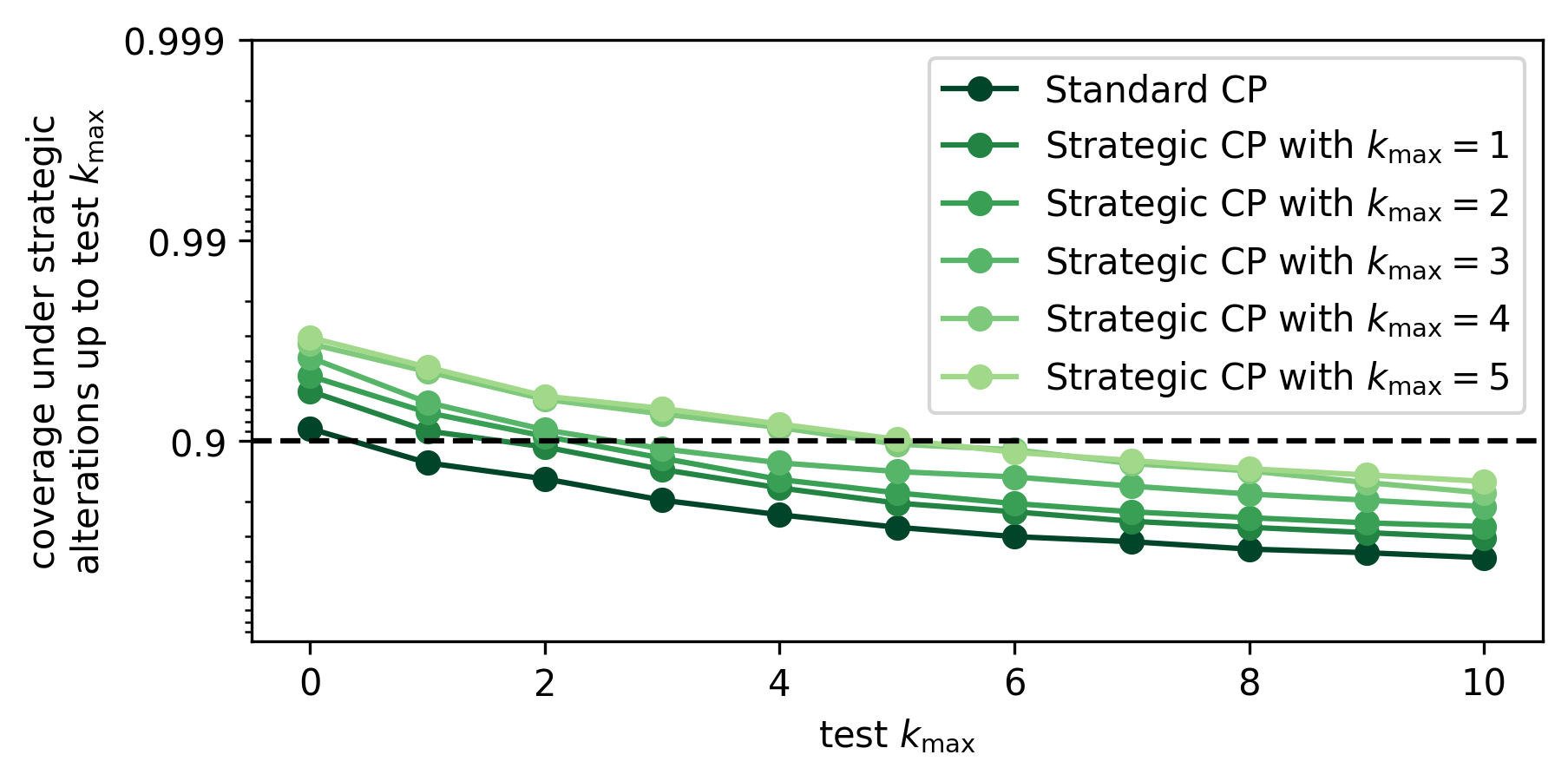}
        \caption{Strategic Conformal Prediction on top of a Strategic XGBoost Model}
    \end{subfigure}
    \caption{\textbf{Coverage of our method and standard CP for varying levels of strategic alterations.} The severity ($k_{\max}$) of strategic alterations at test time is indicated in the x-axis, while different curves correspond to conformal calibrations targeting different levels of $k_{\max}$ (along with standard CP) on the \texttt{academic-dropout} dataset, all with $\alpha = 0.1$. The y-axis indicates the test miscoverage under the strategic alterations up to the severity indicated on the x-axis. These curves follow the behaviour guaranteed by Theorems~\ref{thm:conformal-marginal-guarantee} and \ref{thm:training-conditional-guarantee} as well as Proposition~\ref{thm:as-delta-converges}. Note that using a model more well-suited to strategic alterations ``stabilizes'' these curves, making them closer together and smoother. For more details and more datasets, see the supplementary material.}
    \label{fig:curves}
    \vspace{-1em}
\end{figure*}

\textbf{Datasets} We evaluate our method and baselines on 3 datasets for briefness, with 6 more used in the supplementary material.
\vspace{-0.4cm}
\begin{itemize}
    \item \verb|academic-dropout|~\cite{dataset-academic-dropout}: A dataset for prediction of academic dropout in higher education. Students try to exclude the possibility of dropout from the predictive sets, as the possibility of such an outcome may hinder opportunities.
\vspace{-0.25cm}
    \item \verb|spambase|~\cite{dataset-spambase}: A dataset for classification of email into spam and non-spam. Spammers try to exclude the spam outcome from the predictive sets.
\vspace{-0.25cm}
    \item \verb|shoppers|~\cite{dataset-shoppers}: A dataset for prediction of revenue generated by products. Sellers have the incentive to make the system believe that there will be a larger revenue, as it would increase exposure.
\end{itemize}
\vspace{-0.4cm}
\textbf{Base models} We apply our method atop two XGBoost models: one trained in the usual way (`Plain XGBoost') and another trained by repeated risk minimization~\cite{perdomo-performative-prediction} (`Strategic XGBoost'), which is more well-suited to the strategic setting.
The pseudocode for the exact repeated risk minimization procedure we use can be found in the supplementary material.
For conformal prediction, we then consider the conformity score $s(x, y) = \lvert \mu(x) - y \rvert$ for regression (where $\mu : \mathcal{X} \to \R$ is the base XGBoost model) and $s(x, y) = 1 - p(y | x)$ for classification (where $p(y | \cdot) : \mathcal{X} \to [0, 1]$ for each $y \in \mathcal{Y}$ are the conditional probabilities from the base XGBoost model).

\textbf{Strategic Alterations} In our experiments, we consider two particular forms of strategic alterations: one based on a utility-cost decomposition (Section~\ref{sec:delta-utility-cost}) and another based on the iterative random search formulation (Section~\ref{sec:delta-proximal}).
In either case, we consider ``fully-rational'' utility functions of the form
\( u(X') = \inf_{y \in \Omega} s(X', y) \)
for some $\Omega \subset \mathcal{Y}$.
The exact choices of $\Omega$ are done on a per-dataset basis and can be found in the supplementary material.

\textbf{Baselines} As our baselines we consider Standard Conformal Prediction over (i) the plain XGBoost model, and (ii) the Strategic XGBoost model.
The latter (combining a model fitted considering the strategic setting and then using standard conformal prediction) is, to our knowledge, the closest to uncertainty quantification for the strategic setting that one can get with previously existing methods.

\begin{table*}
    \caption{\textbf{Evaluation of our method on multiple datasets and forms of strategic alterations for $\alpha = 0.1$.} Strategic coverages substantially below the specified confidence level of $90\%$ (below $60\%$) are marked in \textbf{\textcolor{purple}{purple}}. Note how our method is consistently close to the specified coverage, while standard CP never attains it, most of the time substantially so. This improvement from Strategic CP can result in small increases of interval size, but this can be managed by using a base model crafted for the strategic setting. Details and an extended version of this table with more datasets can be found in the supplementary material.}
    \label{tbl:varying}
    \centering
    \begin{tabular}{ccc cc c}
        & \multirow[b]{2}{*}{ \textbf{\shortstack{UNDERLYING \\ MODEL}}\vspace{-0.85em} }
        &
        & \multicolumn{2}{c}{ \textbf{\shortstack{STRATEGIC COVERAGE}} }
        & \multirow[b]{2}{*}{ \textbf{\shortstack{AVG SET \\ SIZE DIFF}}\vspace{-0.85em} }
        \\\cmidrule(lr){4-5}
        \textbf{DATASET}
        &
        & \textbf{$\Delta$s}
        & \textbf{OURS}
        & \textbf{STD CP}
        &
        \\
        \hline

        \multirow{4}{*}{\texttt{academic-dropout}}
        & \multirow{2}{*}{Plain XGBoost}
        & Utility-cost
        & $\mathbf{90\% \pm 2\%}$ & \textcolor{purple}{$\mathbf{57\% \pm 3\%}$} & $+1.08 \pm 0.05$ \\
        &
        & Rand. Search.
        & $\mathbf{92\% \pm 2\%}$ & \textcolor{purple}{$\mathbf{49\% \pm 3\%}$} & $+1.40 \pm 0.04$ \\
        & \multirow{2}{*}{Strategic XGBoost}
        & Utility-cost
        & $\mathbf{91\% \pm 2\%}$ & $76\% \pm 3\%$ & $+0.52 \pm 0.04$ \\
        &
        & Rand. Search.
        & $\mathbf{92\% \pm 2\%}$ & $75\% \pm 3\%$ & $+0.60 \pm 0.05$ \\[.15cm]

        \multirow{4}{*}{\texttt{spambase}}
        & \multirow{2}{*}{Plain XGBoost}
        & Utility-cost
        & $\mathbf{91\% \pm 2\%}$ & \textcolor{purple}{$\mathbf{15\% \pm 2\%}$} & $+1.06 \pm 0.04$ \\
        &
        & Rand. Search.
        & $\mathbf{87\% \pm 2\%}$ & \textcolor{purple}{$\mathbf{29\% \pm 3\%}$} & $+0.77 \pm 0.04$ \\
        & \multirow{2}{*}{Strategic XGBoost}
        & Utility-cost
        & $\mathbf{90\% \pm 2\%}$ & \textcolor{purple}{$\mathbf{40\% \pm 3\%}$} & $+0.38 \pm 0.03$ \\
        &
        & Rand. Search.
        & $\mathbf{89\% \pm 2\%}$ & \textcolor{purple}{$\mathbf{46\% \pm 3\%}$} & $+0.53 \pm 0.04$ \\[.15cm]

        \multirow{4}{*}{\texttt{shoppers}}
        & \multirow{2}{*}{Plain XGBoost}
        & Utility-cost
        & $\mathbf{89\% \pm 1\%}$ & \textcolor{purple}{$\mathbf{44\% \pm 2\%}$} & $+0.14 \pm 0.01$ \\
        &
        & Rand. Search.
        & $\mathbf{88\% \pm 1\%}$ & \textcolor{purple}{$\mathbf{47\% \pm 2\%}$} & $+0.13 \pm 0.01$ \\
        & \multirow{2}{*}{Strategic XGBoost}
        & Utility-cost
        & $\mathbf{89\% \pm 1\%}$ & $64\% \pm 2\%$ & $+0.08 \pm 0.01$ \\
        &
        & Rand. Search.
        & $\mathbf{88\% \pm 1\%}$ & $72\% \pm 2\%$ & $+0.08 \pm 0.01$
    \end{tabular}
    \vspace{-1em}
\end{table*}

\textbf{Computational budget and code} All experiments were conducted on an AMD Ryzen 9 5950X CPU (2.2GHz/5.0GHz, 32 threads) with 64GB of RAM.
The code is available at \verb|redacted URL| and the supplementary material. 

As a first target of investigation, we look in Figure~\ref{fig:coverages-and-sizes} at the strategic coverage $\P[\forall \Delta \in \DeltaSet,\ Y \in C_t (\Delta (X))]$ obtained by our method compared to baselines (Figure~\ref{fig:coverages}), along with the corresponding average sizes of the predictive sets (Figure~\ref{fig:sizes}) on the \verb|academic-dropout| dataset.
It is immediately evident that standard conformal prediction does not obtain any semblance of coverage. Using a strategically-refitted model, which is more well-adapted to the strategic setting, is an improvement, but not nearly enough.
Our method, in contrast, consistently and tightly obtains the specified nominal coverage, regardless of how well-adapted the underlying model is.

We see that our method does, unfortunately, lead to larger predicted sets. This is inevitable: the original sets were too small and thus did not cover the outcome under strategic alterations, and so such an increase is justified.
Nevertheless, we do see -- as was indicated by Proposition~\ref{thm:tightness-1} -- that using a model more well-suited to the strategic setting (namely, our Strategic XGBoost model) leads to smaller average interval sizes in general, as the underlying conformity score is then more well-adapted to the strategic alterations.

We also perform a similar experiment while varying the dataset and strategic alterations being considered -- results of which can be found in Table~\ref{tbl:varying}.
Our method consistently dominates standard conformal prediction (by a substantial amount!) throughout all of the settings explored, showcasing its reliability and versatility.

Figure~\ref{fig:curves} considers a setting in which the strategic alterations used for calibration may not match the ones that occur in practice.
In particular, it considers the case where both alterations follow the iterative random search scheme of Section~\ref{sec:delta-proximal}, but with a different number of iterations.
In practical terms, this corresponds to the actual population putting a ``different amount of effort'' into breaking the deployed model than anticipated during calibration.

We see that, as designed for in Section~\ref{sec:delta-proximal}, if the actual population does ``less effort'' than anticipated (i.e., the $k_{\max}$ in practice is \emph{smaller} than the one used for calibration), then we are completely protected. If the population does ``more effort'' (i.e., the $k_{\max}$ in practice is \emph{larger} than the one used for calibration), then we are not protected in general, \emph{unless} the $k_{\max}$ used for calibration is large enough; if that is so, then, as anticipated in Proposition~\ref{thm:as-delta-converges}, we will be sufficiently close to obtaining coverage for the accumulation point of the strategic alterations, and thus protect even against larger amounts of effort than originally foreseen.

Moreover, still in Figure~\ref{fig:curves} we can clearly see the impact of using a base model that is more well-suited to the strategic setting: such a model ``stabilizes'' the curves in the plot, making it so that the varying sizes of $k_{\max}$, both for test and calibration, have a smaller impact on the overall curve, and thus leading to more efficient and robust predictive sets.

Finally, we also analyze the behavior of our method under a more strenuous form of model specification where we alter the actual stochastic step used for the stochastic simulations.
The figures for this experiment can be found in the supplementary material.
While significant misspecification does break our guarantees, small changes approximately maintain coverage, in line with the results of Propositions~\ref{thm:robustness-1} and \ref{thm:robustness-2}.
\vspace{-0.2cm}
\section{Conclusion}
\vspace{-0.2cm}
In this work we presented Strategic Conformal Prediction, a variant of the Conformal Prediction framework that is well-suited for the strategic setting, where the individuals' covariates change in response to the deployed model.

Our method is, to our knowledge, the first instance of uncertainty quantification for the strategic setting.
We show, both theoretically and empirically, that it properly ensures coverage in spite of strategic alterations, whereas previously available methods (which were not tailored for the strategic setting) do not. This is further supplemented by a series of theoretical results on the behavior of our method under model misspecification, establishing its robustness. For these reasons, we believe our method to be not only a meaningful theoretical contribution the the landscape of strategic ML, but also a reliable tool for real-world application. 

\bibliography{paper}

\onecolumn
\appendix

\section{Proofs}\label{suppl:proofs}

We will use the following well-known lemma from the Conformal Prediction literature:

\begin{lemma}[Quantile lemma]\label{thm:quantile-lemma}
    If $V_1, \ldots, V_{n+1}$ are exchangeable random variables, then for any $\beta \in (0, 1)$, we have
    \[ \P[ V_{n+1} \leq \hat{q}_{\beta}(V_1, \ldots, V_n, +\infty) ] \geq \beta, \]
    where $\hat{q}_{\beta}(W_1, \ldots, W_m)$ is the empirical quantile of $W_1, \ldots, W_m$, given by
    \[ \hat{q}_{\beta}(W_1, \ldots, W_m) = \inf \left\{ t \in \R : \frac{1}{m} \sum_{i=1}^m \ind[W_i \leq t] \geq \beta \right\}. \]
\end{lemma}
\begin{proof}
    See Lemma 1 of \cite{conformal-under-covariate-shift}.
\end{proof}

\begin{theorem}[Theorem~\ref{thm:conformal-marginal-guarantee} in the main text]
    Let $Z_1, \ldots, Z_n, Z^\test$ be $n + 1$ exchangeable random variables in $\mathcal{X} \times \mathcal{Y}$. Let $t^\strat_\alpha$ be as above. Then, for any $\alpha \in (0, 1)$,
    \[ \P\left[ \forall \Delta \in \mathcal{D}, \ Y^\test \in C_{t^\strat_\alpha} (\Delta (X^\test)) \right] \geq 1 - \alpha. \]
\end{theorem}

\begin{proof}
    By definition of $C_{t^\strat_\alpha}$,
    \begin{align*}
        &\P[\forall \Delta \in \DeltaSet,\ Y^\test \in C_{t^\strat_\alpha} (\Delta (X^\test))]
        \\ &\quad = \P[\forall \Delta \in \DeltaSet,\ s(\Delta (X^\test), Y^\test) \leq t^\strat_\alpha]
        \\ &\quad = \P[\sup_{\Delta \in \DeltaSet} s(\Delta (X^\test), Y^\test) \leq t^\strat_\alpha]
    \end{align*}
    For convenience, let $S_i := \sup_{\Delta \in \DeltaSet} s(\Delta (X_i), Y_i)$ and $S^\test := \sup_{\Delta \in \DeltaSet} s(\Delta (X^\test), Y^\test)$, and let $S_{(1)}, \ldots, S_{(n)}, S_{(n+1)}$ be the order statistics of $S_1, \ldots, S_n, S^\test$. Then
    \begin{align*}
        \P[\sup_{\Delta \in \DeltaSet} s(\Delta (X^\test), Y^\test) \leq t^\strat_\alpha]
        = \P[S^\test \leq \widehat{q}_{1-\alpha} (S_1, \ldots, S_n, +\infty)].
    \end{align*}
    And, by Lemma~\ref{thm:quantile-lemma}:
    \begin{align*}
        &\P[S^\test \leq \widehat{q}_{1-\alpha} (S_1, \ldots, S_n, +\infty)]
        = \P[S^\test \leq \widehat{q}_{1-\alpha} (S_1, \ldots, S_n, S^\test)] \\
        &= \P[S^\test \leq S_{( \lceil (1-\alpha) (n+1) \rceil )}]
        = \P[\rank(S^\test) \leq \lceil (1-\alpha) (n+1) \rceil] \\
        &= \P_{U \sim \Unif(1 .. n+1)}[U \leq \lceil (1-\alpha) (n+1) \rceil]
        = \frac{\lceil (1-\alpha) (n+1) \rceil}{n+1} \geq \frac{(1-\alpha) (n+1)}{n+1} = 1 - \alpha.
    \end{align*}
\end{proof}

To prove our training-conditional guarantee, we will leverage a generalization of a lemma of \cite{bian-training-conditional-conformal}:

\begin{lemma}[Slight generalization of Lemma~1 from \cite{bian-training-conditional-conformal}]\label{thm:training-conditional-lemma}
    Let $n \geq 2$ and choose a holdout set $A$ with $\emptyset \subsetneq A \subsetneq [n]$.
    Let $\varphi_{[n] \setminus A} = \mathcal{A}((X_i, Y_i) : i \in [n] \setminus A) : \mathcal{X} \times \mathcal{Y} \to \R$, where $\mathcal{A}$ is any algorithm and may be deterministic or randomized. Define
    \[ p_{A,\varphi} (x, y) := \frac{1}{\lvert A \rvert} \sum_{i \in A} \ind[ \varphi_{[n] \setminus A}(X_i, Y_i) \geq \varphi_{[n] \setminus A}(x, y) ], \]
    and
    \[ p_{A,\varphi}^\star (x, y) := \P[\varphi_{[n] \setminus A} (X, Y) \geq \varphi_{[n] \setminus A} (x, y) | \varphi_{[n] \setminus A}]. \]
    Then $p_{A,\varphi}^\star (x, y)$ is a valid $p$-value conditional on the training data, i.e.,
    \[ \P[p_{A,\varphi}^\star (x, y) \leq a | \mathcal{D}_n] <= a \text{ for all } a \in [0, 1], \text{ almost surely over } \mathcal{D}_n. \]
    Moreover, for any $\epsilon \geq \sqrt{\log(2) / 2 \lvert A \rvert}$,
    \[ \P\left[\sup_{x,y} \bigl( p_{A,\varphi}^\star (x, y) - p_{A,\varphi} (x, y) \bigr) > \epsilon\right] \leq e^{-2 \lvert A \rvert \epsilon^2}. \]
\end{lemma}
\begin{proof}
    First, for any fixed function $s : \mathcal{X} \times \mathcal{Y} \to \R$, define
    \[ \bar{F}_s (t) = \P_P [s(X, Y) \geq t]. \]
    In other words, $\bar{F}_\mu (t)$ is the right-tailed CDF of the conformity score $s(X, Y)$ under $(X,Y) \sim P$.
    We can therefore write
    \[ p_{A,\varphi}^* (X_{n+1}, Y_{n+1}) = \bar{F}_s (s(X_{n+1}, Y_{n+1})). \]

    Since $(X_{n+1}, Y_{n+1}) \sim P$ (and is independent of $s$), this is clearly a valid $p$-value by definition of $\bar{F}_s$.

    Next, for any $(x, y) \in \mathcal{X} \times \mathcal{Y}$, we can calculate
    \begin{align*}
        p_{A,\varphi}^* (x, y) - p_{A,\varphi} (x, y)
        &= \bar{F}_s (s(x, y)) - \frac{1}{\lvert A \rvert} \sum_{i \in A} \ind[ \varphi_{[n] \setminus A} (X_i, Y_i) \geq \varphi_{[n] \setminus A} (x, y) ] \\
        &\leq \sup_{t \in \R} \left( \bar{F}_s (s(x, y)) - \frac{1}{\lvert A \rvert} \sum_{i \in A} \ind[ \varphi_{[n] \setminus A} (X_i, Y_i) \geq t ] \right).
    \end{align*}

    Finally, since $(X_i, Y_i)_{i \in A}$ are drawn i.i.d. from $P$ and are independent from $\varphi_{[n] \setminus A}$, for any $\epsilon \geq \sqrt{log(2)/ 2 \lvert A \rvert}$ the Dvoretzky-Kiefer-Wolfowitz-Massart (DKW) inequality implies that, conditional on $\varphi_{[n] \setminus A}$,
    \[ \sup_{t \in RR} \left( \bar{F}_s (t) - \frac{1}{\lvert A \rvert} \sum_{i \in A} \ind[ \varphi_{[n] \setminus A}(X_i, Y_i) \geq t ] \right) \leq \epsilon \]
    holds with probability at least $1 - e^{-2 \lvert A \rvert \epsilon^2}$. The same bound therefore holds marginally as well.
\end{proof}

We can now prove the Theorem:

\begin{theorem}[Theorem~\ref{thm:training-conditional-guarantee} in the main text]
    For any $\delta \in (0, 1)$, with probability of at least $1 - \delta$ over the draw of i.i.d. random variables $Z_1, \ldots, Z_n$, it holds that
    \begin{align*}
        & \P\left[ \forall \Delta \in \DeltaSet, \ Y^\test \in C_{t^\strat_\alpha} (\Delta (X^\test)) \ \middle|\ t^\strat_\alpha \right]
        \\ &\quad \geq 1 - \alpha - \sqrt{\frac{\log 1/\delta}{2n}}.
    \end{align*}
\end{theorem}
\begin{proof}
    By definition of $C_{t^\strat_\alpha}$,
    \begin{align*}
        & \P\left[ \forall \Delta \in \DeltaSet, \ Y^\test \in C_{t^\strat_\alpha} (\Delta (X^\test)) \ \middle|\ t^\strat_\alpha \right]
        \\ &\quad = \P\left[ \forall \Delta \in \DeltaSet, \ s(\Delta(X^\test), Y^\test) \leq t^\strat_\alpha \ \middle|\ t^\strat_\alpha \right]
        \\ &\quad = \P\left[ \sup_{\Delta \in \DeltaSet} s(\Delta(X^\test), Y^\test) \leq t^\strat_\alpha \ \middle|\ t^\strat_\alpha \right]
        \\ &\quad = \P\left[ \sum_{i=1}^n \ind\left[\sup_{\Delta' \in \DeltaSet} s(\Delta'(X_i), Y_i) > \sup_{\Delta \in \DeltaSet} s(\Delta(X^\test), Y^\test)\right] > (n+1) \alpha - 1 \ \middle|\ t^\strat_\alpha \right]
        \\ &\quad = \P\left[ \sum_{i=1}^n \ind\left[\sup_{\Delta' \in \DeltaSet} s(\Delta'(X_i), Y_i) > \sup_{\Delta \in \DeltaSet} s(\Delta(X^\test), Y^\test)\right] > n \alpha - (1 - \alpha) \ \middle|\ t^\strat_\alpha \right]
        \\ &\quad \geq \P\left[ \sum_{i=1}^n \ind\left[\sup_{\Delta' \in \DeltaSet} s(\Delta'(X_i), Y_i) > \sup_{\Delta \in \DeltaSet} s(\Delta(X^\test), Y^\test)\right] > n \alpha \ \middle|\ t^\strat_\alpha \right]
        \\ &\quad = \P\left[ \sum_{i=1}^n \ind\left[\sup_{\Delta \in \DeltaSet} s(\Delta(X_i), Y_i) > \sup_{\Delta \in \DeltaSet} s(\Delta(X^\test), Y^\test)\right] > n \alpha \ \middle|\ t^\strat_\alpha \right]
        = \P\left[ p_{[n]\setminus[n_0],\varphi}(X^\test, Y^\test) > \alpha \ \middle|\ t^\strat_\alpha \right],
    \end{align*}
    for $\varphi(x, y) = \sup_{\Delta \in \DeltaSet} s(\Delta (x), y)$.

    Now, by Lemma~\ref{thm:training-conditional-lemma}, we have that for any $\epsilon \geq \sqrt{\log(2) / 2 n}$,
    \[ \P\left[ \sup_{x,y} \bigl( p^*_{A,\varphi}(x, y) - p_{A,\varphi}(x, y) \bigr) > \epsilon \right] \leq e^{-2 n \epsilon^2}; \]
    solving for $\delta = e^{-2 n \epsilon^2}$, we get that, with probability of at least $1 - \delta$,
    \[ \sup_{x,y} \bigl( p^*_{A,\varphi}(x, y) - p_{A,\varphi}(x, y) \bigr) \leq \sqrt{ \frac{\log 1/\delta}{2 n} }, \]
    and thus that still with probability of at least $1 - \delta$,
    \begin{align*}
        & \P\left[ \forall \Delta \in \DeltaSet, \ Y^\test \in C_{t^\strat_\alpha} (\Delta (X^\test)) \ \middle|\ t^\strat_\alpha \right]
        \geq \P\left[ p_{[n]\setminus[n_0],\varphi}(X^\test, Y^\test) > \alpha \ \middle|\ t^\strat_\alpha \right]
        \\ &\quad = \P\left[ p^\star_{[n]\setminus[n_0],\varphi}(X^\test, Y^\test) - \left( p^\star_{[n]\setminus[n_0],\varphi}(X^\test, Y^\test) - p_{[n]\setminus[n_0],\varphi}(X^\test, Y^\test) \right) > \alpha \ \middle|\ t^\strat_\alpha \right]
        \\ &\quad = \P\left[ p^\star_{[n]\setminus[n_0],\varphi}(X^\test, Y^\test) > \alpha + \left( p^\star_{[n]\setminus[n_0],\varphi}(X^\test, Y^\test) - p_{[n]\setminus[n_0],\varphi}(X^\test, Y^\test) \right) \ \middle|\ t^\strat_\alpha \right]
        \\ &\quad \geq \P\left[ p^\star_{[n]\setminus[n_0],\varphi}(X^\test, Y^\test) > \alpha + \sup_{x, y} \left( p^\star_{[n]\setminus[n_0],\varphi}(x, y) - p_{[n]\setminus[n_0],\varphi}(x, y) \right) \ \middle|\ t^\strat_\alpha \right]
        \\ &\quad \geq \P\left[ p^\star_{[n]\setminus[n_0],\varphi}(X^\test, Y^\test) > \alpha + \sqrt{\frac{\log 1/\delta}{2 n}} \ \middle|\ t^\strat_\alpha \right]
        \\ &\quad = 1 - \P\left[ p^\star_{[n]\setminus[n_0],\varphi}(X^\test, Y^\test) \leq \alpha + \sqrt{\frac{\log 1/\delta}{2 n}} \ \middle|\ t^\strat_\alpha \right]
        \geq 1 - \alpha - \sqrt{\frac{\log 1/\delta}{2 n}},
    \end{align*}
    where the last step holds by Lemma~\ref{thm:training-conditional-lemma}, finishing the proof.
\end{proof}

The proofs of marginal coverage the group-conditional and label-conditional variants of our method are similar to the proof of Theorem~\ref{thm:conformal-marginal-guarantee}, with modifications analogous to the ones in the non-strategic setting.

\begin{theorem}[Theorem~\ref{thm:conformal-marginal-guarantee-groupconditional} in the main text]
    Let $Z_1, \ldots, Z_n, Z^\test$ be $n + 1$ exchangeable random variables in $\mathcal{X} \times \mathcal{Y}$. Let $C^{\stratm(G_j)}_\alpha$ be as above (for $j = 1, \ldots, l$). Then, for any $\alpha \in (0, 1)$,
    \begin{align*}
        \P\left[ \forall \Delta \in \DeltaSet, \ Y^\test \in C^{\stratm(G_j)}_{\alpha} (\Delta (X^\test)) \middle| X^\test \in G_j \right] \\ \geq 1 - \alpha, \quad \forall j = 1, \ldots, l. \quad
    \end{align*}
\end{theorem}
\begin{proof}
    Consider any group $G_j$.
    By definition of $C^{\strat-(G_j)}_\alpha$,
    \begin{align*}
        &\P[\forall \Delta \in \DeltaSet,\ Y^\test \in C^{\strat-(G_j)}_\alpha (\Delta (X^\test)) | X^\test \in G_j]
        \\ &\quad = \P[\forall \Delta \in \DeltaSet,\ s(\Delta (X^\test), Y^\test) \leq t^{\strat-(G_j)}_\alpha | X^\test \in G_j]
        \\ &\quad = \P[\sup_{\Delta \in \DeltaSet} s(\Delta (X^\test), Y^\test) \leq t^{\strat-(G_j)}_\alpha | X^\test \in G_j]
    \end{align*}
    For convenience, let $S_i := \sup_{\Delta \in \DeltaSet} s(\Delta (X_i), Y_i)$ and $S^\test := \sup_{\Delta \in \DeltaSet} s(\Delta (X^\test), Y^\test)$, $k = \lvert {i = 1, \ldots, n : X_i \in G_j} \rvert$, and let $S_{(1)}, \ldots, S_{(n)}, S_{(n+1)}$ be the order statistics of $S_1, \ldots, S_n, S^\test$. Then
    \begin{align*}
        \P[\sup_{\Delta \in \DeltaSet} s(\Delta (X^\test), Y^\test) \leq t^{\strat-(G_j)}_\alpha | X^\test \in G_j]
        = \P[S^\test \leq \widehat{q}_{1-\alpha} (S_1, \ldots, S_n, +\infty | X_* \in G_j) | X^\test \in G_j].
    \end{align*}
    And, by Lemma~\ref{thm:quantile-lemma}:
    \begin{align*}
        &\P[S^\test \leq \widehat{q}_{1-\alpha} (S_1, \ldots, S_n, +\infty | X_* \in G_j) | X^\test \in G_j]
        = \P[S^\test \leq \widehat{q}_{1-\alpha} (S_1, \ldots, S_n, S^\test | X_* \in G_j) | X^\test \in G_j] \\
        &= \P[S^\test \leq S_{( \lceil (1-\alpha) (k+1) \rceil )} | X^\test \in G_j]
        = \P[\rank(S^\test | X_* \in G_j) \leq \lceil (1-\alpha) (k+1) \rceil] \\
        &= \P_{U \sim \Unif(1 .. k+1)}[U \leq \lceil (1-\alpha) (k+1) \rceil]
        = \frac{\lceil (1-\alpha) (k+1) \rceil}{k+1} \geq \frac{(1-\alpha) (k+1)}{k+1} = 1 - \alpha.
    \end{align*}
\end{proof}

For the label-conditional case, see Section~\ref{sec:label-conditional}.

Most of our remaining results leverage the total variation distance, defined as such:

\begin{definition}[Total variation distance]
    Consider a measurable space $(\Omega, \mathcal{F})$ and probability measures(/distributions) $P$ and $Q$ defined on $(\Omega, \mathcal{F})$.
    The total variation distance between $P$ and $Q$ is defined as
    \[ \TV (P\ \|\ Q) = \sup_{A \in \mathcal{F}} \lvert P(A) - Q(A) \rvert. \]
\end{definition}

The definition above defines the total variation distance as an Integral Probability Metric (IPM).
We note the well-known fact that it is also an $f$-divergence:
\[ \TV(P\ \|\ Q) = \frac{1}{2} \E_{Z \sim P} \left[ \left\lvert \frac{\dif Q}{\dif P} - 1 \right\rvert \right]. \]

\begin{proposition}[Proposition~\ref{thm:robustness-1} in the main text]
    Let $C : \mathcal{X} \to 2^\mathcal{Y}$ be a set predictor.
    For any $\Delta \in \DeltaSet$,
    \begin{align*}
        &\left\lvert \P_{X,Y}[Y \in C_\alpha (\Delta (X))] - \P_{X,Y}[Y \in C_\alpha (\Delta^\star (X))] \right\rvert
        \\ &\quad \leq \E_X\left[\TV( \Delta (X) | X \ \Vert\ \Delta^\star (X) | X )\right].
    \end{align*}
\end{proposition}
\begin{proof}
    First, we rewrite the probabilities as probabilities over specific joint distributions:

    \begin{align*}
        \P_{X,Y}[Y \in C_{t^\strat_\alpha} (\Delta (X))]
        &= \P_{\Delta(X),X,Y}[Y \in C_{t^\strat_\alpha} (\Delta (X))], \\
        \P_{X,Y}[Y \in C_{t^\strat_\alpha} (\Delta^\star (X))]
        &= \P_{\Delta^\star(X),X,Y}[Y \in C_{t^\strat_\alpha} (\Delta^\star (X))].
    \end{align*}

    By the definition of the total variation distance, it then follows that

    \[
        \Bigl\lvert \P_{\Delta(X),X,Y}[Y \in C_t (\Delta (X))] - \P_{\Delta^\star(X),X,Y}[Y \in C_t (\Delta^* (X))] \Bigr\rvert
        \leq \TV(\Delta(X),X,Y \ \|\ \Delta^*(X),X,Y),
    \]

    which, by the dual representation of the total variation distance as an $f$-divergence, can be rewritten as

    \begin{align*}
        &\TV(\Delta(X),X,Y \ \|\ \Delta^\star(X),X,Y)
        = \frac{1}{2}\, \E_{\Delta^\star(X),X,Y}\left[\left\lvert \frac{\dif P_{\Delta(X),X,Y}}{\dif P_{\Delta^\star(X),X,Y}} - 1 \right\rvert\right]
        \\ &\quad = \frac{1}{2}\, \E_{\Delta^\star(X),X,Y}\left[\left\lvert \frac{\dif P_{\Delta(X)|X,Y} \cdot \dif P_{X,Y}}{\dif P_{\Delta^\star(X)|X,Y} \cdot \dif P_{X,Y}} - 1 \right\rvert\right]
        = \frac{1}{2}\, \E_{\Delta^\star(X),X,Y}\left[\left\lvert \frac{\dif P_{\Delta(X)|X,Y}}{\dif P_{\Delta^\star(X)|X,Y}} - 1 \right\rvert\right]
        \\ &\quad = \frac{1}{2}\, \E_{\Delta^\star(X),X,Y}\left[\left\lvert \frac{\dif P_{\Delta(X)|X}}{\dif P_{\Delta^\star(X)|X}} - 1 \right\rvert\right]
        = \E_{X,Y}\left[ \frac{1}{2}\, \E_{\Delta^\star(X)|X,Y}\left[\left\lvert \frac{\dif P_{\Delta(X)|X}}{\dif P_{\Delta^\star(X)|X}} - 1 \right\rvert \middle| X,Y\right] \right]
        \\ &\quad = \E_X\left[ \frac{1}{2}\, \E_{\Delta^\star(X)|X}\left[\left\lvert \frac{\dif P_{\Delta(X)|X}}{\dif P_{\Delta^\star(X)|X}} - 1 \right\rvert \middle| X\right] \right]
        = \E_X[ \TV(\Delta(X)|X \lVert \Delta^\star(X)|X) ],
    \end{align*}
    which concludes the proof.
\end{proof}

\begin{proposition}[Proposition~\ref{thm:robustness-2} in the main text]
    Let $C : \mathcal{X} \to 2^\mathcal{Y}$ be a set predictor satisfying Equation 2. Then
    \begin{align*}
        &\P_{X,Y}[Y \in C_\alpha (\Delta^\star (X))]
        \\ &\quad \geq 1 - \alpha - \inf_{\Delta \in \DeltaSet} \E_X\left[\TV( \Delta (X) | X \ \Vert\ \Delta^\star (X) | X )\right].
    \end{align*}
\end{proposition}
\begin{proof}
    By Proposition~\ref{thm:robustness-1}, for any $\Delta \in \DeltaSet$,
    \[
        \left\lvert \P_{X,Y}[Y \in C_\alpha (\Delta (X))] - \P_{X,Y}[Y \in C_\alpha (\Delta^\star (X))] \right\rvert
        \leq \E_X\left[\TV( \Delta (X) | X \ \Vert\ \Delta^\star (X) | X )\right],
    \]
    which means that
    \[
        \P_{X,Y}[Y \in C_\alpha (\Delta^\star (X))]
        \geq \P_{X,Y}[Y \in C_\alpha (\Delta (X))] - \E_X\left[\TV( \Delta (X) | X \ \Vert\ \Delta^\star (X) | X )\right],
    \]
    and by Theorem~\ref{thm:conformal-marginal-guarantee} the probability on the right-hand-side is lower-bounded by $1 - \alpha$, and thus
    \[
        \P_{X,Y}[Y \in C_\alpha (\Delta^\star (X))]
        \geq 1 - \alpha - \E_X\left[\TV( \Delta (X) | X \ \Vert\ \Delta^\star (X) | X )\right],
    \]
    Taking the supremum over $\Delta \in \DeltaSet$ on both sides, we get
    \begin{align*}
        \sup_{\Delta \in \DeltaSet} \P_{X,Y}[Y \in C_\alpha (\Delta^\star (X))]
        &\geq \sup_{\Delta \in \DeltaSet} \left( 1 - \alpha - \E_X\left[\TV( \Delta (X) | X \ \Vert\ \Delta^\star (X) | X )\right] \right)
        \\
        \P_{X,Y}[Y \in C_\alpha (\Delta^\star (X))]
        &\geq 1 - \alpha - \inf_{\Delta \in \DeltaSet} \E_X\left[\TV( \Delta (X) | X \ \Vert\ \Delta^\star (X) | X )\right],
    \end{align*}
    as desired.
\end{proof}

\begin{proposition}[Proposition~\ref{thm:as-delta-converges} in the main text]
    Let $(\Delta_k)_{k=0}^\infty$ be a sequence of alteration functions $\Delta_k : \mathcal{X} \to \mathcal{X}$, such that $\lim_{k \to \infty} \Delta_k (X) = \overline{\Delta}(X)$ according to the total variation metric.
    Then, for any $k_{\max} > 0$, consider the set predictor $C_{t^\strat_\alpha}$ generated by $\DeltaSet = \{ \Delta_0, \ldots, \Delta_{k_{\max}} \}$. There exists some $\epsilon_{k_{\max}}$ such that,
    \[ \P_{X,Y}[Y \in C_{t^\strat_\alpha} (\overline{\Delta} (X))] \geq 1 - \alpha - \epsilon_{k_{\max}}, \]
    and $\lim_{k_{\max} \to \infty} \epsilon_{k_{\max}} = 0$.
\end{proposition}
\begin{proof}
    Consider $\epsilon_{k_{\max}} = \TV(\Delta_k\ \|\ \overline{\Delta})$. Since the $\Delta_k$ converge to $\overline{\Delta}$ in total variation distance, by definition, it holds that $\lim_{k \in \infty} \TV(\Delta_k\ \|\ \overline{\Delta}) = 0$. And, by the definition of total variation distance, we have that
    \begin{align*}
        \P_{X,Y}[Y \in C_{t^\strat_\alpha}(\overline{\Delta}(X))]
        &= \E_{\overline{\Delta}(X)}[\P_Y[Y \in C_{t^\strat_\alpha}(\overline{\Delta}(X)) | \overline{\Delta}(X)]]
        \\ &\geq \E_{\Delta_k(X)}[\P_Y[Y \in C_{t^\strat_\alpha}(\Delta_k(X)) | \Delta_k(X)]] + \TV(\Delta_k\ \|\ \overline{\Delta})
        \\ &= \P_{X,Y}[Y \in C_{t^\strat_\alpha}(\Delta_k(X))] + \TV(\Delta_k\ \|\ \overline{\Delta})
        \\ &\geq 1 - \alpha + \TV(\Delta_k\ \|\ \overline{\Delta}) = 1 - \alpha - \epsilon_{k_{\max}}.
    \end{align*}
\end{proof}

To prove the following results we will make use of the following well-known lemma about the total-variation distance:

\begin{lemma}\label{thm:change-of-measure-tv}
    For any two distributions $P$ and $Q$ over $\mathcal{H}$ and function $\phi : \mathcal{H} \to [0, 1]$, it holds that
    \[ \bigl\lvert \E_{X \sim P}[\phi(X)] - \E_{Y \sim Q}[\phi(Y)] \bigr\rvert \leq \TV(P \| Q). \]
\end{lemma}
\begin{proof}
    See, e.g., \cite{novel-change-of-measure}, Lemma 4.
\end{proof}

\begin{proposition}[Proposition~\ref{thm:tightness-1} in the main text]
    Let $\mu : \mathcal{X} \to \R$ be a possibly-stochastic function representing a base model.
    Consider the predictive sets $C_{t^\strat_\alpha}$ with conformity score $s : \mathcal{X} \times \R \to \R$ given by $s(x, y) = \lvert \mu(x) - y \rvert$ with $\mathcal{Y}$ being a compact subset of $\mathcal{Y}$ and alteration functions $\Delta \in \DeltaSet$, and contrast it with the corresponding standard conformal predictive sets $C_{t^\std_\alpha}$. Then, almost surely over $X$,
    \begin{align*}
        &\E_{C_{t^\strat_\alpha}}[\relleb_{\mathcal{Y}}\left(C_{t^\strat_\alpha}(X)\right)] \leq \E_{C_{t^\std_\alpha}}[\relleb_{\mathcal{Y}}\left(C_{t^\std_\alpha}(X)\right)]
        \\ &+ 2\, \TV\left(\Bigl( s(X, Y) \Bigr)_{[j_{n,\alpha}]} \ \middle\|\ \Bigl( \sup_{\Delta \in \DeltaSet} s(\Delta (X), Y) \Bigr)_{[j_{n,\alpha}]} \right),
    \end{align*}
    for $j_{n,\alpha} = \lceil (1 - \alpha) (1 + n) \rceil$.
\end{proposition}
\begin{proof}
    First we need to connect the size of the conformal sets $C_t (X)$ with the calibrated threshold $t$. In the case of the conformity score we are considering, $s(x, y) = \lvert \mu(x) - y \rvert$, this is particularly simple, with the interval size corresponding to $2t$ precisely. So we have that
    \[
        \E_{C_{t^\strat_\alpha}} [ \leb(C_{t^\strat_\alpha} (X)) ] = \E_{t^\strat_\alpha} [t^\strat_\alpha]
        \qquad \text{and} \qquad
        \E_{C_{t^\std_\alpha}} [ \leb(C_{t^\std_\alpha} (X)) ] = \E_{t^\std_\alpha} [t^\std_\alpha].
    \]
    Next, note that the calibrated thresholds correspond exactly to the $j_{n,\alpha}$-th order statistic of the conformity scores (resp. suprema of conformity scores in the strategic setting), so

    Finally, we will use Lemma~\ref{thm:change-of-measure-tv}. Since $\mathcal{Y}$ is bounded and by definition $C_t (X) \subset \mathcal{Y} \ \forall X$, by monotonicity of measures it must also hold that
    \begin{align*}
        \leb(\mathcal{Y}) &\geq \leb(C_{t^\strat_\alpha} (X)) = 2 t^\strat_\alpha = 2 (\sup_{\Delta \in \DeltaSet} s(\Delta (X), Y))_{[j_{n,\alpha}]}
        \\
        \leb(\mathcal{Y}) &\geq \leb(C_{t^\std_\alpha} (X)) = 2 t^\std_\alpha = 2 (s(X, Y))_{[j_{n,\alpha}]},
    \end{align*}
    and thus, by Lemma~\ref{thm:change-of-measure-tv}, taking $\phi(t) = \leb(C_t(X)) = 2t$,
    \[ \E[\phi(t^\strat)] = \E[\leb(C_{t^\strat_\alpha} (X))] = 2 (\sup_{\Delta \in \DeltaSet} s(\Delta (X), Y))_{[j_{n,\alpha}]} \leq \leb(C_{t^\std_\alpha} (X)) \]
\end{proof}

\begin{corollary}[Corollary~\ref{thm:tightness-2} in the main text]
    In the same setting as Proposition~\ref{thm:tightness-1}, further assume that $\mu(X) = Y + \epsilon$ for some random $\epsilon$ (arbitrarily dependent on $X$ and $Y$), with $\epsilon \in [-M, M]$ almost surely. Then, almost surely over $X$,
    \begin{align*}
        &\E_{C_{t^\strat_\alpha}}[\leb_{\mathcal{Y}}\left(C_{t^\strat_\alpha}(X)\right)] \leq 2 M
        \\ &\qquad + 2\, \TV\left((\epsilon)_{[j_{n,\alpha}]} \ \middle\|\ \Bigl( \sup_{\Delta \in \DeltaSet} s(\Delta (X), Y) \Bigr)_{[j_{n,\alpha}]} \right),
    \end{align*}
    for $j_{n,\alpha} = \lceil (1 - \alpha) (1 + n) \rceil$.
\end{corollary}
\begin{proof}
    By Proposition~\ref{thm:tightness-1},
    \[ \E[\relleb_{\mathcal{Y}}[C_{t^\strat_\alpha}(X)]] \leq \E[\relleb_{\mathcal{Y}}[C_{t^\std_\alpha}(X)]] + 2\, \TV\left( \Bigl( s(X, Y) \Bigr)_{[j_{n,\alpha}]} \ \middle\|\ \Bigl( \sup_{\Delta \in \DeltaSet} s(\Delta (X), Y) \Bigr)_{[j_{n,\alpha}]} \right). \]
    Now, note that since $\mu(X) = Y + \epsilon$ almost surely, we have that $s(X, Y) = \lvert \mu(X) - Y \rvert = \lvert Y + \epsilon - Y \rvert = \lvert \epsilon \rvert \leq M$ also almost surely. Thus
    \[ \E[\relleb_{\mathcal{Y}}[C_{t^\std_\alpha}(X)]] \leq 2 t^\std_\alpha = 2 \Bigl( s(X, Y) \Bigr)_{[j_{n,\alpha}]} = 2 (\epsilon)_{[j_{n,\alpha}]} \leq 2 M, \]
    and so
    \begin{align*}
        \E[\relleb_{\mathcal{Y}}[C_{t^\strat_\alpha}(X)]]
        &\leq 2 M + 2\, \TV\left( \Bigl( s(X, Y) \Bigr)_{[j_{n,\alpha}]} \ \middle\|\ \Bigl( \sup_{\Delta \in \DeltaSet} s(\Delta (X), Y) \Bigr)_{[j_{n,\alpha}]} \right).
        \\ &\leq 2 M + 2\, \TV\left( ( \epsilon )_{[j_{n,\alpha}]} \ \middle\|\ \Bigl( \sup_{\Delta \in \DeltaSet} s(\Delta (X), Y) \Bigr)_{[j_{n,\alpha}]} \right).
    \end{align*}
\end{proof}

\section{Additional results}

\subsection{Label-Conditional Guarantee}\label{sec:label-conditional}

We can get a Label-Conditional Guarantee~\cite{label-conditional-conformal} for the case where $\mathcal{Y}$ is discrete.
We want to find a $C^{\stratm(\mathcal{Y})}_\alpha$ such that
\begin{align*}
    \P[\forall \Delta \in \DeltaSet,\ Y \in C^{\stratm(\mathcal{Y})}_\alpha (\Delta (X)) | Y = y] \geq 1 - \alpha,
    \\ \text{for all } y \in \mathcal{Y}. \quad
\end{align*}
Again, we can do strategic calibration separately for each label, and use the corresponding thresholds.
\begin{align*}
    &C^{\stratm(\mathcal{Y})}_\alpha(X) = \left\{ y \in \mathcal{Y} : s(X, y) \leq t^{\stratm(y)}_\alpha \right\}
    \\
    &t^{\stratm(y)}_\alpha = \inf \Biggl\{ t \in \R : \sum_{\substack{i=1 \\ X_i \in G_j}}^n \ind\left[ \sup_{\Delta \in \DeltaSet} s(\Delta(X_i), Y_i) > t \right]
        \\ &\qquad\qquad\qquad\qquad\qquad\qquad\qquad\quad + 1 \leq \alpha (n_{G_j}+1) \Biggr\}.
\end{align*}
\begin{theorem}\label{thm:conformal-marginal-guarantee-labelconditional}
    Let $Z_1, \ldots, Z_n, Z^\test$ be $n + 1$ exchangeable random variables in $\mathcal{X} \times \mathcal{Y}$. Let $C^{\stratm(\mathcal{Y})}_\alpha$ be as above. Then, for any $\alpha \in (0, 1)$,
    \begin{align*}
        \P\left[ \forall \Delta \in \DeltaSet, \ Y^\test \in C^{\stratm(\mathcal{Y})}_{\alpha} (\Delta (X^\test)) \middle| Y^\test = y \right] \\ \geq 1 - \alpha, \quad \forall y \in \mathcal{Y}. \quad
    \end{align*}
\end{theorem}
\begin{proof}
    Consider any label $y \in \mathcal{Y}$.
    By definition of $C^{\strat-(\mathcal{Y})}_\alpha$,
    \begin{align*}
        &\P[\forall \Delta \in \DeltaSet,\ Y^\test \in C^{\strat-(\mathcal{Y})}_\alpha (\Delta (X^\test)) | Y^\test = y]
        \\ &\quad = \P[\forall \Delta \in \DeltaSet,\ s(\Delta (X^\test), Y^\test) \leq t^{\strat-(y)}_\alpha | Y^\test = y]
        \\ &\quad = \P[\sup_{\Delta \in \DeltaSet} s(\Delta (X^\test), Y^\test) \leq t^{\strat-(y)}_\alpha | Y^\test = y]
    \end{align*}
    For convenience, let $S_i := \sup_{\Delta \in \DeltaSet} s(\Delta (X_i), Y_i)$ and $S^\test := \sup_{\Delta \in \DeltaSet} s(\Delta (X^\test), Y^\test)$, $k = \lvert {i = 1, \ldots, n : Y_i = y} \rvert$, and let $S_{(1)}, \ldots, S_{(n)}, S_{(n+1)}$ be the order statistics of $S_1, \ldots, S_n, S^\test$. Then
    \begin{align*}
        \P[\sup_{\Delta \in \DeltaSet} s(\Delta (X^\test), Y^\test) \leq t^{\strat-(y)}_\alpha | Y^\test = y]
        = \P[S^\test \leq \widehat{q}_{1-\alpha} (S_1, \ldots, S_n, +\infty | Y_* = y) | Y^\test = y].
    \end{align*}
    And, by Lemma~\ref{thm:quantile-lemma}:
    \begin{align*}
        &\P[S^\test \leq \widehat{q}_{1-\alpha} (S_1, \ldots, S_n, +\infty | Y_* = y) | Y^\test = y]
        = \P[S^\test \leq \widehat{q}_{1-\alpha} (S_1, \ldots, S_n, S^\test | Y_* = y) | Y^\test = y] \\
        &= \P[S^\test \leq S_{( \lceil (1-\alpha) (k+1) \rceil )} | Y^\test = y]
        = \P[\rank(S^\test | Y_* = y) \leq \lceil (1-\alpha) (k+1) \rceil] \\
        &= \P_{U \sim \Unif(1 .. k+1)}[U \leq \lceil (1-\alpha) (k+1) \rceil]
        = \frac{\lceil (1-\alpha) (k+1) \rceil}{k+1} \geq \frac{(1-\alpha) (k+1)}{k+1} = 1 - \alpha.
    \end{align*}
\end{proof}

Theorem~\ref{thm:conformal-marginal-guarantee-labelconditional} holds true even if $\mathcal{Y}$ is not discrete. However, the predicted intervals will contain infinitely many points, since the thresholds for the infinitely many values of $y \in \mathcal{Y}$ that did not appear in the calibration set will be infinite.

\section{Experiment details}\label{suppl:experiment-details}

\subsection{Datasets}\label{suppl:datasets}

We list our datasets along with their corresponding ``fully-rational'' utilities $u(X')$. In these utilities, $\mu(x)$ and $p(\dot | x)$ correspond to the trained base model, as described in the main text.

\begin{itemize}
    \item \verb|academic-dropout|~\cite{dataset-academic-dropout}: A dataset for prediction of academic dropout in higher education. Students try to exclude the possibility of dropout from the predictive sets, as the possibility of such an outcome may hinder opportunities. $u(X') = -p(\text{will dropout} | X')$.
    \item \verb|spambase|~\cite{dataset-spambase}: A dataset for classification of email into spam and non-spam. Spammers try to exclude the spam outcome from the predictive sets. $u(X') = -p(\text{is spam} | X')$.
    \item \verb|shoppers|~\cite{dataset-shoppers}: A dataset for prediction of sales of products. Sellers have the incentive to make the system believe that there will be a larger amount of sales, as it would increase exposure. $u(X') = -p(\text{will not be a sale} | X')$.
    \item \verb|news|~\cite{dataset-news}: A dataset for prediction of the number of `shares' for news articles (e.g., for recommendation purposes). News sources have an incentive to have the algorithm believe there will be a larger number of shares, as it would increase exposure. $u(X') = \sup_{0 \leq y \leq 1000} \lvert \mu(X') - y \rvert$.
    \item \verb|wine|~\cite{dataset-wine}: A dataset for assessment of wine quality, with samples composed of physicochemical tests of red and white vinho verde wine samples from the north of Portugal. Wineries have an incentive to obtain predictions of higher quality, as it would increase the value of their wine. $u(X') = \sup_{0 \leq y \leq 7} \lvert \mu(X') - y \rvert$.
    \item \verb|productivity|~\cite{dataset-productivity}: A dataset for inferring the productivity of workers of a garmet manufacturing process. Workers naturally wish to appear more productive. $u(X') = \sup_{0 \leq y \leq 0.5} \lvert \mu(X') - y \rvert$.
    \item \verb|taiwan|~\cite{dataset-taiwan}: A dataset for credit scoring over a population of customers in Taiwan. Individuals attempt to have positive credit scores. $u(X') = -p(\text{will default} | X')$.
    \item \verb|bank-marketing|~\cite{dataset-bank-marketing}: A dataset for the prediction of the success of marketing campaigns of a Portuguese banking institution, which can then be leveraged to make calls. Malicious individuals may seek to fool these predictions for their own benefit, e.g., to maximize opportunities. $u(X') = -p(\text{will churn} | X')$.
    \item \verb|census-income|~\cite{dataset-census-income}: A dataset for predicting whether individuals' income exceeds \$50K/yr based on census data. Such predictions can then be used, e.g., for credit scoring, and so there is significant incentive to make the system predict a higher income. $u(X') = -p(\text{is under \$50k} | X')$.
\end{itemize}

\subsection{Base models and conformity score}\label{suppl:base-models}

We apply our method atop two XGBoost models: one trained in the usual way (`Plain XGBoost') and another trained by repeated risk minimization~\cite{perdomo-performative-prediction} (`Strategic XGBoost'), which is more well-suited to the strategic setting.

Let $\mathrm{train}(X_{1:n}, Y_{1:n}) : \mathcal{X} \to \mathcal{Y}$ denote the XGBoost model trained on supervised data $X_{1:n} \in \mathcal{X}^n$ and $Y_{1:n} \in \mathcal{Y}^n$, and let $X^\mathrm{tr}_{1:n}, Y^\mathrm{tr}_{1:n}$ be the complete training dataset. Furthermore, let $\Delta_{h}$ denote the strategic alterations as a function of the trained model $h$.

The `Plain XGBoost' model is given simply by $\mathrm{train}(X^\mathrm{tr}_{1:n}, Y^\mathrm{tr}_{1:n})$.
The `Strategic XGBoost' model corresponds to $\mathrm{train}(X^{(10)}_{1:n}, Y^\mathrm{tr}_{1:n})$,where
\[ X^{(0)}_{1:n} = X^\mathrm{tr}_{1:n} \qquad\qquad\qquad X^{(k+1)}_{1:n} = \Delta_{\mathrm{train}(X^{(k)}_{1:n}, Y)} (X^\mathrm{tr}_{1:n}). \]
For conformal prediction, we then consider the conformity score $s(x, y) = \lvert \mu(x) - y \rvert$ for regression (where $\mu : \mathcal{X} \to \R$ is the base XGBoost model) and $s(x, y) = 1 - p(y | x)$ for classification (where $p(y | \cdot) : \mathcal{X} \to [0, 1]$ for each $y \in \mathcal{Y}$ are the conditional probabilities from the base XGBoost model).

\subsection{Strategic alterations}\label{suppl:strategic-alterations}

Throughout, we consider two kinds of strategic alterations:

\paragraph{Utility + cost}
Following the scheme described in Section 2.1.1, using the utility specified in Section~\ref{suppl:datasets} for the corresponding datasets, and cost given by the Malahanobis distance $c(X, X') = X^T \Sigma^{-1} X'$, where $\Sigma$ is the estimated covariance matrix of the covariates. The set $\DeltaSet$ is generated by regularization scales $\lambda \in \{10^{-7}, 1, 5\}$, and optimization is done by random search over the domain (whose zeroth-order nature allows for usage over non-differentiable base models).

\paragraph{Iterative random search}
Following the scheme described in Section 2.1.2 and 2.1.3, using the utility specified in Section~\ref{suppl:datasets} for the corresponding datasets, while sampling the $\delta$ from a multivariate normal distribution with mean 0 and covariance matching the one estimated from the data rescaled by $0.2$, with just two $\delta$ being sampled per iteration.

\subsection{Figure 1: coverage and interval sizes for our method versus standard CP}

On the left, the x-axis corresponds to the confidence level $1 - \alpha$ for the conformal calibration, while the y-axis corresponds to the empirical coverage under strategic alterations, $\widehat{\P}[\forall \Delta \in \DeltaSet,\ Y^\test \in C_{t^\strat_\alpha}(X^\test)]$, estimated on the test set.
All instances of Figure~1 use the iterative random search alterations.

Ideally, both will match up, leading to the `exact coverage' diagnoal line. Below the line we get `invalid coverage', meaning that the probability is below the specified $1 - \alpha$, and above the line we get `valid coverage', meaning that the probability is above the specified $1 - \alpha$ (though possibly excessively so).

On the right, the x-axis is the same, but the y-axis reports the mean interval size under strategic alterations, $\widehat{\E}[\sup_{\Delta \in \DeltaSet} \mathrm{size}(C_{t^\strat_\alpha}(X^\test))]$, where $\mathrm{size}$ is the Lebesgue measure in the regression case and the counting measure for the classification case, similarly estimated on the test set.

\subsection{Figure 2: coverage of our method andstandard CP for varying levels of strategic alterations}

All instances of Figure~2 use the iterative random search alterations.
The x-axis corresponds to the size of the $\DeltaSet$ used for metrics on the test set, while the color of the curves indicate the size of the $\DeltaSet$ used for calibration (as indicated in the legend).
The y-axis corresponds to the value of the strategic coverage under the level of alteration indicated on the x-axis.

All conformal calibrations were done with $1 - \alpha = 90\%$, which is marked on the plots by the dashed line. (Above the dashed line is valid coverage, while below is invalid coverage.)

The left plot is for the Plain XGBoost model, while the right one is for the Strategic XGBoost one.

\subsection{Table 1: evaluation of our method on multiple datasets and forms of strategic alterations for $\alpha = 0.1$}

Table columns:
\begin{itemize}
    \item \textbf{DATASET}: the dataset being considered.
    \item \textbf{UNDERLYING MODEL}: the base model -- either Plain XGBoost or Strategic XGBoost
    \item \textbf{$\Delta$s}: the strategic alterations being considered -- either Utility-cost (the utlity + cost construction) or Rand. Search. (the iterative random search construction).
    \item \textbf{STRATEGIC COVERAGE: OURS}: the strategic coverage of our method (Strategic Conformal Prediction): $\widehat{\P}[\forall \Delta \in \DeltaSet,\ Y^\test \in C_{t^\strat_\alpha}(X^\test)]$, estimated on the test set;
    \item \textbf{STRATEGIC COVERAGE: STD}: the strategic coverage of standard conformal prediction: $\widehat{\P}[\forall \Delta \in \DeltaSet,\ Y^\test \in C_{t^\std_\alpha}(X^\test)]$, estimated on the test set;
    \item \textbf{AVG SET SIZE DIFF}: the average strategic size difference: $\widehat{\E}[\inf_{\Delta \in \DeltaSet} \mathrm{size}(C_{t^\std_\alpha}(X^\test))]$, estimated on the test set.
\end{itemize}

All $\pm$ correspond to $95\%$ bootstrapped confidence intervals.
Black bold coverages indicate the best method for that row (incidentally always ours), while purple coverages indicate substantially invalid coverages (more than 30\% away from attaining valid coverage).

\subsection{Extra figure: coverage under change of the stochastic step}

With this experiment, our goal is to analyze the behavior of our method under a more strenuous form of model specification where we alter the actual stochastic step used for the stochastic simulations.
To do so, we consider the strategic alterations based on iterative random search. Let $\DeltaSet = \{\Delta_0, \ldots, \Delta_k\}$ denote the set of these strategic alterations.
We consider a family of `modified strategic alterations' $\widetilde{\Delta}_k (X) = \Delta_k (X) + \mathcal{N}(0, \lambda \Sigma)$, where $\lambda > 0$ and $\Sigma$ is the estimated covariance matrix of the data.
This way, by taking $\lambda = 0$ (at the limit) we recover the original $\Delta$s, while by increasing $\lambda$ we stray away from it, increasing $\TV(\widetilde{\Delta}_k(X) \ \|\ \Delta_k(X))$.

The figures plot on the x-axis the values of this $\lambda$, and on the y-axis the strategic coverage.
The blue line corresponds to Plain XGBoost, and the orange one to Strategic XGBoost.

Interestingly, using Strategic XGBoost increases our reliance on the correct specification of the $\Delta$s, and so actually \emph{decreases} our robustness.

\section{More experiment results}\label{suppl:more-experiments}

\subsection{Figure 1: coverage and interval sizes for our method versus standard CP}

\begin{figure}[H]
    \centering
    \includegraphics[width=.9\textwidth]{figures/fig1-legend.png}
    \begin{subfigure}[t]{0.48\textwidth}
        \centering
        \includegraphics[width=\columnwidth]{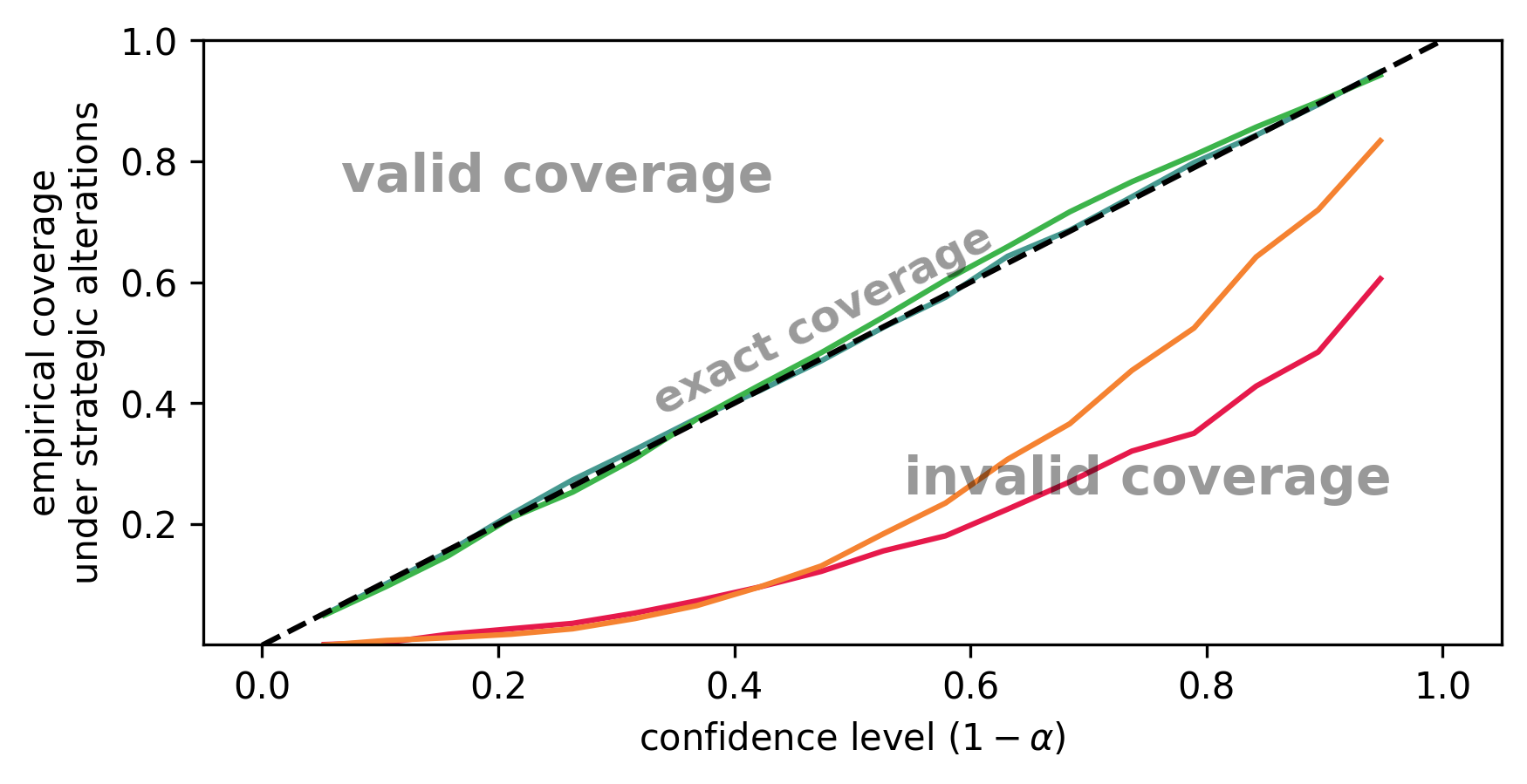}
        \caption{Coverage of standard CP \& our method under strategic alterations over varying values of $\alpha$.}
    \end{subfigure}\ \quad
    \begin{subfigure}[t]{0.48\textwidth}
        \centering
        \includegraphics[width=\columnwidth]{figures/size-mini.png}
        \caption{Average sizes of the predictive sets of standard CP \& our method under strategic alterations.}
    \end{subfigure}
    \caption{\texttt{academic-dropout}}
\end{figure}

\begin{figure}[H]
    \centering
    \includegraphics[width=.9\textwidth]{figures/fig1-legend.png}
    \begin{subfigure}[t]{0.48\textwidth}
        \centering
        \includegraphics[width=\columnwidth]{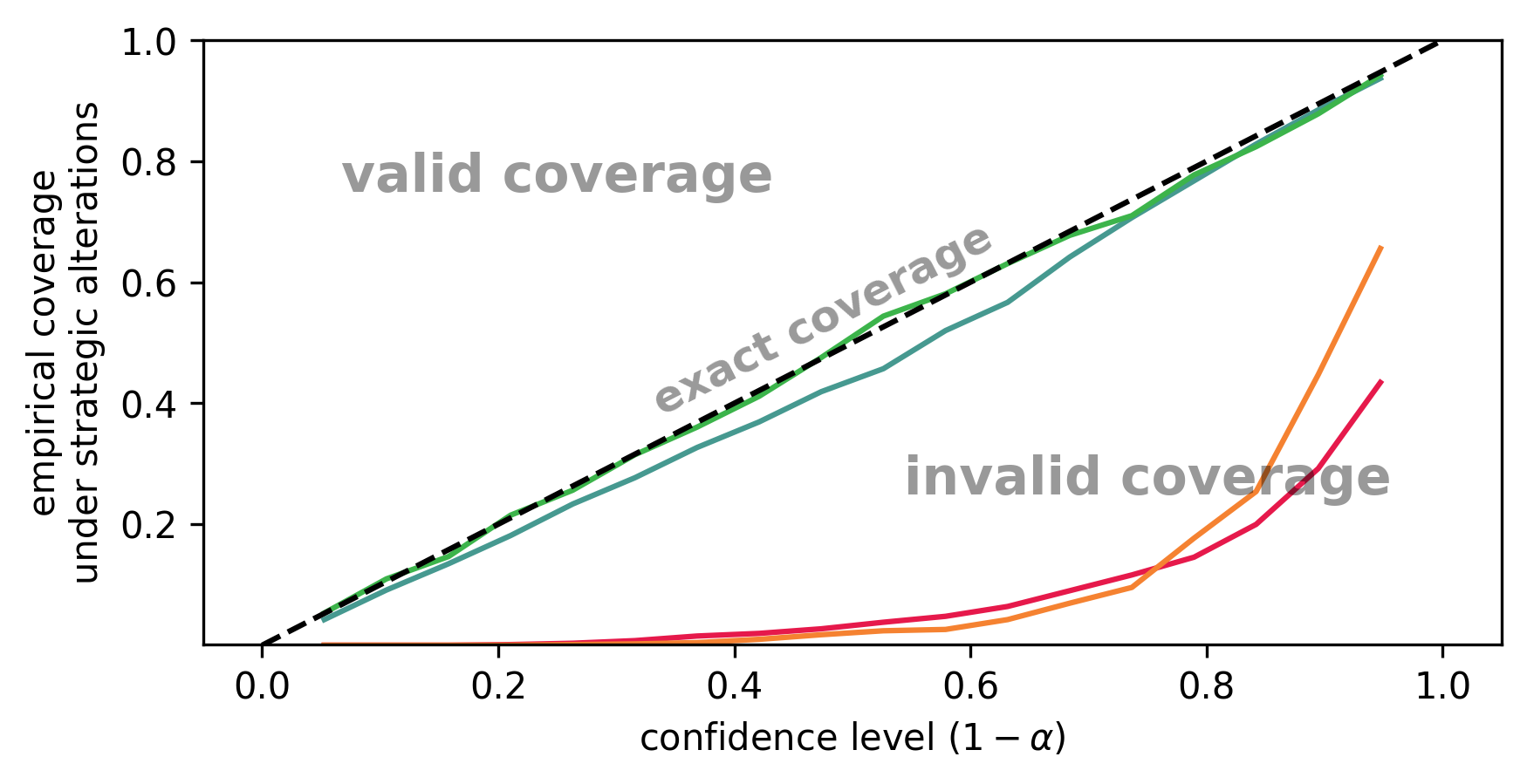}
        \caption{Coverage of standard CP \& our method under strategic alterations over varying values of $\alpha$.}
    \end{subfigure}\ \quad
    \begin{subfigure}[t]{0.48\textwidth}
        \centering
        \includegraphics[width=\columnwidth]{figures/size-mini.png}
        \caption{Average sizes of the predictive sets of standard CP \& our method under strategic alterations.}
    \end{subfigure}
    \caption{\texttt{spambase}}
\end{figure}

\begin{figure}[H]
    \centering
    \includegraphics[width=.9\textwidth]{figures/fig1-legend.png}
    \begin{subfigure}[t]{0.48\textwidth}
        \centering
        \includegraphics[width=\columnwidth]{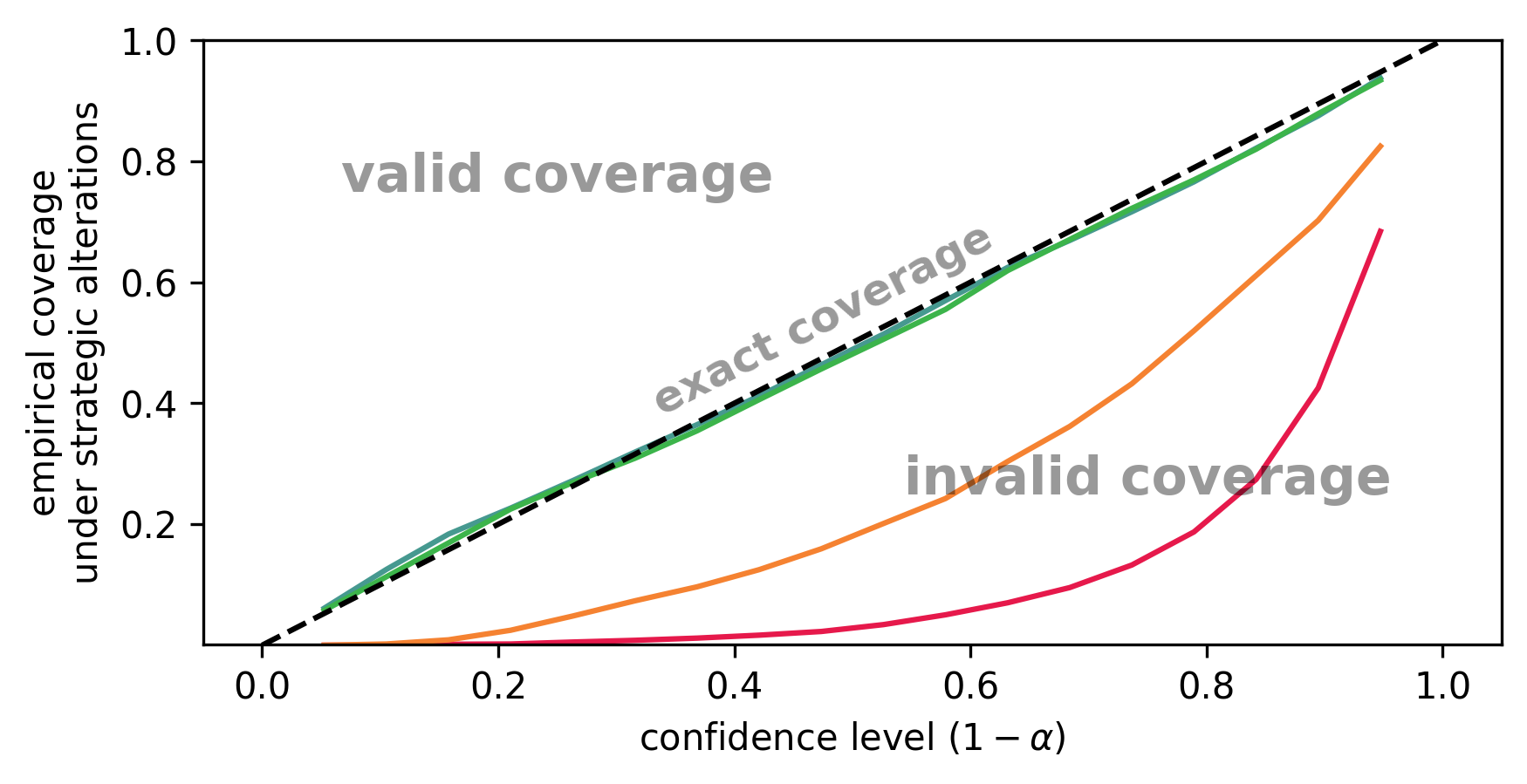}
        \caption{Coverage of standard CP \& our method under strategic alterations over varying values of $\alpha$.}
    \end{subfigure}\ \quad
    \begin{subfigure}[t]{0.48\textwidth}
        \centering
        \includegraphics[width=\columnwidth]{figures/size-mini.png}
        \caption{Average sizes of the predictive sets of standard CP \& our method under strategic alterations.}
    \end{subfigure}
    \caption{\texttt{shoppers}}
\end{figure}

\begin{figure}[H]
    \centering
    \includegraphics[width=.9\textwidth]{figures/fig1-legend.png}
    \begin{subfigure}[t]{0.48\textwidth}
        \centering
        \includegraphics[width=\columnwidth]{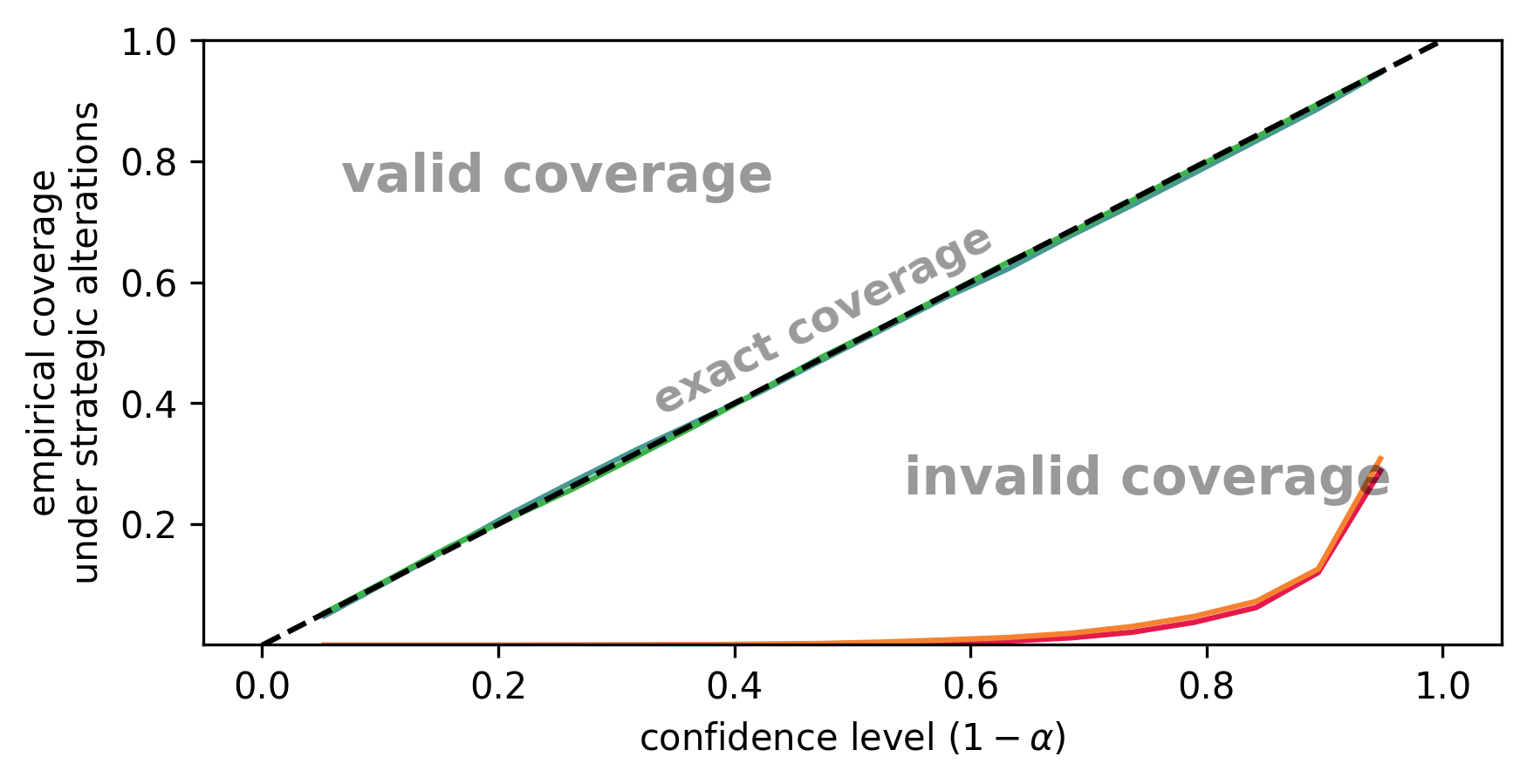}
        \caption{Coverage of standard CP \& our method under strategic alterations over varying values of $\alpha$.}
    \end{subfigure}\ \quad
    \begin{subfigure}[t]{0.48\textwidth}
        \centering
        \includegraphics[width=\columnwidth]{figures/size-mini.png}
        \caption{Average sizes of the predictive sets of standard CP \& our method under strategic alterations.}
    \end{subfigure}
    \caption{\texttt{news}}
\end{figure}

\begin{figure}[H]
    \centering
    \includegraphics[width=.9\textwidth]{figures/fig1-legend.png}
    \begin{subfigure}[t]{0.48\textwidth}
        \centering
        \includegraphics[width=\columnwidth]{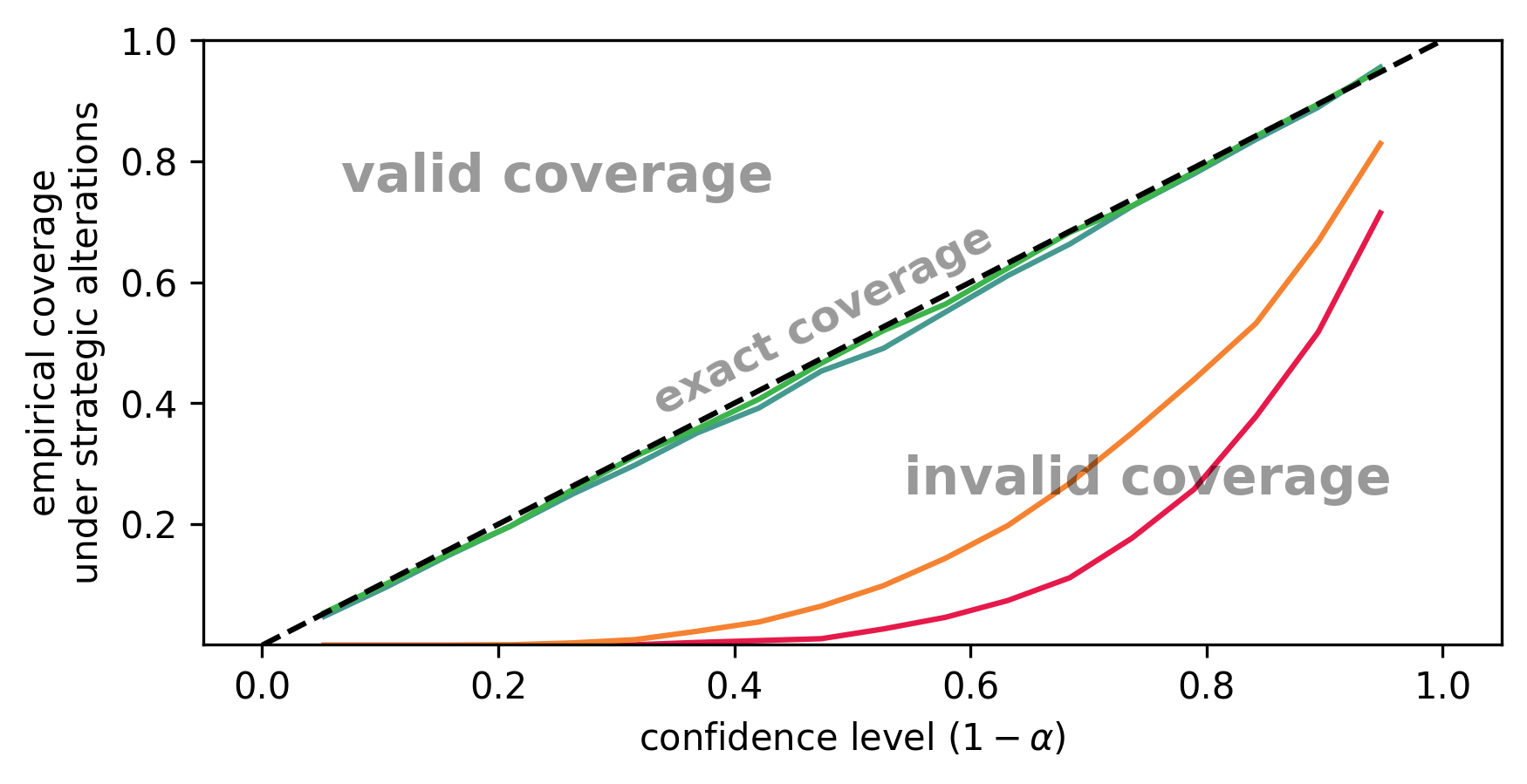}
        \caption{Coverage of standard CP \& our method under strategic alterations over varying values of $\alpha$.}
    \end{subfigure}\ \quad
    \begin{subfigure}[t]{0.48\textwidth}
        \centering
        \includegraphics[width=\columnwidth]{figures/size-mini.png}
        \caption{Average sizes of the predictive sets of standard CP \& our method under strategic alterations.}
    \end{subfigure}
    \caption{\texttt{wine}}
\end{figure}

\begin{figure}[H]
    \centering
    \includegraphics[width=.9\textwidth]{figures/fig1-legend.png}
    \begin{subfigure}[t]{0.48\textwidth}
        \centering
        \includegraphics[width=\columnwidth]{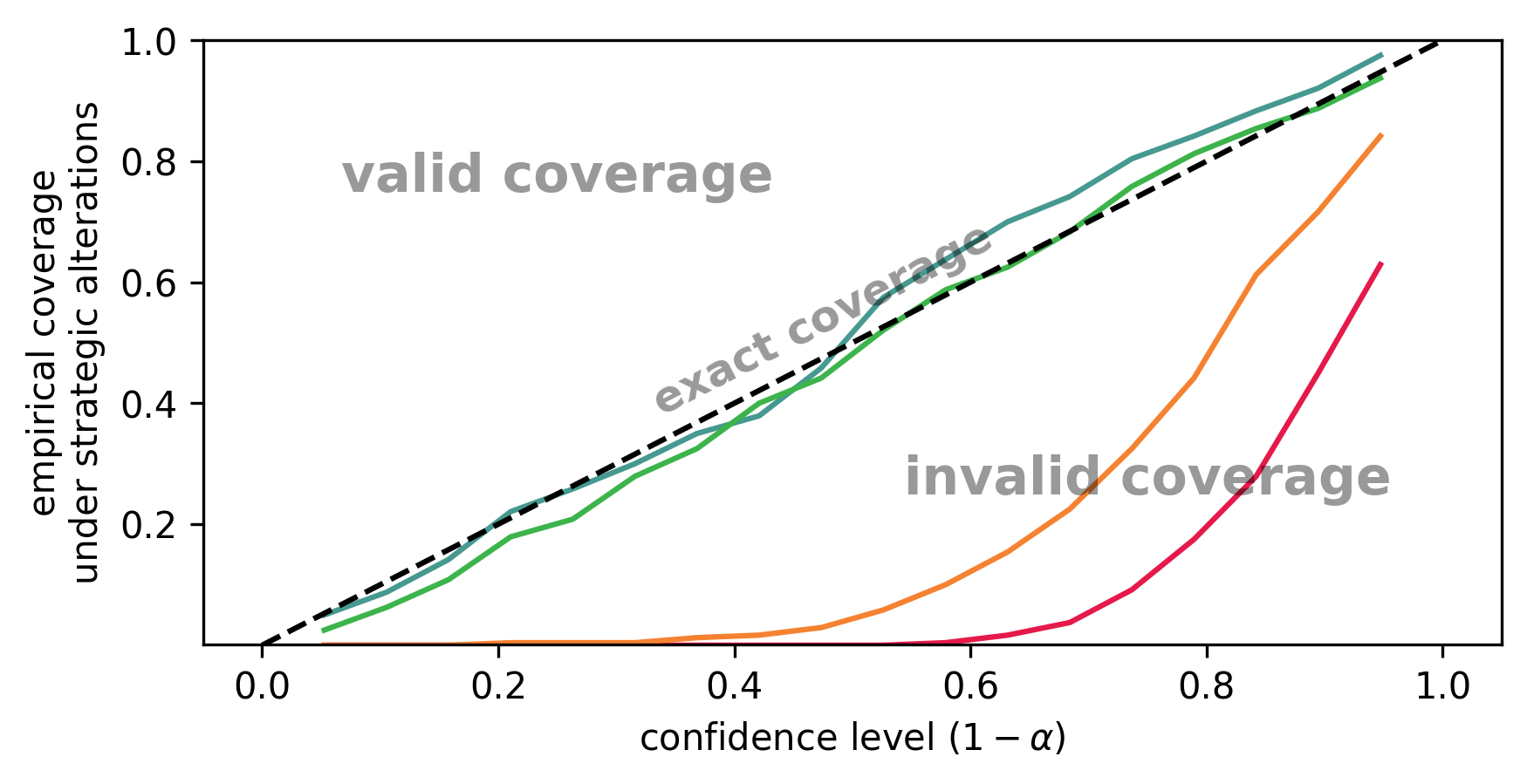}
        \caption{Coverage of standard CP \& our method under strategic alterations over varying values of $\alpha$.}
    \end{subfigure}\ \quad
    \begin{subfigure}[t]{0.48\textwidth}
        \centering
        \includegraphics[width=\columnwidth]{figures/size-mini.png}
        \caption{Average sizes of the predictive sets of standard CP \& our method under strategic alterations.}
    \end{subfigure}
    \caption{\texttt{productivity}}
\end{figure}

\begin{figure}[H]
    \centering
    \includegraphics[width=.9\textwidth]{figures/fig1-legend.png}
    \begin{subfigure}[t]{0.48\textwidth}
        \centering
        \includegraphics[width=\columnwidth]{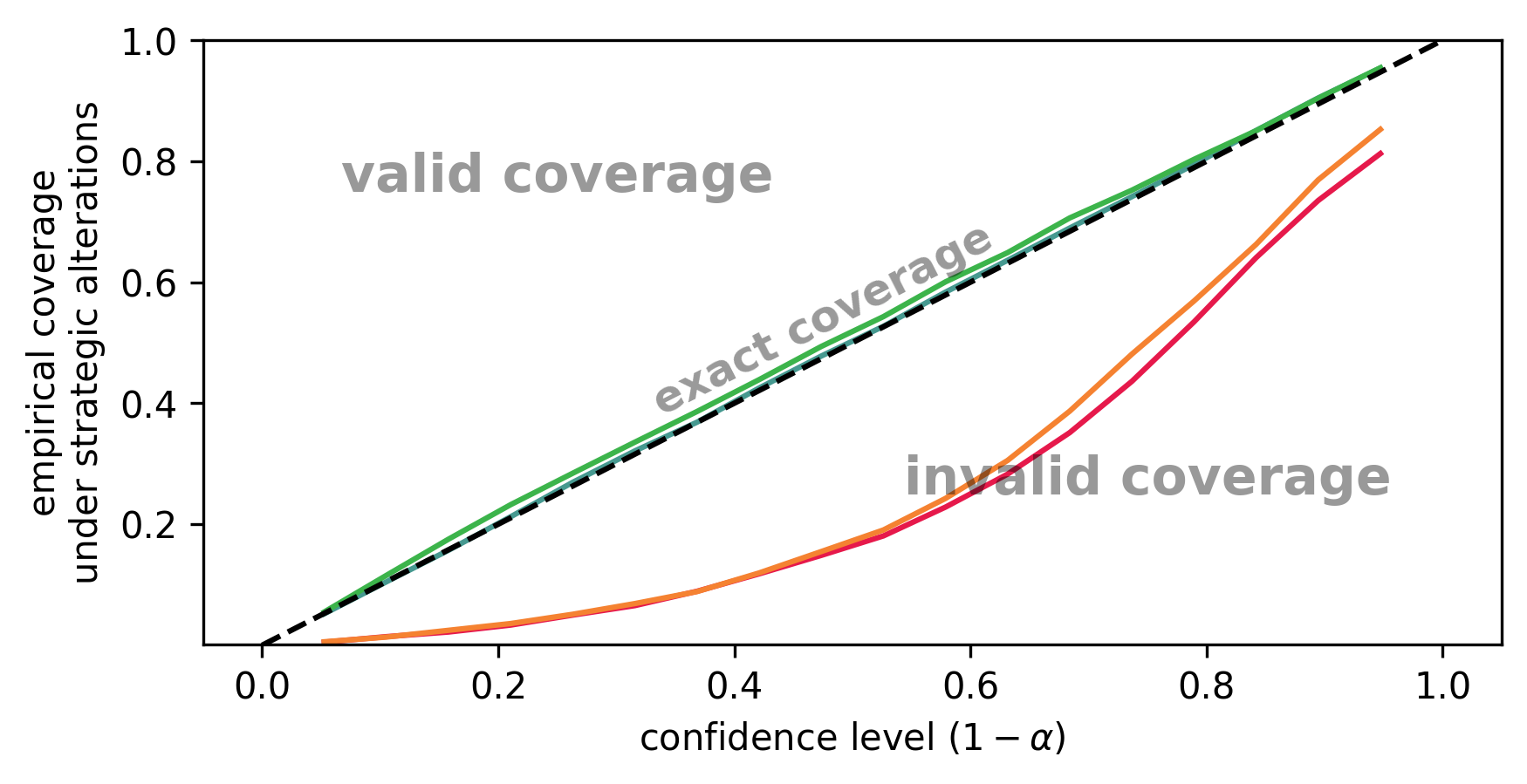}
        \caption{Coverage of standard CP \& our method under strategic alterations over varying values of $\alpha$.}
    \end{subfigure}\ \quad
    \begin{subfigure}[t]{0.48\textwidth}
        \centering
        \includegraphics[width=\columnwidth]{figures/size-mini.png}
        \caption{Average sizes of the predictive sets of standard CP \& our method under strategic alterations.}
    \end{subfigure}
    \caption{\texttt{taiwan}}
\end{figure}

\begin{figure}[H]
    \centering
    \includegraphics[width=.9\textwidth]{figures/fig1-legend.png}
    \begin{subfigure}[t]{0.48\textwidth}
        \centering
        \includegraphics[width=\columnwidth]{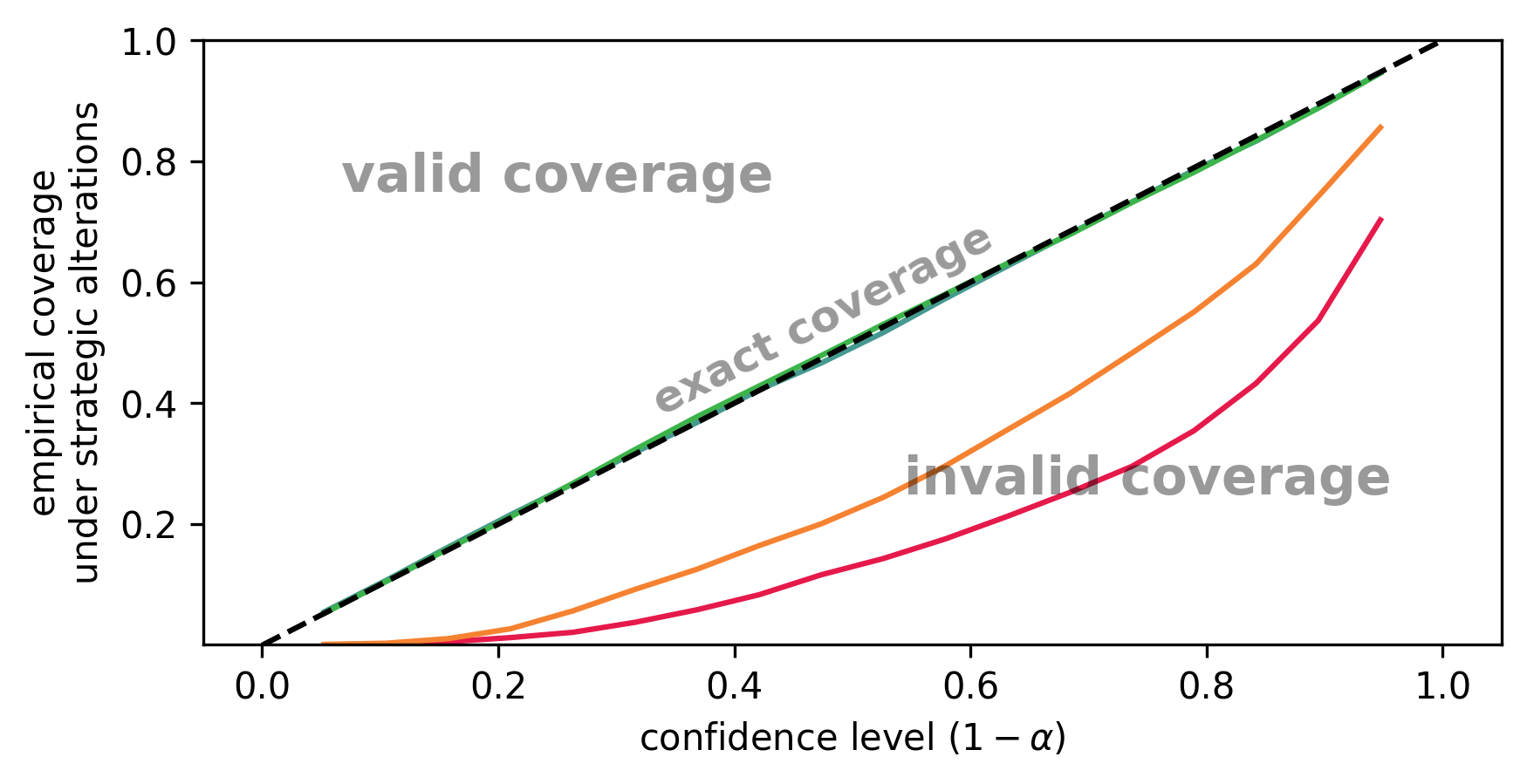}
        \caption{Coverage of standard CP \& our method under strategic alterations over varying values of $\alpha$.}
    \end{subfigure}\ \quad
    \begin{subfigure}[t]{0.48\textwidth}
        \centering
        \includegraphics[width=\columnwidth]{figures/size-mini.png}
        \caption{Average sizes of the predictive sets of standard CP \& our method under strategic alterations.}
    \end{subfigure}
    \caption{\texttt{bank-marketing}}
\end{figure}

\begin{figure}[H]
    \centering
    \includegraphics[width=.9\textwidth]{figures/fig1-legend.png}
    \begin{subfigure}[t]{0.48\textwidth}
        \centering
        \includegraphics[width=\columnwidth]{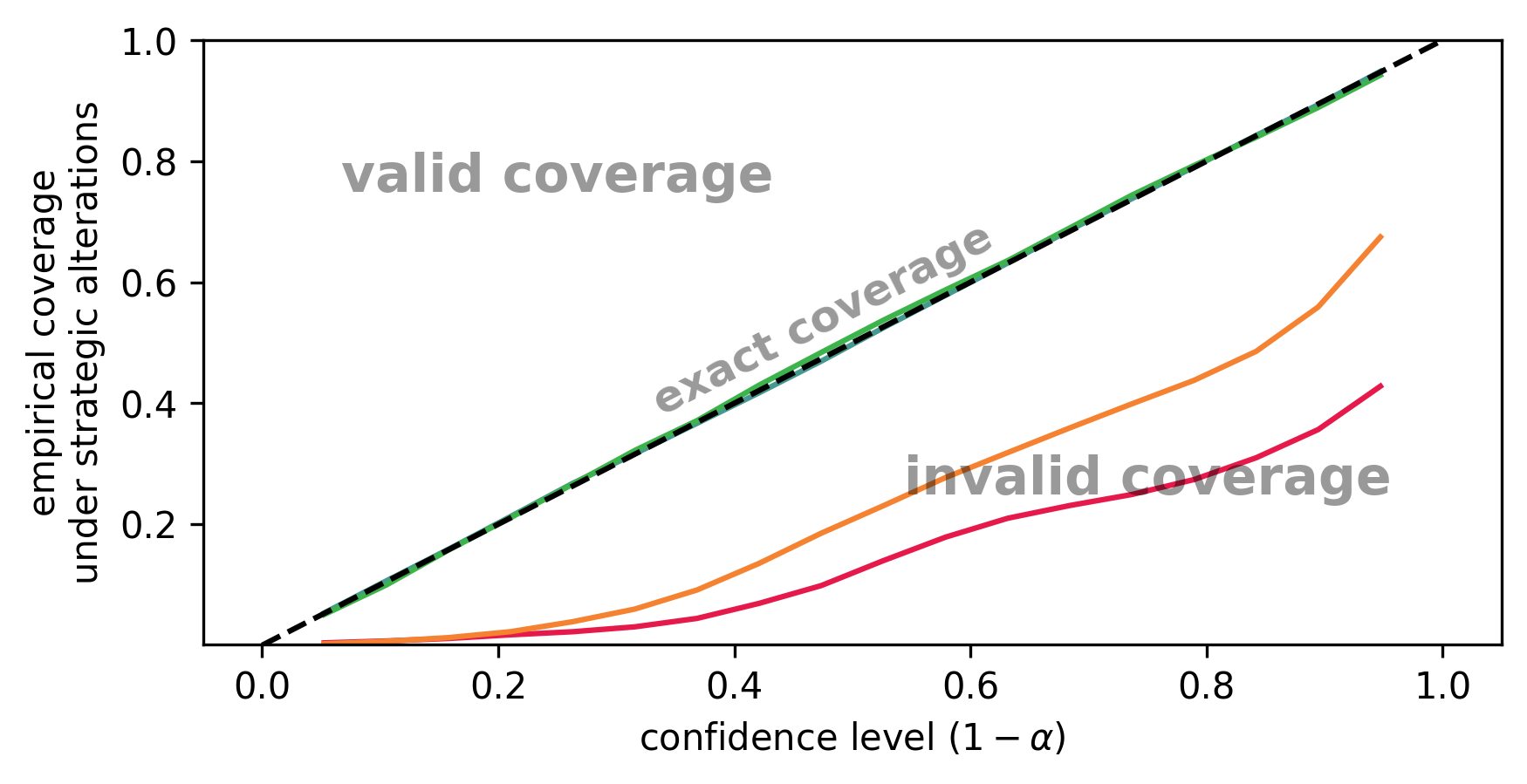}
        \caption{Coverage of standard CP \& our method under strategic alterations over varying values of $\alpha$.}
    \end{subfigure}\ \quad
    \begin{subfigure}[t]{0.48\textwidth}
        \centering
        \includegraphics[width=\columnwidth]{figures/size-mini.png}
        \caption{Average sizes of the predictive sets of standard CP \& our method under strategic alterations.}
    \end{subfigure}
    \caption{\texttt{census-income}}
\end{figure}

\subsection{Figure 2: coverage of our method andstandard CP for varying levels of strategic alterations}

\begin{figure}[H]
    \centering
    \begin{subfigure}[t]{0.48\textwidth}
        \centering
        \includegraphics[width=\columnwidth]{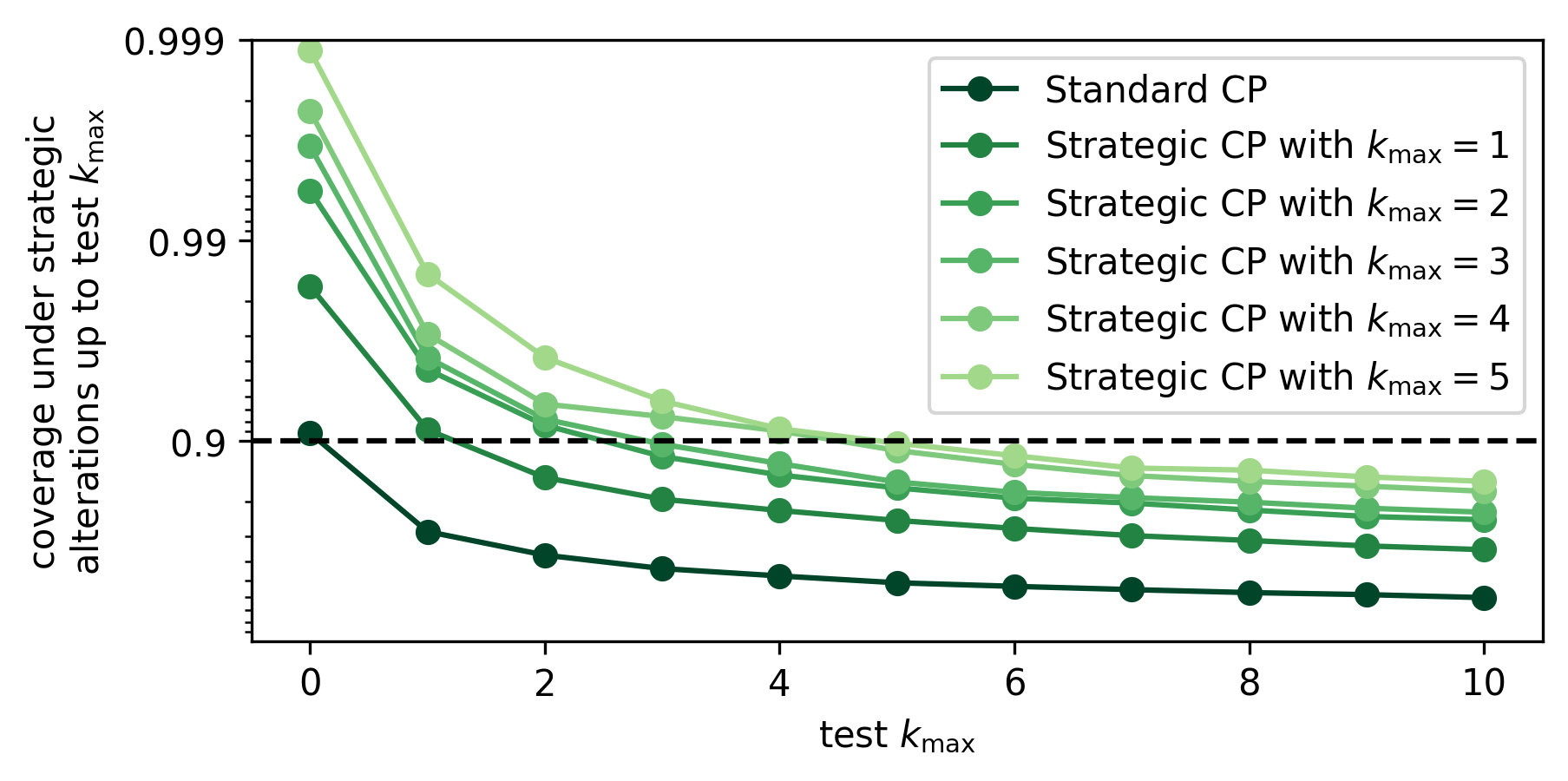}
        \caption{Strategic Conformal Prediction on top of a Plain \linebreak XGBoost Model}
    \end{subfigure}\ \quad
    \begin{subfigure}[t]{0.48\textwidth}
        \centering
        \includegraphics[width=\columnwidth]{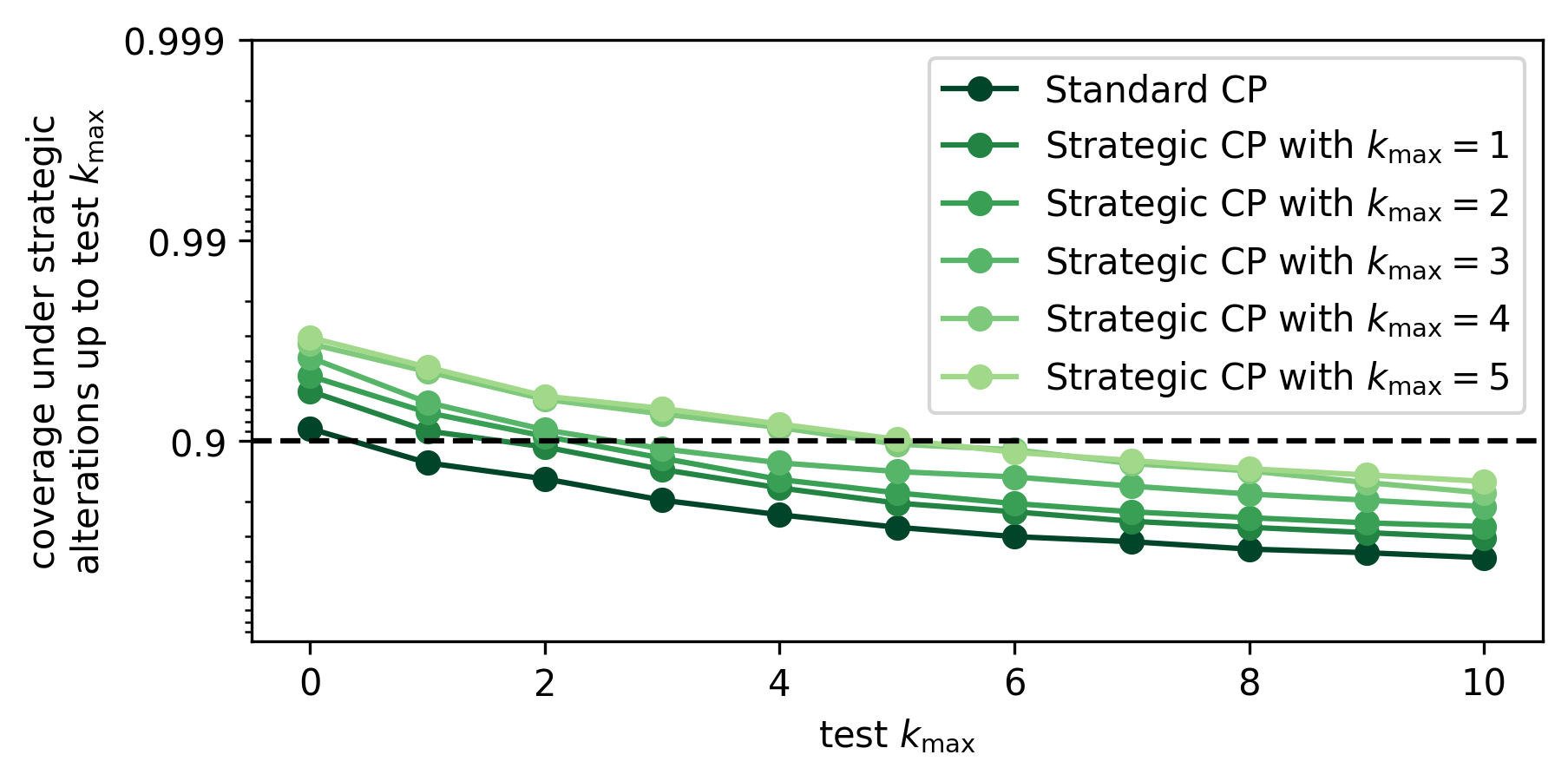}
        \caption{Strategic Conformal Prediction on top of a Strategic XGBoost Model}
    \end{subfigure}
    \caption{\texttt{academic-dropout}}
\end{figure}

\begin{figure}[H]
    \centering
    \begin{subfigure}[t]{0.48\textwidth}
        \centering
        \includegraphics[width=\columnwidth]{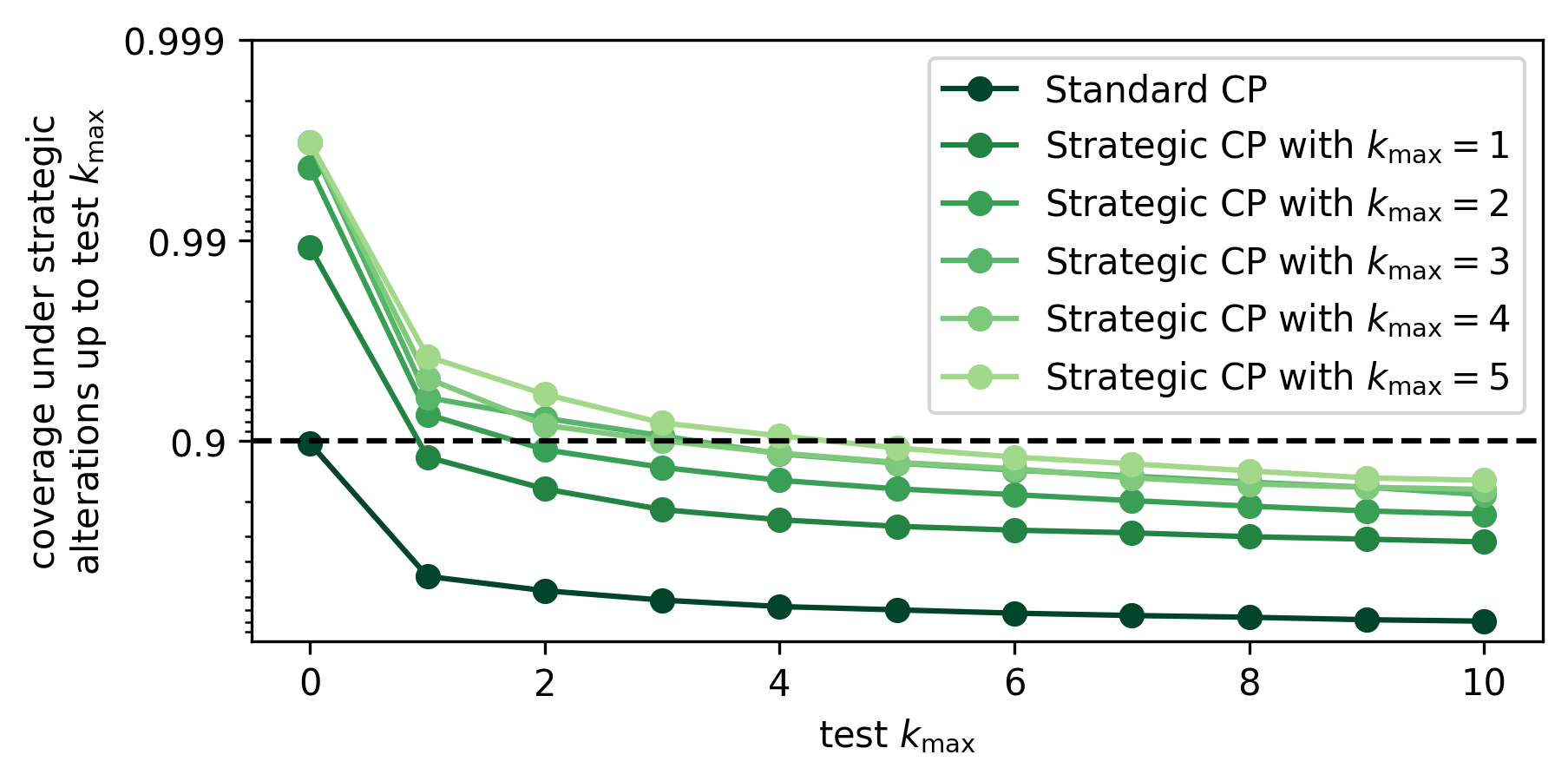}
        \caption{Strategic Conformal Prediction on top of a Plain \linebreak XGBoost Model}
    \end{subfigure}\ \quad
    \begin{subfigure}[t]{0.48\textwidth}
        \centering
        \includegraphics[width=\columnwidth]{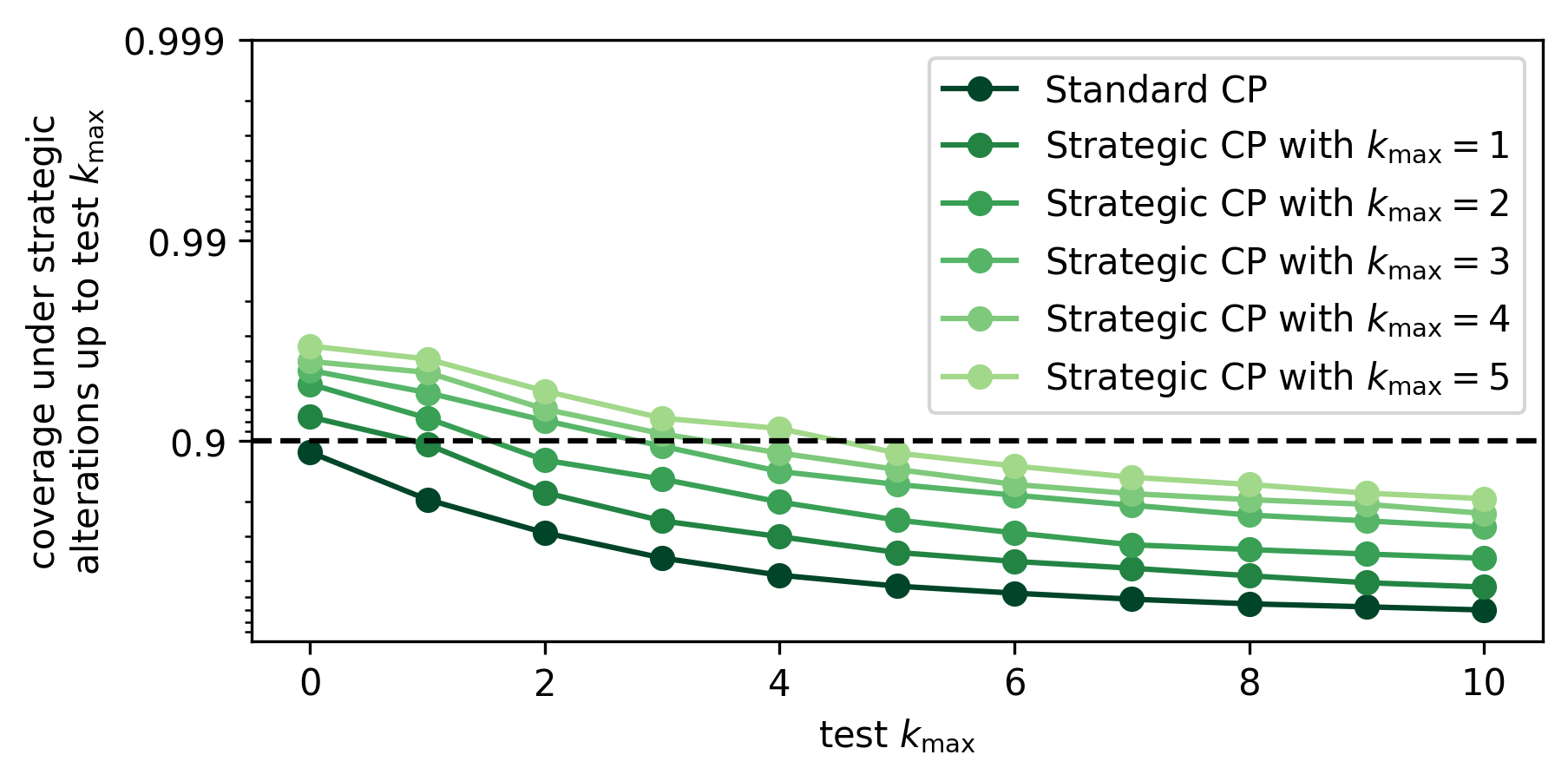}
        \caption{Strategic Conformal Prediction on top of a Strategic XGBoost Model}
    \end{subfigure}
    \caption{\texttt{spambase}}
\end{figure}

\begin{figure}[H]
    \centering
    \begin{subfigure}[t]{0.48\textwidth}
        \centering
        \includegraphics[width=\columnwidth]{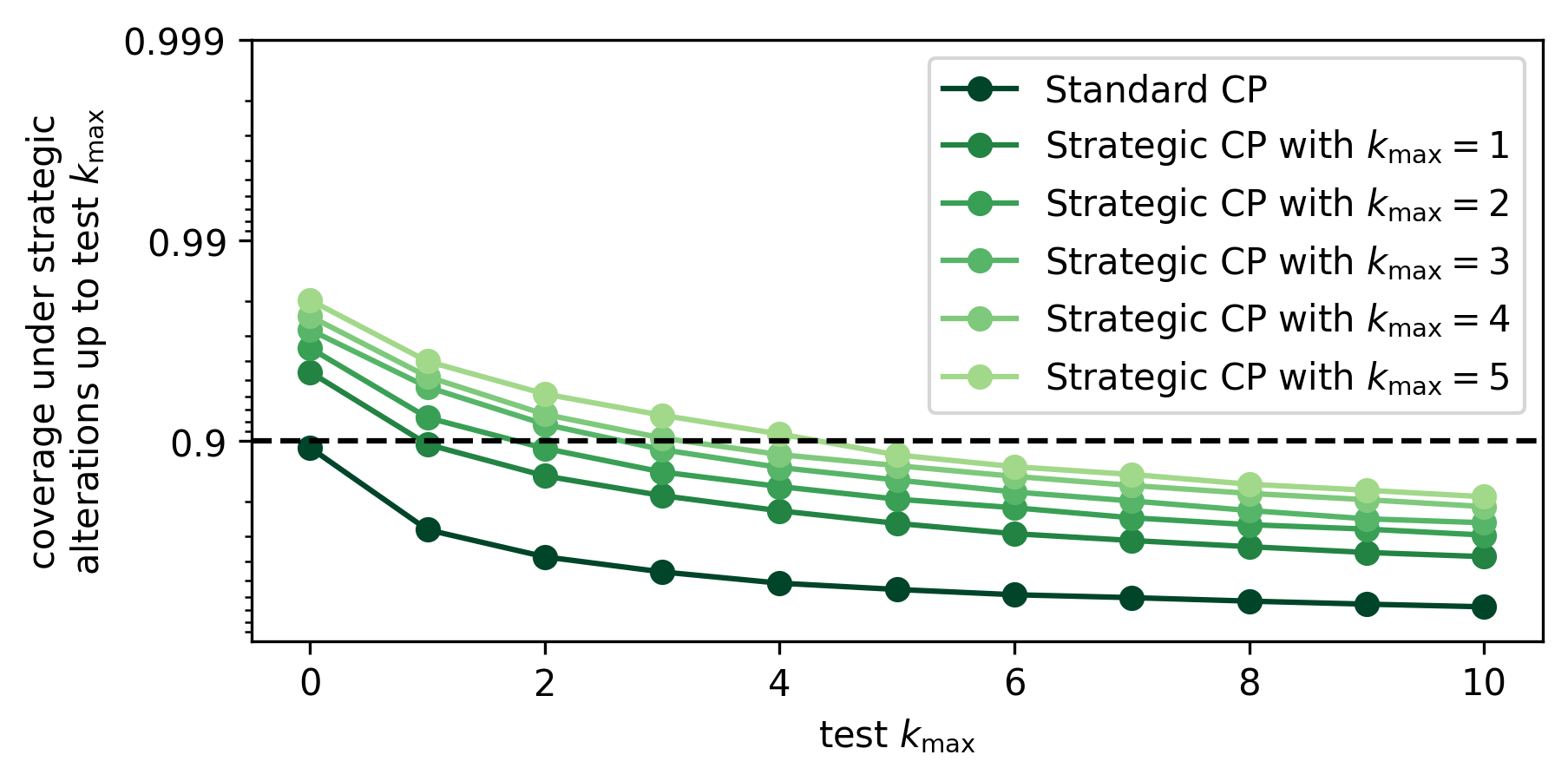}
        \caption{Strategic Conformal Prediction on top of a Plain \linebreak XGBoost Model}
    \end{subfigure}\ \quad
    \begin{subfigure}[t]{0.48\textwidth}
        \centering
        \includegraphics[width=\columnwidth]{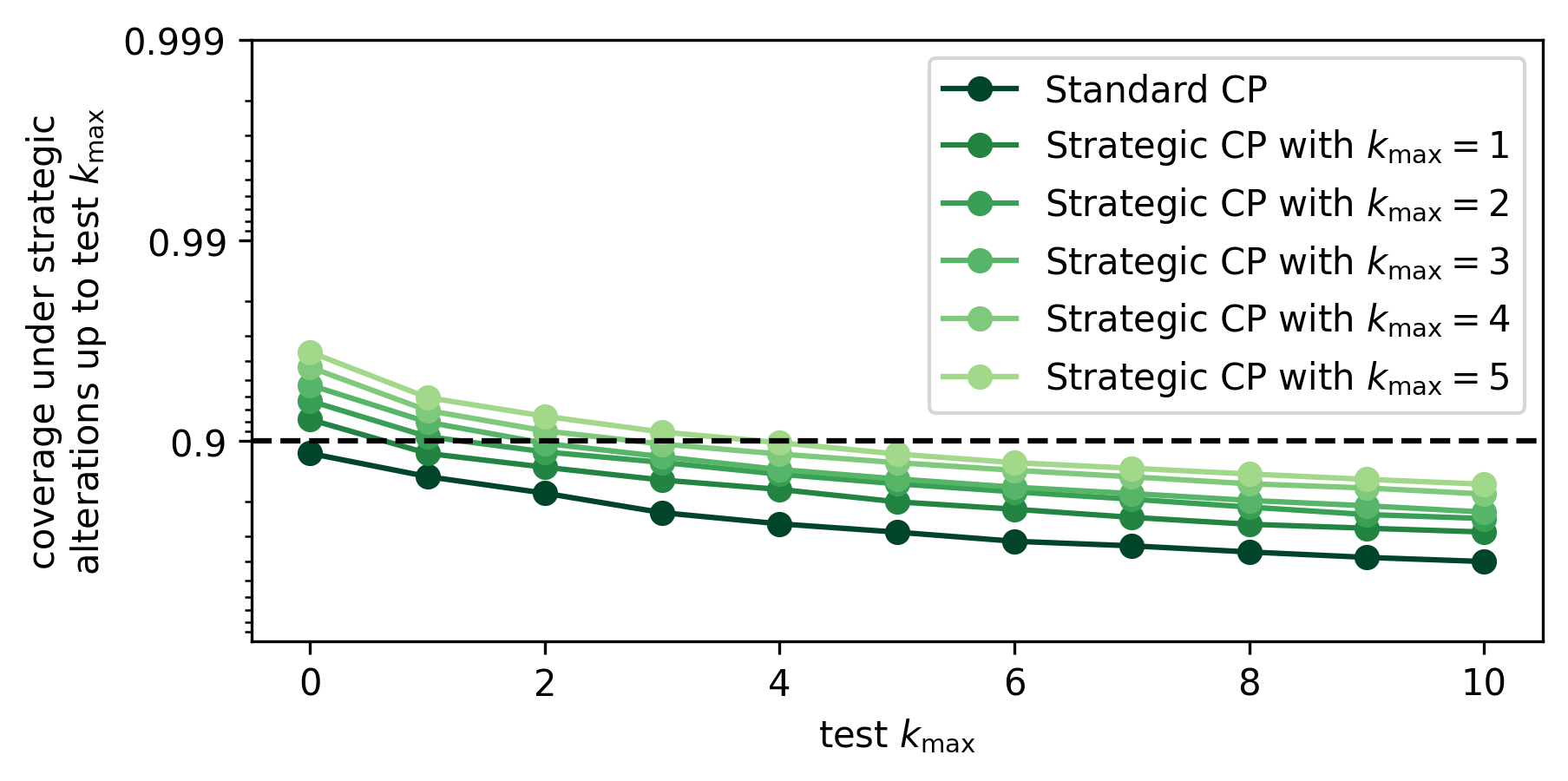}
        \caption{Strategic Conformal Prediction on top of a Strategic XGBoost Model}
    \end{subfigure}
    \caption{\texttt{shoppers}}
\end{figure}

\begin{figure}[H]
    \centering
    \begin{subfigure}[t]{0.48\textwidth}
        \centering
        \includegraphics[width=\columnwidth]{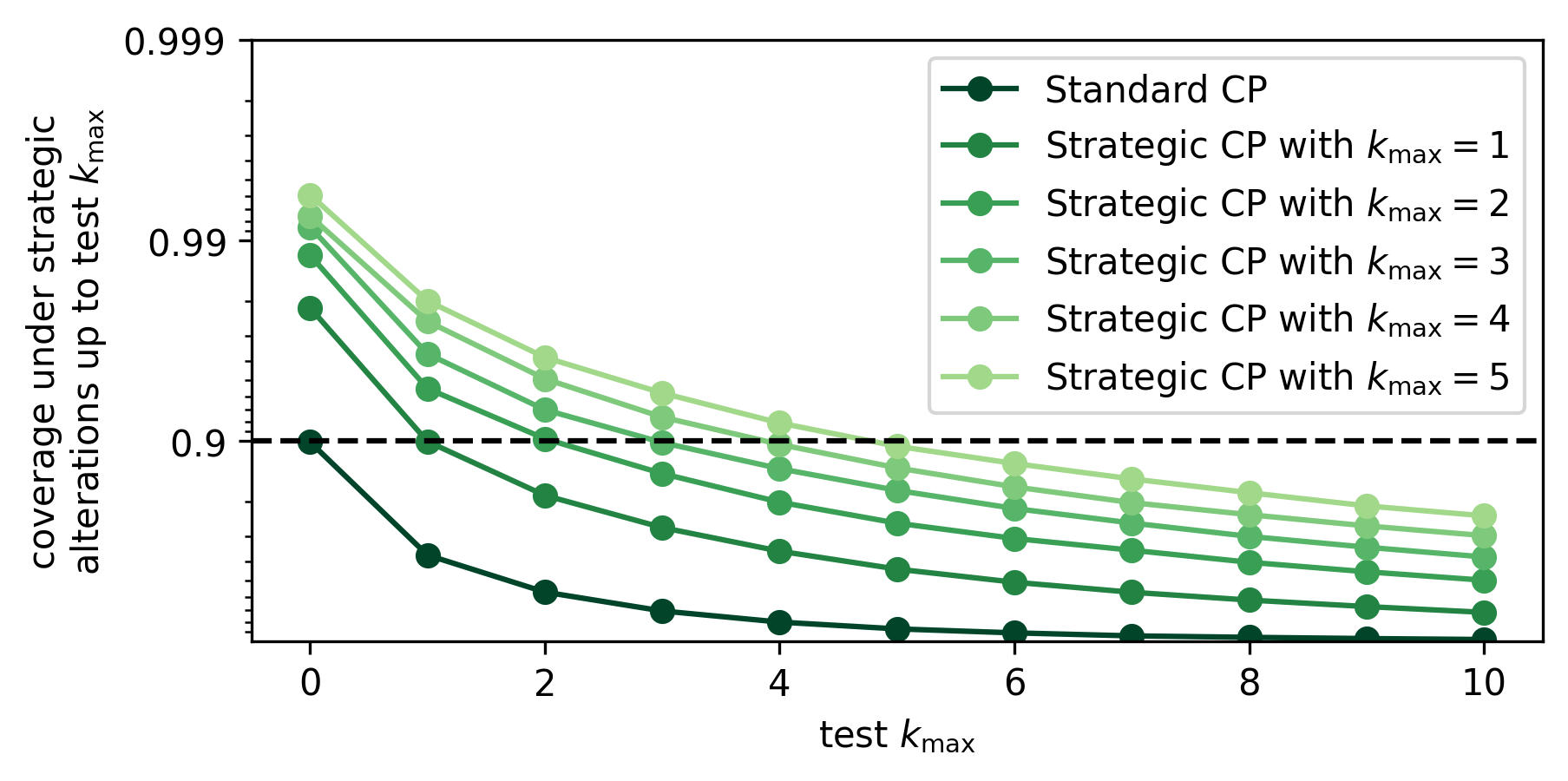}
        \caption{Strategic Conformal Prediction on top of a Plain \linebreak XGBoost Model}
    \end{subfigure}\ \quad
    \begin{subfigure}[t]{0.48\textwidth}
        \centering
        \includegraphics[width=\columnwidth]{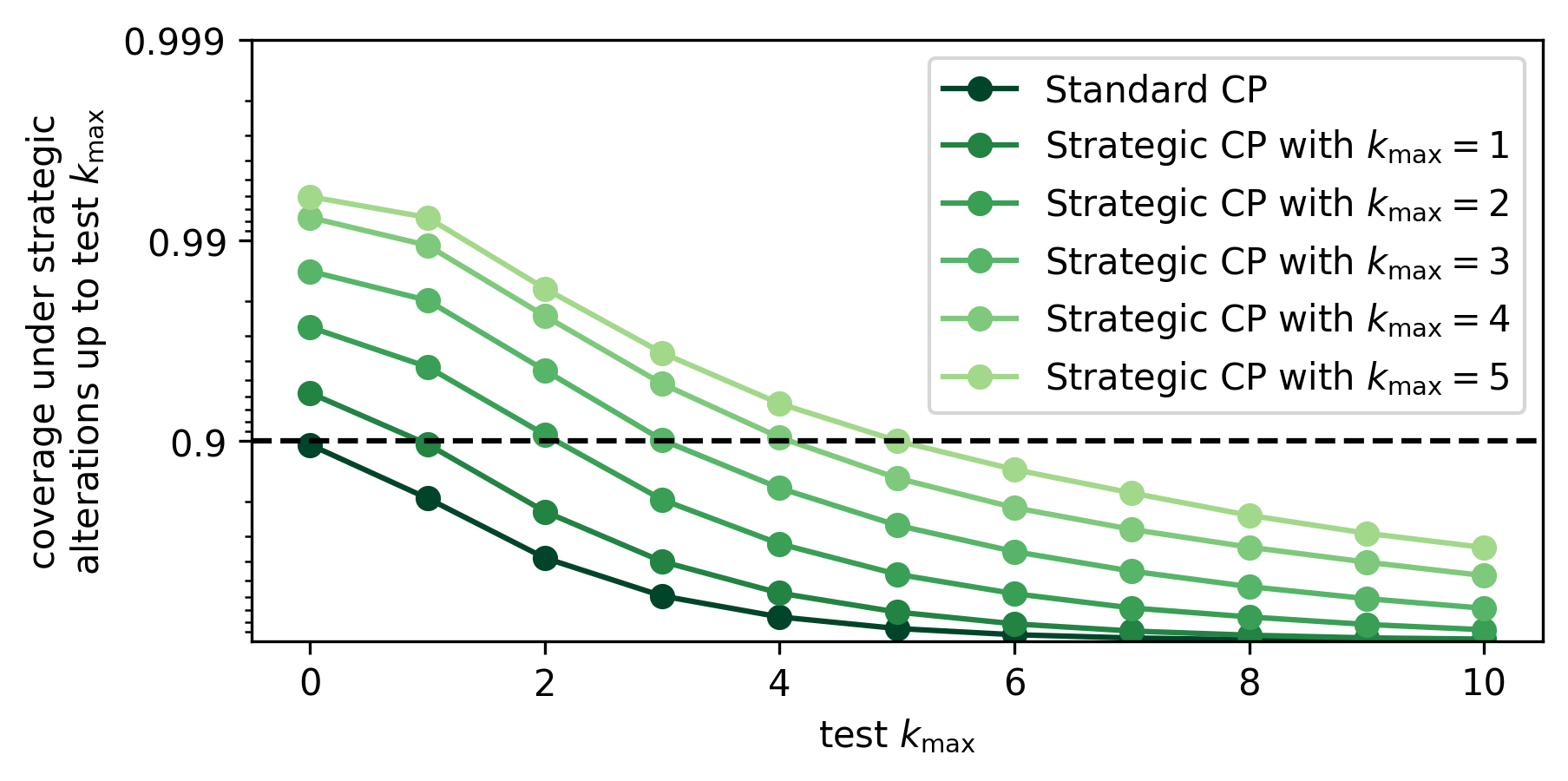}
        \caption{Strategic Conformal Prediction on top of a Strategic XGBoost Model}
    \end{subfigure}
    \caption{\texttt{news}}
\end{figure}

\begin{figure}[H]
    \centering
    \begin{subfigure}[t]{0.48\textwidth}
        \centering
        \includegraphics[width=\columnwidth]{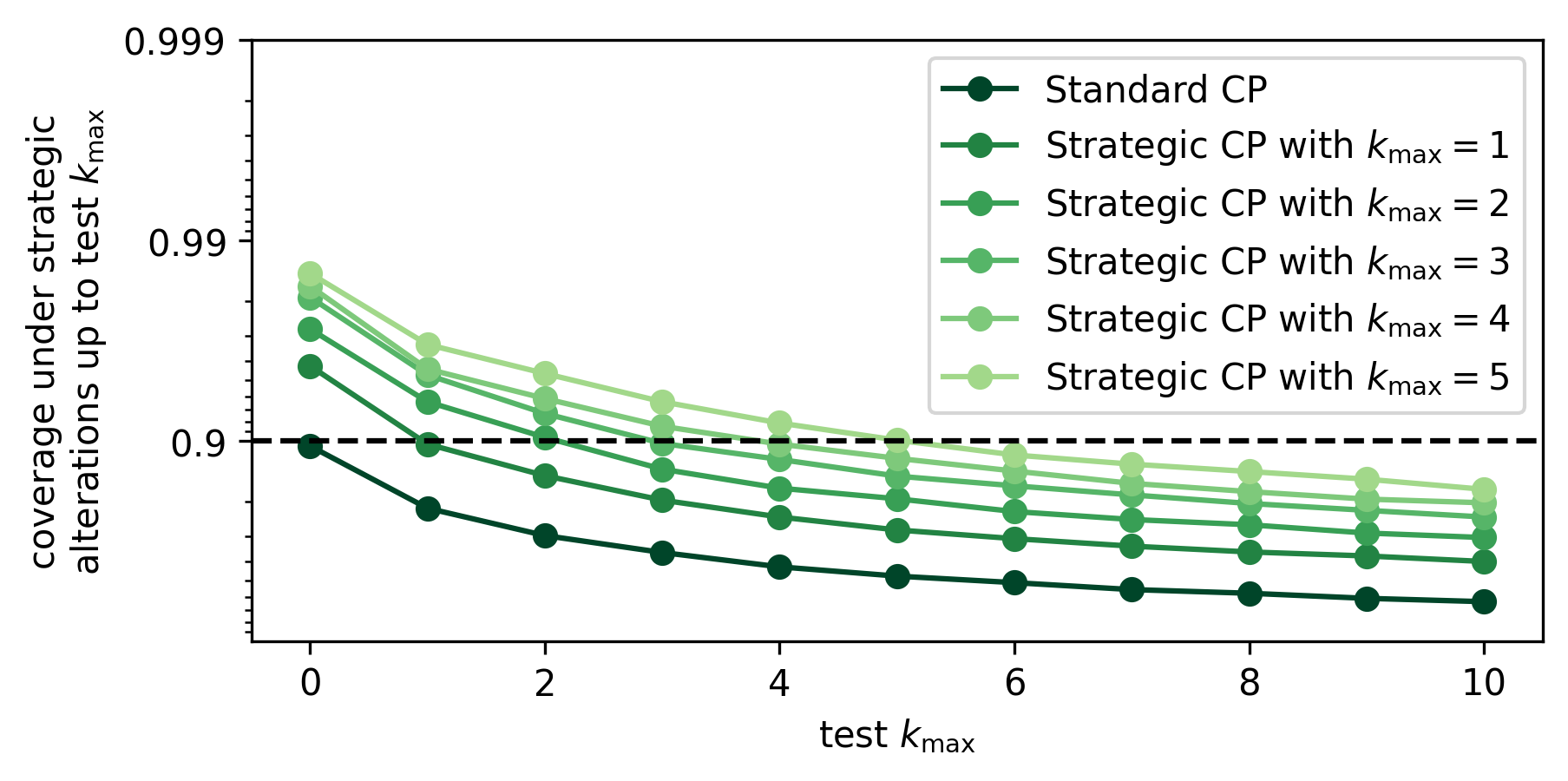}
        \caption{Strategic Conformal Prediction on top of a Plain \linebreak XGBoost Model}
    \end{subfigure}\ \quad
    \begin{subfigure}[t]{0.48\textwidth}
        \centering
        \includegraphics[width=\columnwidth]{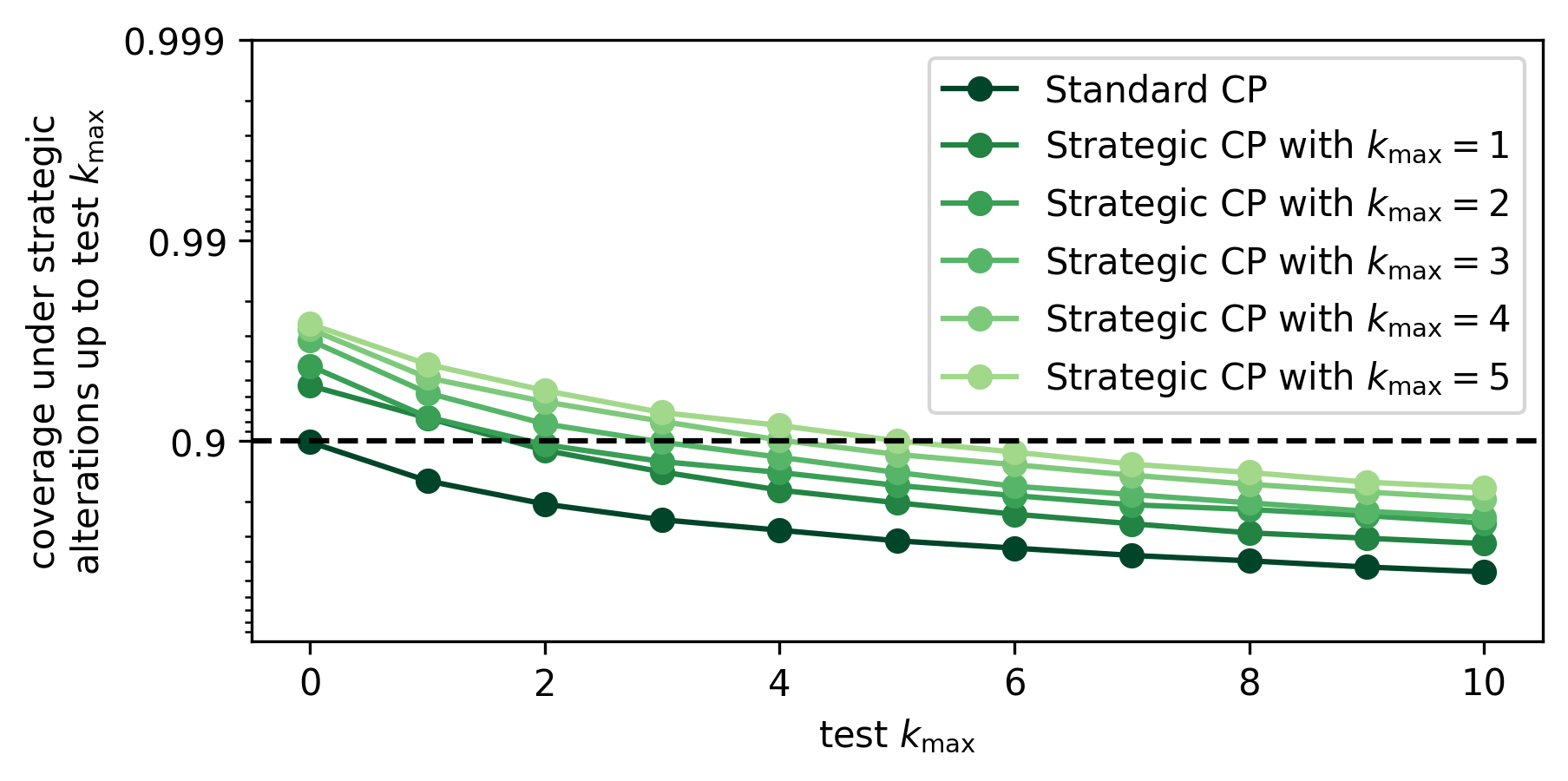}
        \caption{Strategic Conformal Prediction on top of a Strategic XGBoost Model}
    \end{subfigure}
    \caption{\texttt{wine}}
\end{figure}

\begin{figure}[H]
    \centering
    \begin{subfigure}[t]{0.48\textwidth}
        \centering
        \includegraphics[width=\columnwidth]{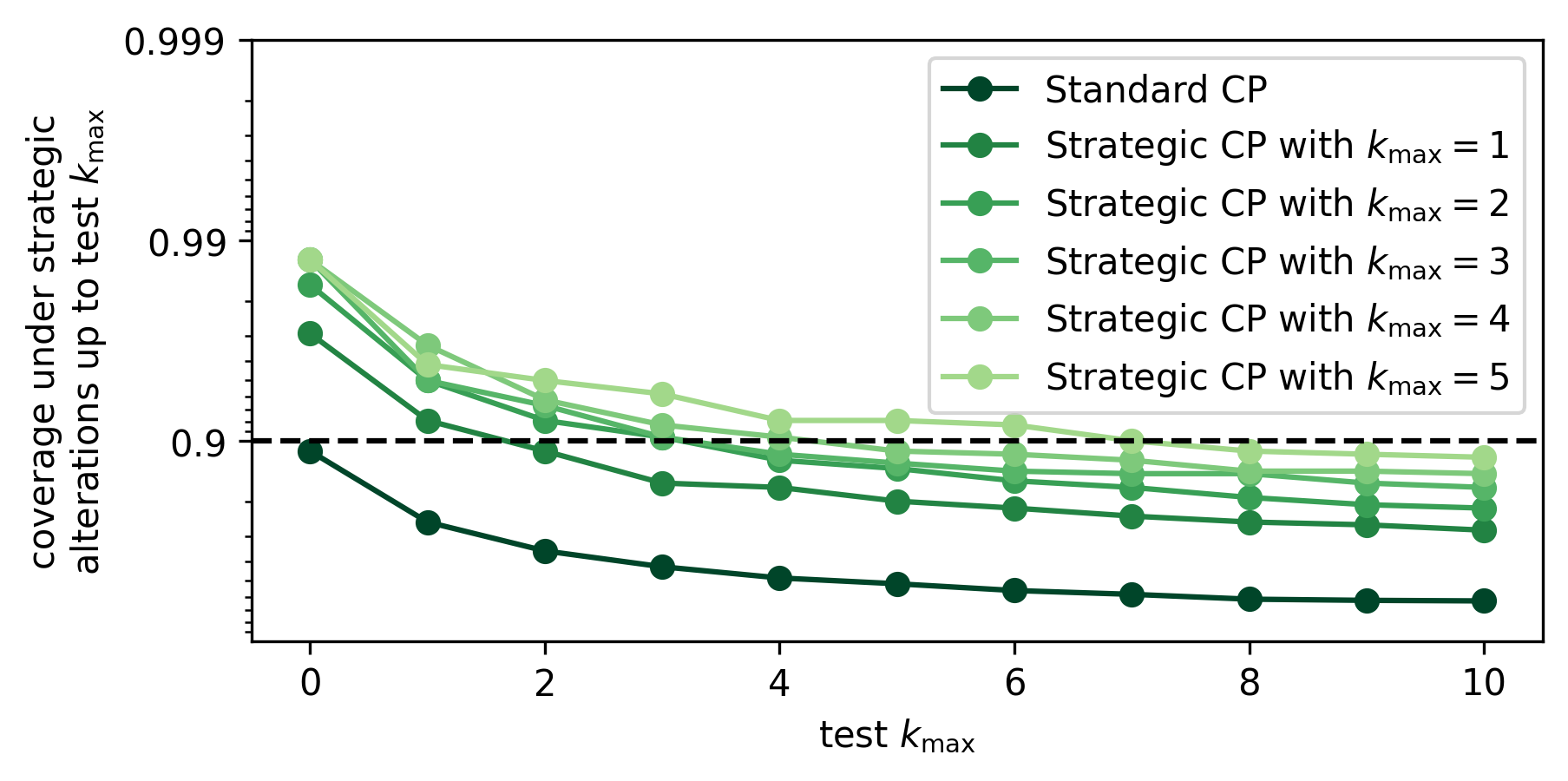}
        \caption{Strategic Conformal Prediction on top of a Plain \linebreak XGBoost Model}
    \end{subfigure}\ \quad
    \begin{subfigure}[t]{0.48\textwidth}
        \centering
        \includegraphics[width=\columnwidth]{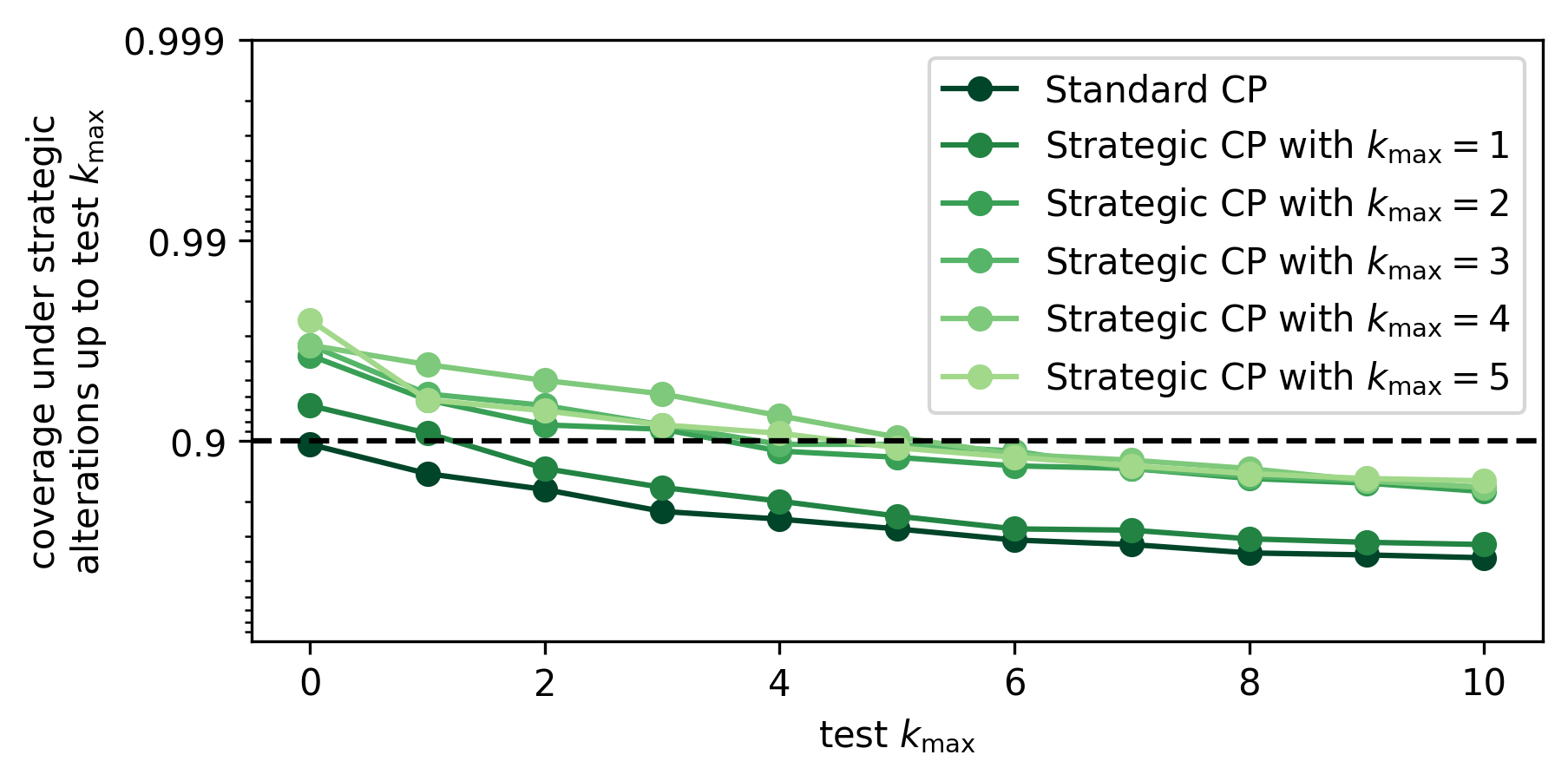}
        \caption{Strategic Conformal Prediction on top of a Strategic XGBoost Model}
    \end{subfigure}
    \caption{\texttt{productivity}}
\end{figure}

\begin{figure}[H]
    \centering
    \begin{subfigure}[t]{0.48\textwidth}
        \centering
        \includegraphics[width=\columnwidth]{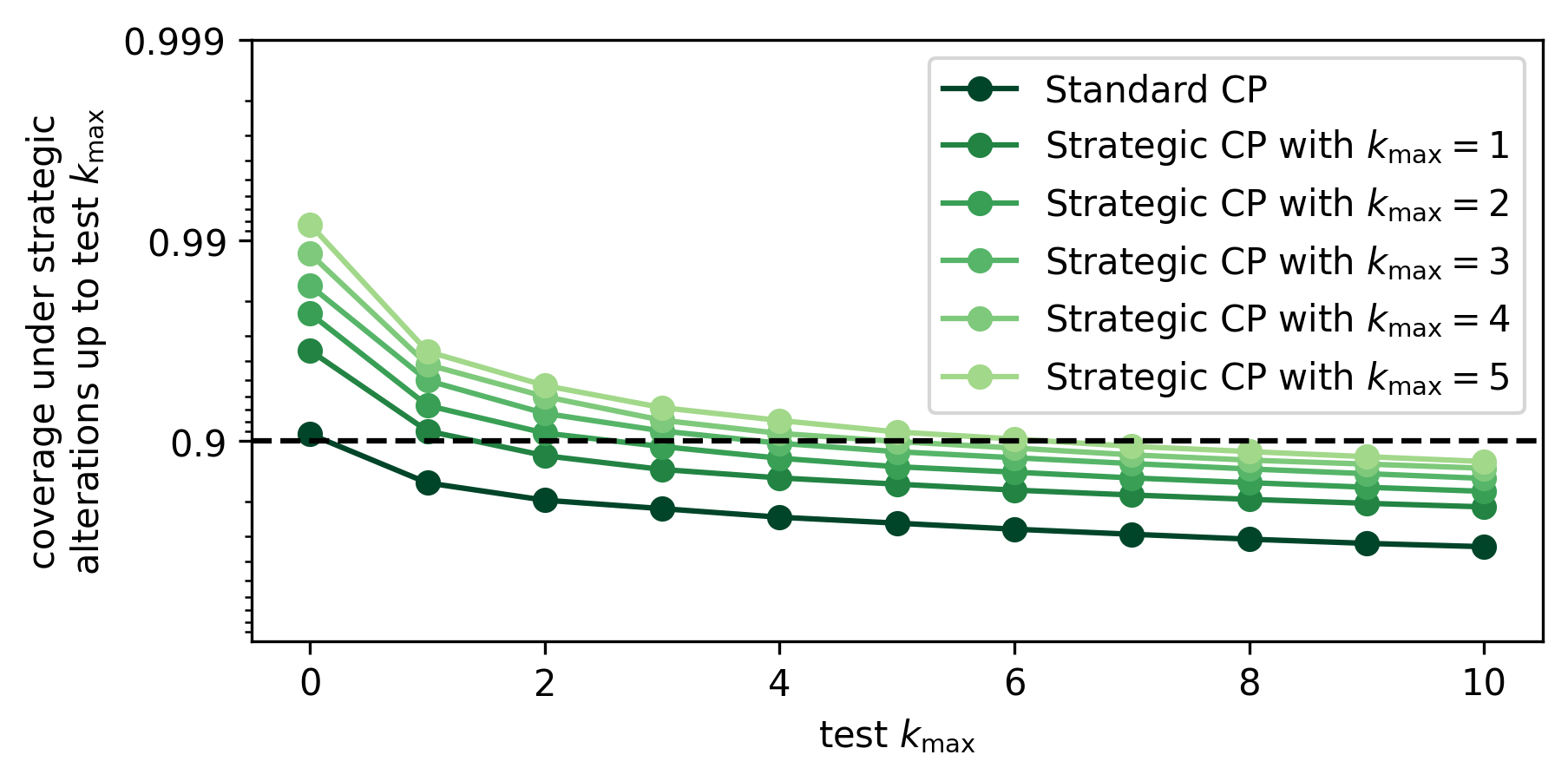}
        \caption{Strategic Conformal Prediction on top of a Plain \linebreak XGBoost Model}
    \end{subfigure}\ \quad
    \begin{subfigure}[t]{0.48\textwidth}
        \centering
        \includegraphics[width=\columnwidth]{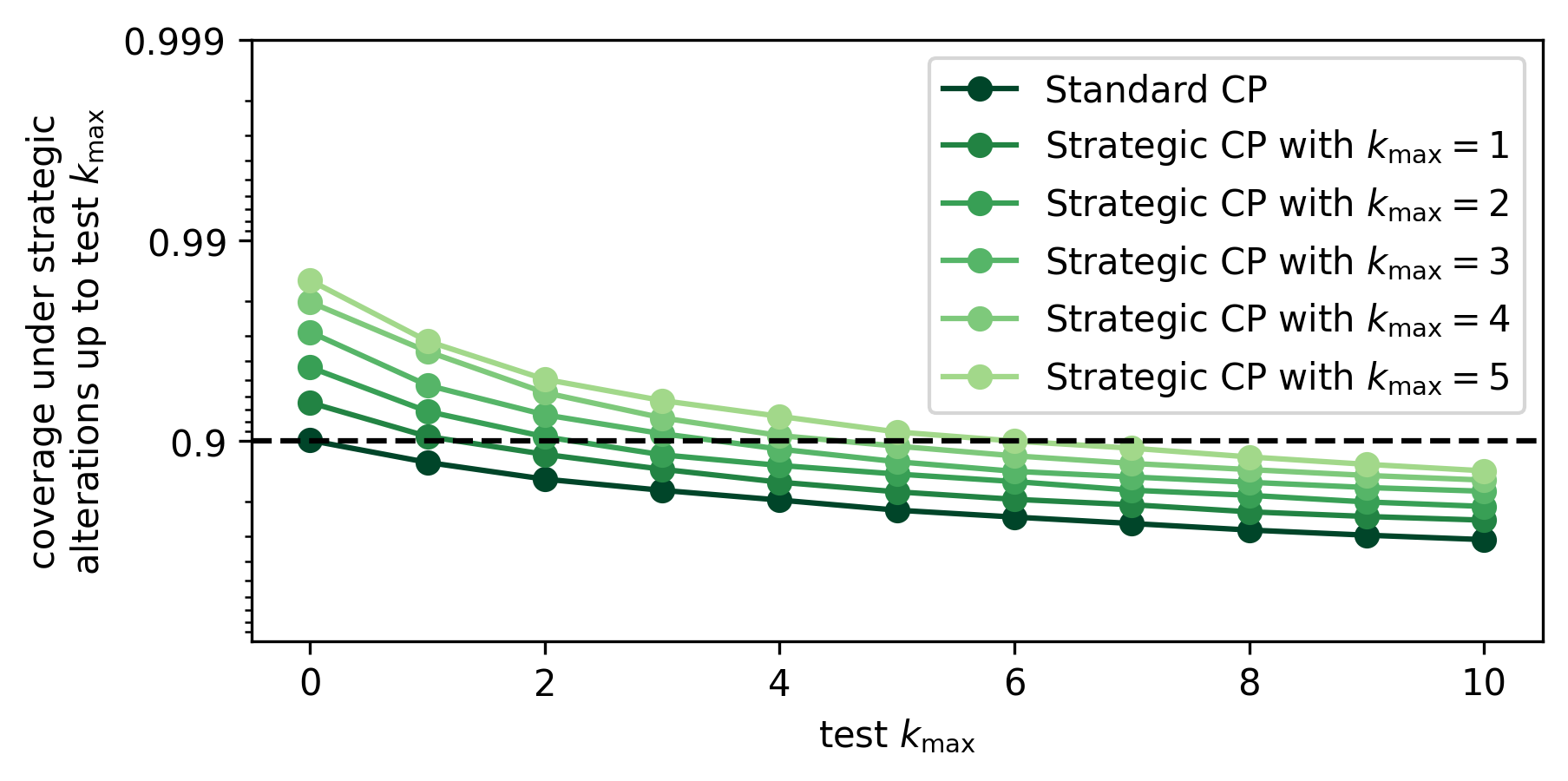}
        \caption{Strategic Conformal Prediction on top of a Strategic XGBoost Model}
    \end{subfigure}
    \caption{\texttt{taiwan}}
\end{figure}

\begin{figure}[H]
    \centering
    \begin{subfigure}[t]{0.48\textwidth}
        \centering
        \includegraphics[width=\columnwidth]{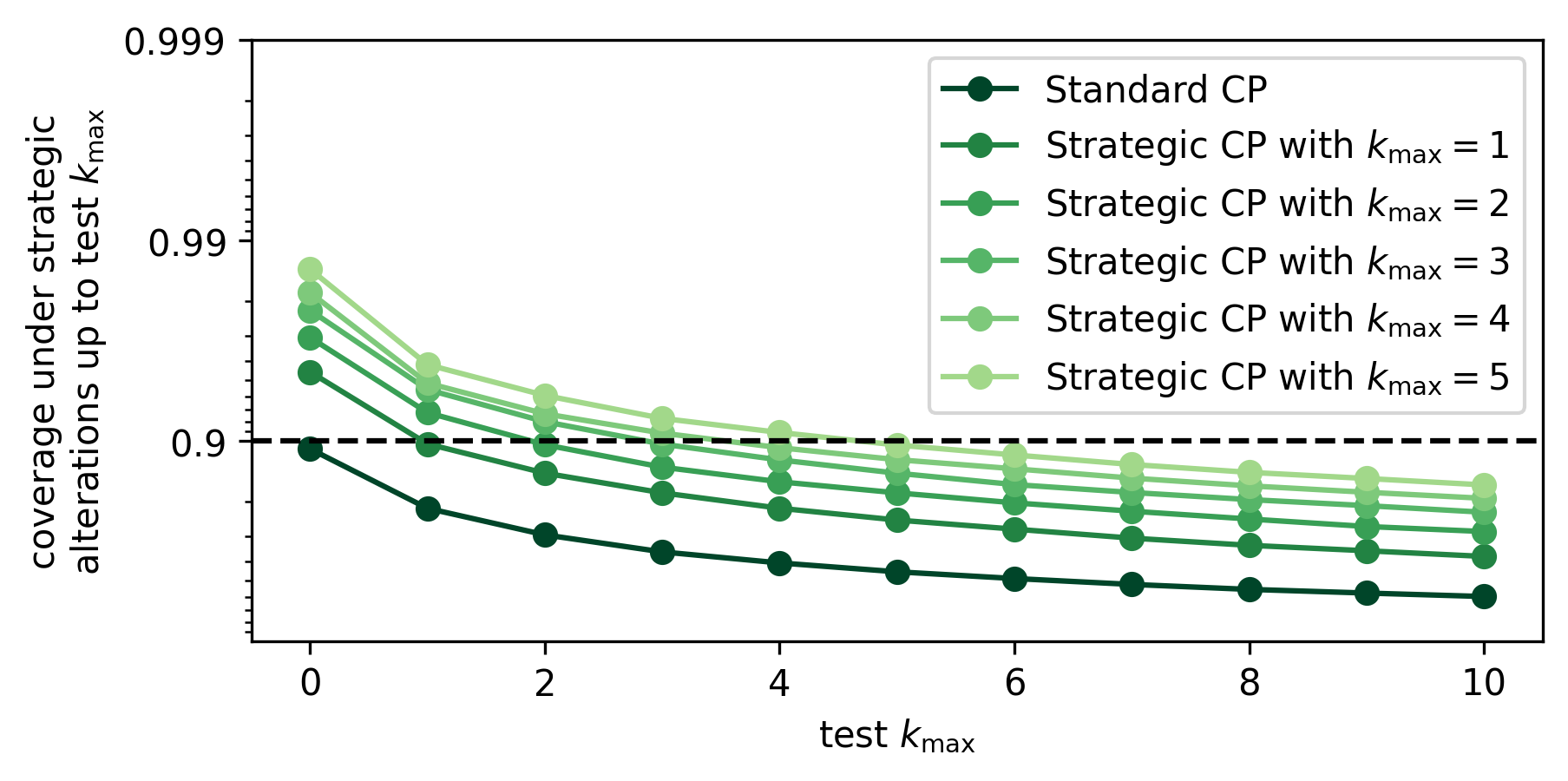}
        \caption{Strategic Conformal Prediction on top of a Plain \linebreak XGBoost Model}
    \end{subfigure}\ \quad
    \begin{subfigure}[t]{0.48\textwidth}
        \centering
        \includegraphics[width=\columnwidth]{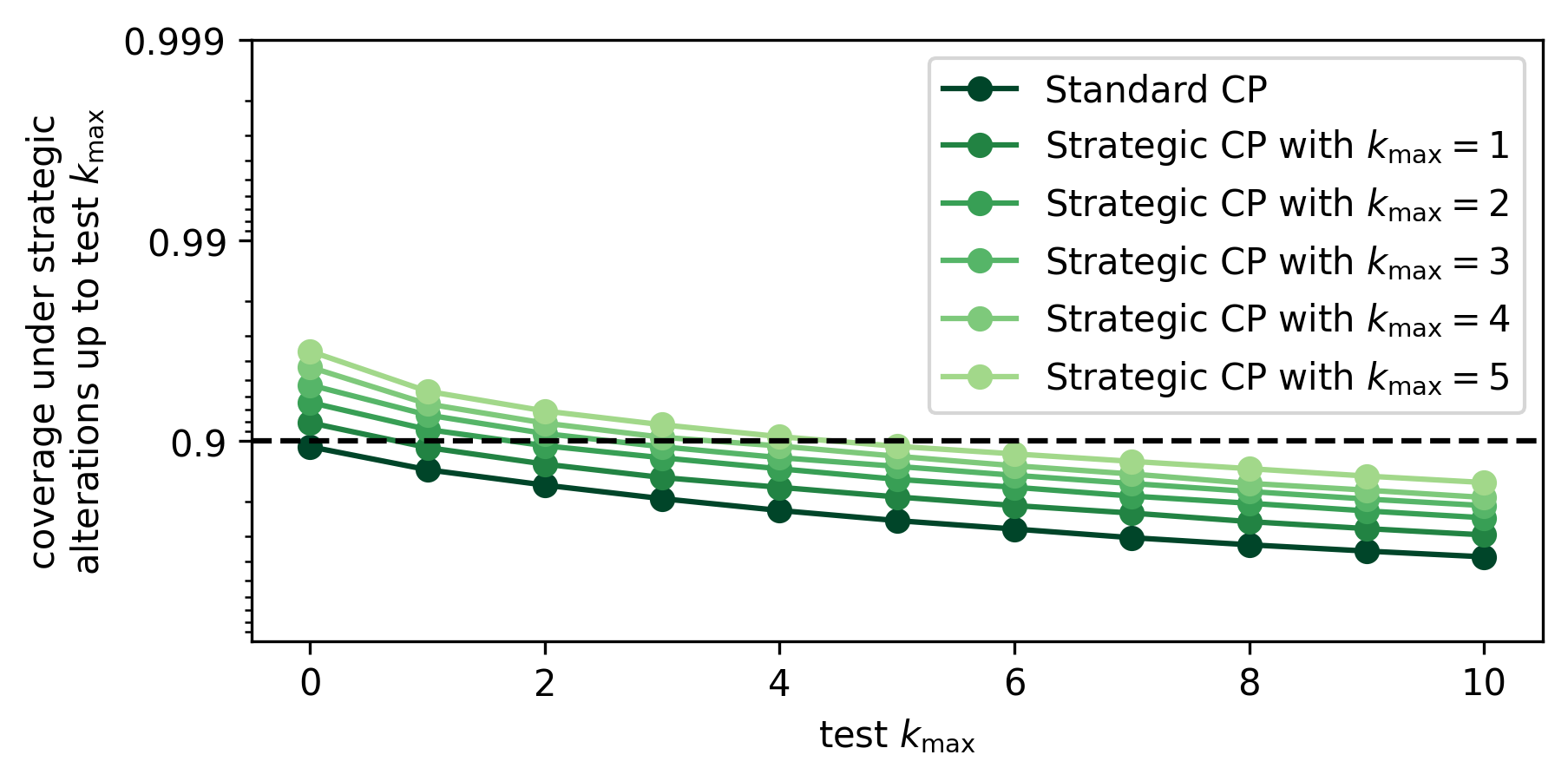}
        \caption{Strategic Conformal Prediction on top of a Strategic XGBoost Model}
    \end{subfigure}
    \caption{\texttt{bank-marketing}}
\end{figure}

\begin{figure}[H]
    \centering
    \begin{subfigure}[t]{0.48\textwidth}
        \centering
        \includegraphics[width=\columnwidth]{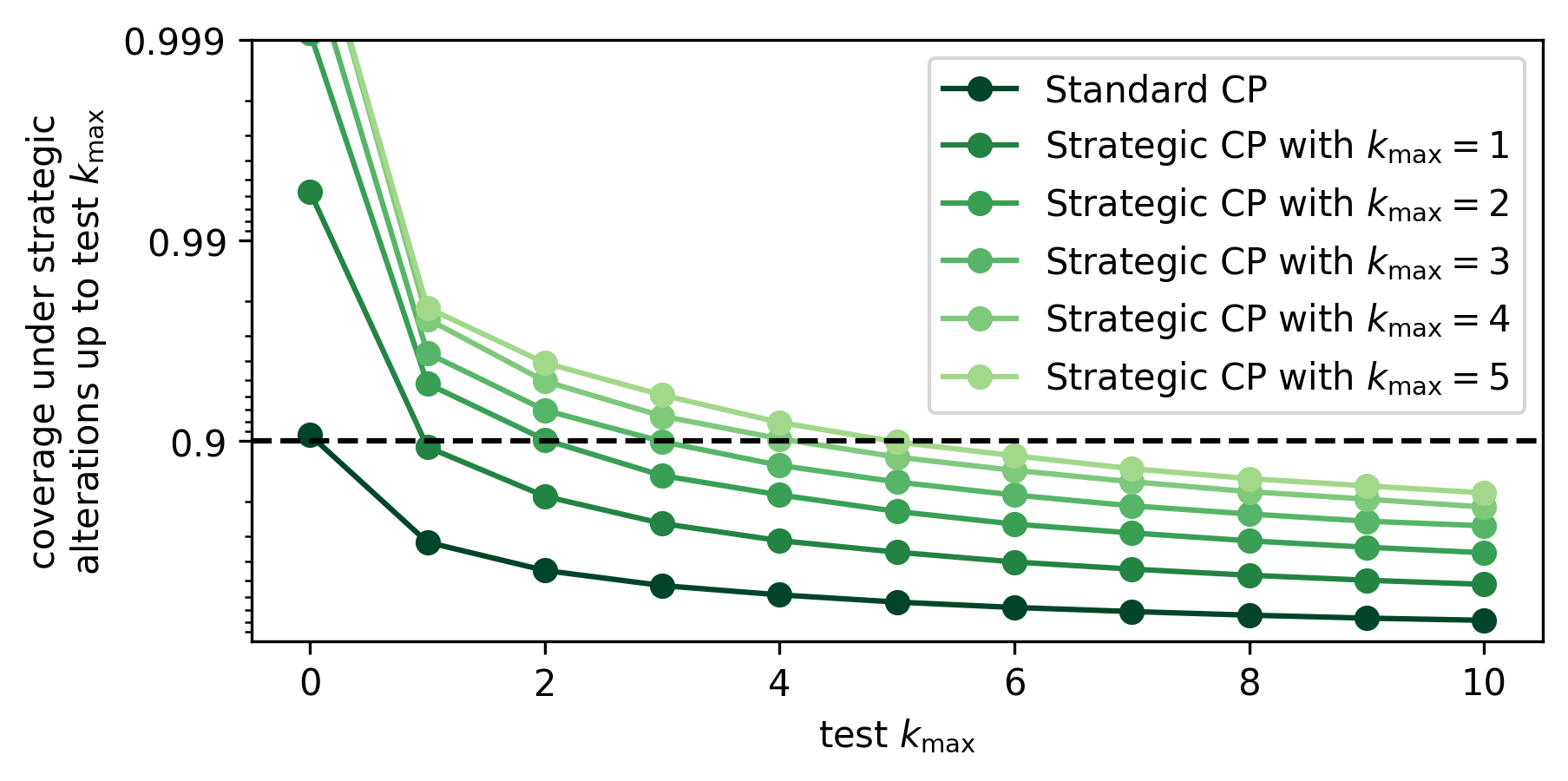}
        \caption{Strategic Conformal Prediction on top of a Plain \linebreak XGBoost Model}
    \end{subfigure}\ \quad
    \begin{subfigure}[t]{0.48\textwidth}
        \centering
        \includegraphics[width=\columnwidth]{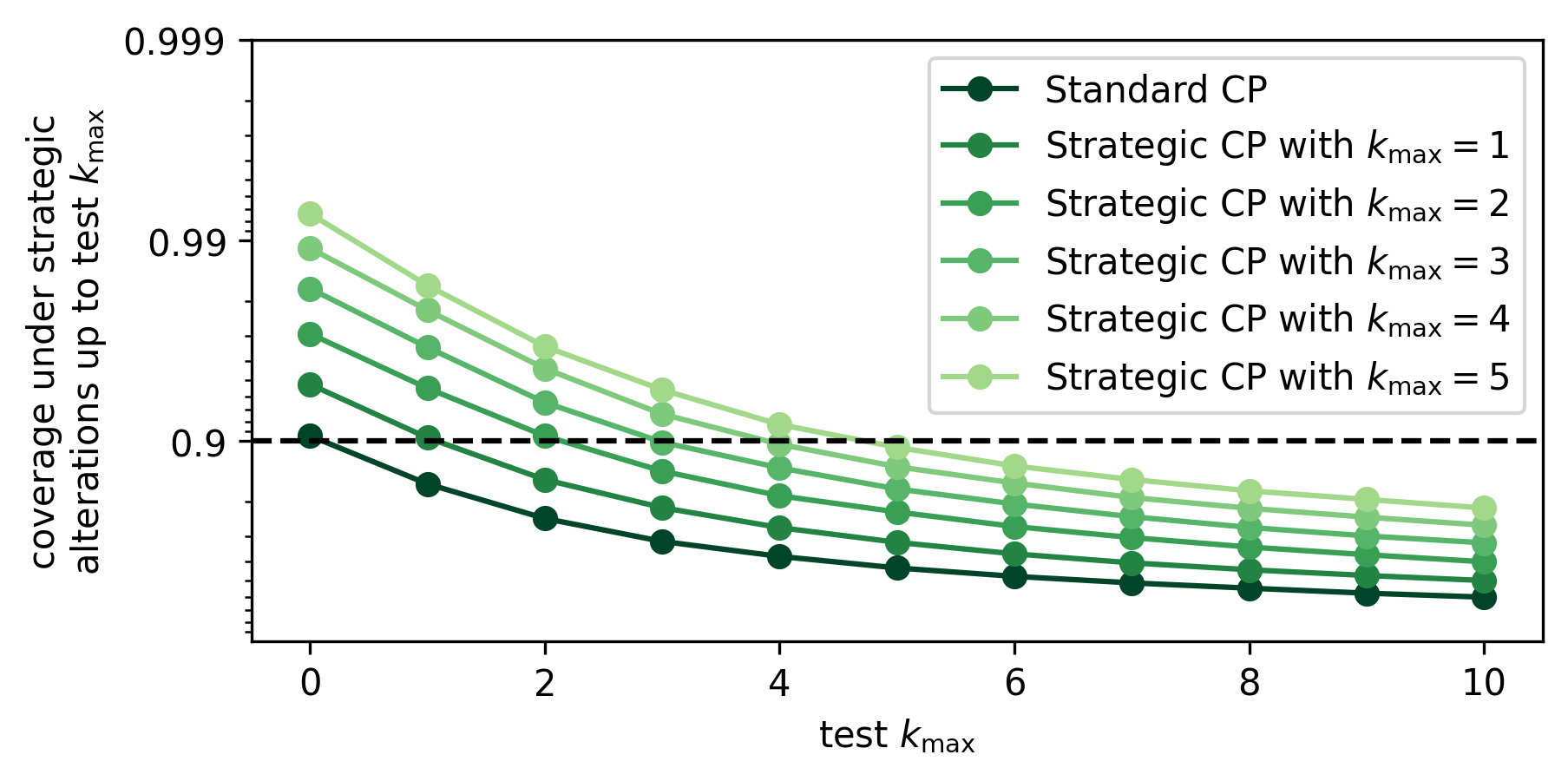}
        \caption{Strategic Conformal Prediction on top of a Strategic XGBoost Model}
    \end{subfigure}
    \caption{\texttt{census-income}}
\end{figure}

\subsection{Table 1: evaluation of our method on multiple datasets and forms of strategic alterations for $\alpha = 0.1$}

\begin{table}[H]
    \caption{Extended version of Table~1. Note that the news dataset has an average set size difference on another order of magnitude only due to the labels of its regression task being of much larger magnitude as well.}
    \centering
    \begin{tabular}{ccc cc c}
        & \multirow[b]{2}{*}{ \textbf{\shortstack{UNDERLYING \\ MODEL}}\vspace{-0.85em} }
        &
        & \multicolumn{2}{c}{ \textbf{\shortstack{STRATEGIC COVERAGE}} }
        & \multirow[b]{2}{*}{ \textbf{\shortstack{AVG SET \\ SIZE DIFF}}\vspace{-0.85em} }
        \\\cmidrule(lr){4-5}
        \textbf{DATASET}
        &
        & \textbf{$\Delta$s}
        & \textbf{OURS}
        & \textbf{STD CP}
        &
        \\
        \hline

        \multirow{4}{*}{\texttt{academic-dropout}}
        & \multirow{2}{*}{Plain XGBoost}
        & Utility-cost
        & $\mathbf{90\% \pm 2\%}$ & \textcolor{purple}{$\mathbf{57\% \pm 3\%}$} & $+1.08 \pm 0.05$ \\
        &
        & Rand. Search.
        & $\mathbf{92\% \pm 2\%}$ & \textcolor{purple}{$\mathbf{49\% \pm 3\%}$} & $+1.40 \pm 0.04$ \\
        & \multirow{2}{*}{Strategic XGBoost}
        & Utility-cost
        & $\mathbf{91\% \pm 2\%}$ & $76\% \pm 3\%$ & $+0.52 \pm 0.04$ \\
        &
        & Rand. Search.
        & $\mathbf{92\% \pm 2\%}$ & $75\% \pm 3\%$ & $+0.60 \pm 0.05$ \\[.15cm]

        \multirow{4}{*}{\texttt{spambase}}
        & \multirow{2}{*}{Plain XGBoost}
        & Utility-cost
        & $\mathbf{91\% \pm 2\%}$ & \textcolor{purple}{$\mathbf{15\% \pm 2\%}$} & $+1.06 \pm 0.04$ \\
        &
        & Rand. Search.
        & $\mathbf{87\% \pm 2\%}$ & \textcolor{purple}{$\mathbf{29\% \pm 3\%}$} & $+0.77 \pm 0.04$ \\
        & \multirow{2}{*}{Strategic XGBoost}
        & Utility-cost
        & $\mathbf{90\% \pm 2\%}$ & \textcolor{purple}{$\mathbf{40\% \pm 3\%}$} & $+0.38 \pm 0.03$ \\
        &
        & Rand. Search.
        & $\mathbf{89\% \pm 2\%}$ & \textcolor{purple}{$\mathbf{46\% \pm 3\%}$} & $+0.53 \pm 0.04$ \\[.15cm]

        \multirow{4}{*}{\texttt{shoppers}}
        & \multirow{2}{*}{Plain XGBoost}
        & Utility-cost
        & $\mathbf{89\% \pm 1\%}$ & \textcolor{purple}{$\mathbf{44\% \pm 2\%}$} & $+0.14 \pm 0.01$ \\
        &
        & Rand. Search.
        & $\mathbf{88\% \pm 1\%}$ & \textcolor{purple}{$\mathbf{47\% \pm 2\%}$} & $+0.13 \pm 0.01$ \\
        & \multirow{2}{*}{Strategic XGBoost}
        & Utility-cost
        & $\mathbf{89\% \pm 1\%}$ & $64\% \pm 2\%$ & $+0.08 \pm 0.01$ \\
        &
        & Rand. Search.
        & $\mathbf{88\% \pm 1\%}$ & $72\% \pm 2\%$ & $+0.08 \pm 0.01$ \\[.15cm]

        \multirow{4}{*}{\texttt{news}}
        & \multirow{2}{*}{Plain XGBoost}
        & Utility-cost
        & $\mathbf{90\% \pm 1\%}$ & \textcolor{purple}{$\mathbf{6\% \pm 1\%}$} & $+131331.63 \pm 0.00$ \\
        &
        & Rand. Search.
        & $\mathbf{91\% \pm 1\%}$ & \textcolor{purple}{$\mathbf{13\% \pm 1\%}$} & $+96299.98 \pm 0.00$ \\
        & \multirow{2}{*}{Strategic XGBoost}
        & Utility-cost
        & $\mathbf{89\% \pm 1\%}$ & \textcolor{purple}{$\mathbf{6\% \pm 0\%}$} & $+75250.51 \pm 00.00$ \\
        &
        & Rand. Search.
        & $\mathbf{89\% \pm 1\%}$ & \textcolor{purple}{$\mathbf{14\% \pm 1\%}$} & $+79403 \pm 0.00$ \\[.15cm]

        \multirow{4}{*}{\texttt{wine}}
        & \multirow{2}{*}{Plain XGBoost}
        & Utility-cost
        & $\mathbf{90\% \pm 2\%}$ & \textcolor{purple}{$\mathbf{49\% \pm 3\%}$} & $+1.63 \pm 0.00$ \\
        &
        & Rand. Search.
        & $\mathbf{90\% \pm 1\%}$ & \textcolor{purple}{$\mathbf{53\% \pm 3\%}$} & $+1.57 \pm 0.00$ \\
        & \multirow{2}{*}{Strategic XGBoost}
        & Utility-cost
        & $\mathbf{90\% \pm 2\%}$ & $61\% \pm 3\%$ & $+1.86 \pm 0.00$ \\
        &
        & Rand. Search.
        & $\mathbf{90\% \pm 2\%}$ & $68\% \pm 3\%$ & $+1.08 \pm 0.00$ \\[.15cm]

        \multirow{4}{*}{\texttt{productivity}}
        & \multirow{2}{*}{Plain XGBoost}
        & Utility-cost
        & $\mathbf{91\% \pm 3\%}$ & \textcolor{purple}{$\mathbf{38\% \pm 6\%}$} & $+0.48 \pm 0.00$ \\
        &
        & Rand. Search.
        & $\mathbf{91\% \pm 3\%}$ & \textcolor{purple}{$\mathbf{42\% \pm 6\%}$} & $+0.43 \pm 0.00$ \\
        & \multirow{2}{*}{Strategic XGBoost}
        & Utility-cost
        & $\mathbf{91\% \pm 3\%}$ & \textcolor{purple}{$\mathbf{45\% \pm 7\%}$} & $+0.42 \pm 0.00$ \\
        &
        & Rand. Search.
        & $\mathbf{89\% \pm 4\%}$ & $74\% \pm 5\%$ & $+0.24 \pm 0.00$ \\[.15cm]

        \multirow{4}{*}{\texttt{taiwan}}
        & \multirow{2}{*}{Plain XGBoost}
        & Utility-cost
        & $\mathbf{90\% \pm 1\%}$ & $77\% \pm 1\%$ & $+0.46 \pm 0.01$ \\
        &
        & Rand. Search.
        & $\mathbf{91\% \pm 1\%}$ & $74\% \pm 1\%$ & $+0.41 \pm 0.01$ \\
        & \multirow{2}{*}{Strategic XGBoost}
        & Utility-cost
        & $\mathbf{90\% \pm 1\%}$ & $75\% \pm 1\%$ & $+0.39 \pm 0.01$ \\
        &
        & Rand. Search.
        & $\mathbf{91\% \pm 1\%}$ & $78\% \pm 1\%$ & $+0.32 \pm 0.01$ \\[.15cm]

        \multirow{4}{*}{\texttt{bank-marketing}}
        & \multirow{2}{*}{Plain XGBoost}
        & Utility-cost
        & $\mathbf{90\% \pm 1\%}$ & \textcolor{purple}{$\mathbf{52\% \pm 1\%}$} & $+0.23 \pm 0.01$ \\
        &
        & Rand. Search.
        & $\mathbf{89\% \pm 1\%}$ & \textcolor{purple}{$\mathbf{55\% \pm 1\%}$} & $+0.22 \pm 0.01$ \\
        & \multirow{2}{*}{Strategic XGBoost}
        & Utility-cost
        & $\mathbf{90\% \pm 1\%}$ & $73\% \pm 1\%$ & $+0.04 \pm 0.00$ \\
        &
        & Rand. Search.
        & $\mathbf{89\% \pm 1\%}$ & $76\% \pm 1\%$ & $+0.06 \pm 0.01$ \\[.15cm]

        \multirow{4}{*}{\texttt{census-income}}
        & \multirow{2}{*}{Plain XGBoost}
        & Utility-cost
        & $\mathbf{89\% \pm 1\%}$ & \textcolor{purple}{$\mathbf{37\% \pm 1\%}$} & $+1.65 \pm 0.02$ \\
        &
        & Rand. Search.
        & $\mathbf{90\% \pm 1\%}$ & \textcolor{purple}{$\mathbf{36\% \pm 1\%}$} & $+1.74 \pm 0.02$ \\
        & \multirow{2}{*}{Strategic XGBoost}
        & Utility-cost
        & $\mathbf{90\% \pm 1\%}$ & $60\% \pm 1\%$ & $+0.99 \pm 0.02$ \\
        &
        & Rand. Search.
        & $\mathbf{89\% \pm 1\%}$ & \textcolor{purple}{$\mathbf{57\% \pm 1\%}$} & $+1.03 \pm 0.02$
    \end{tabular}
    \vspace{-1em}
\end{table}

\subsection{Extra figure: coverage under change of the stochastic step}

\begin{figure}[H]
    \centering
    \includegraphics[width=0.5\columnwidth]{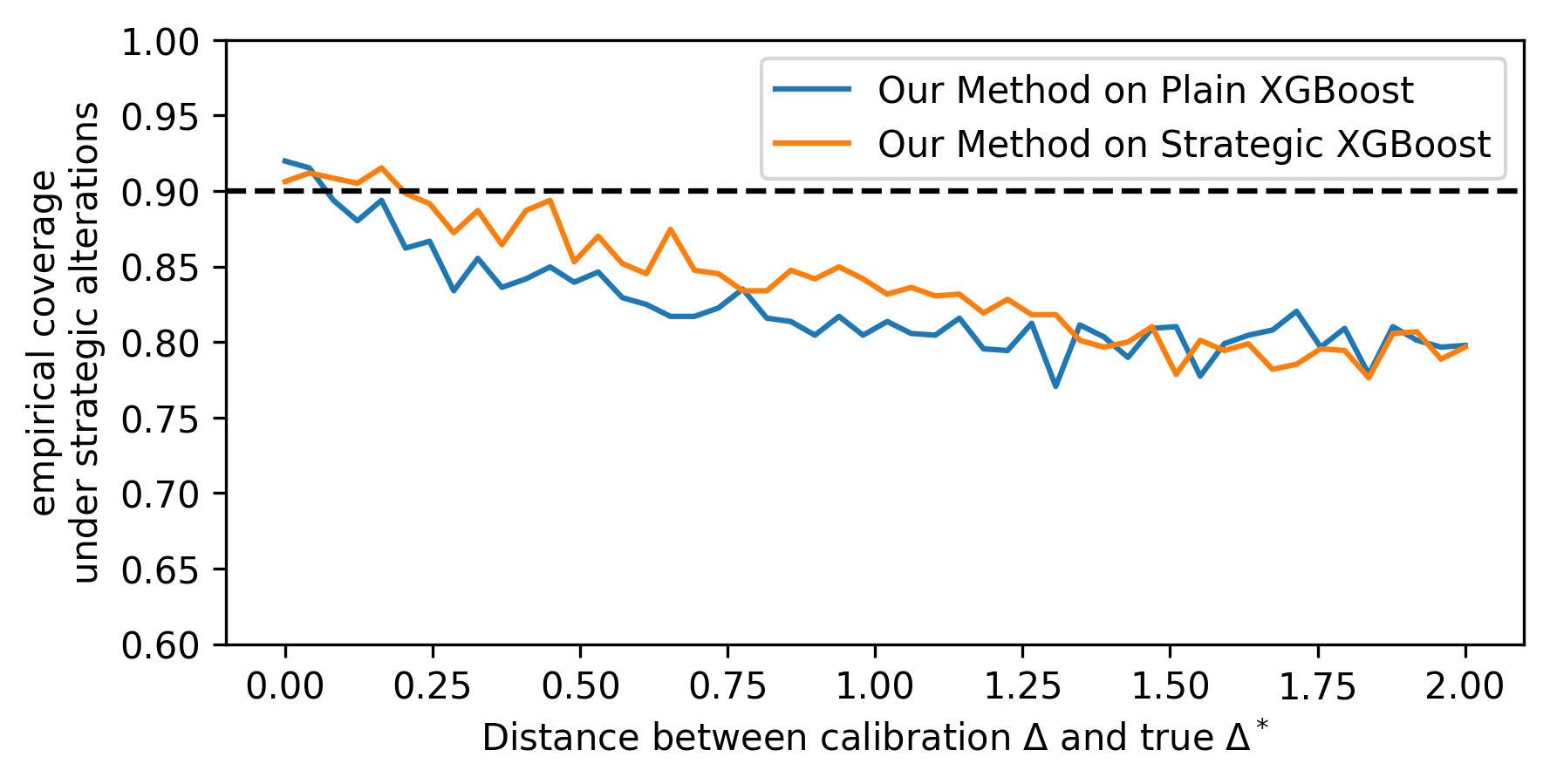}
    \caption{\texttt{academic-dropout}}
\end{figure}

\begin{figure}[H]
    \centering
    \includegraphics[width=0.5\columnwidth]{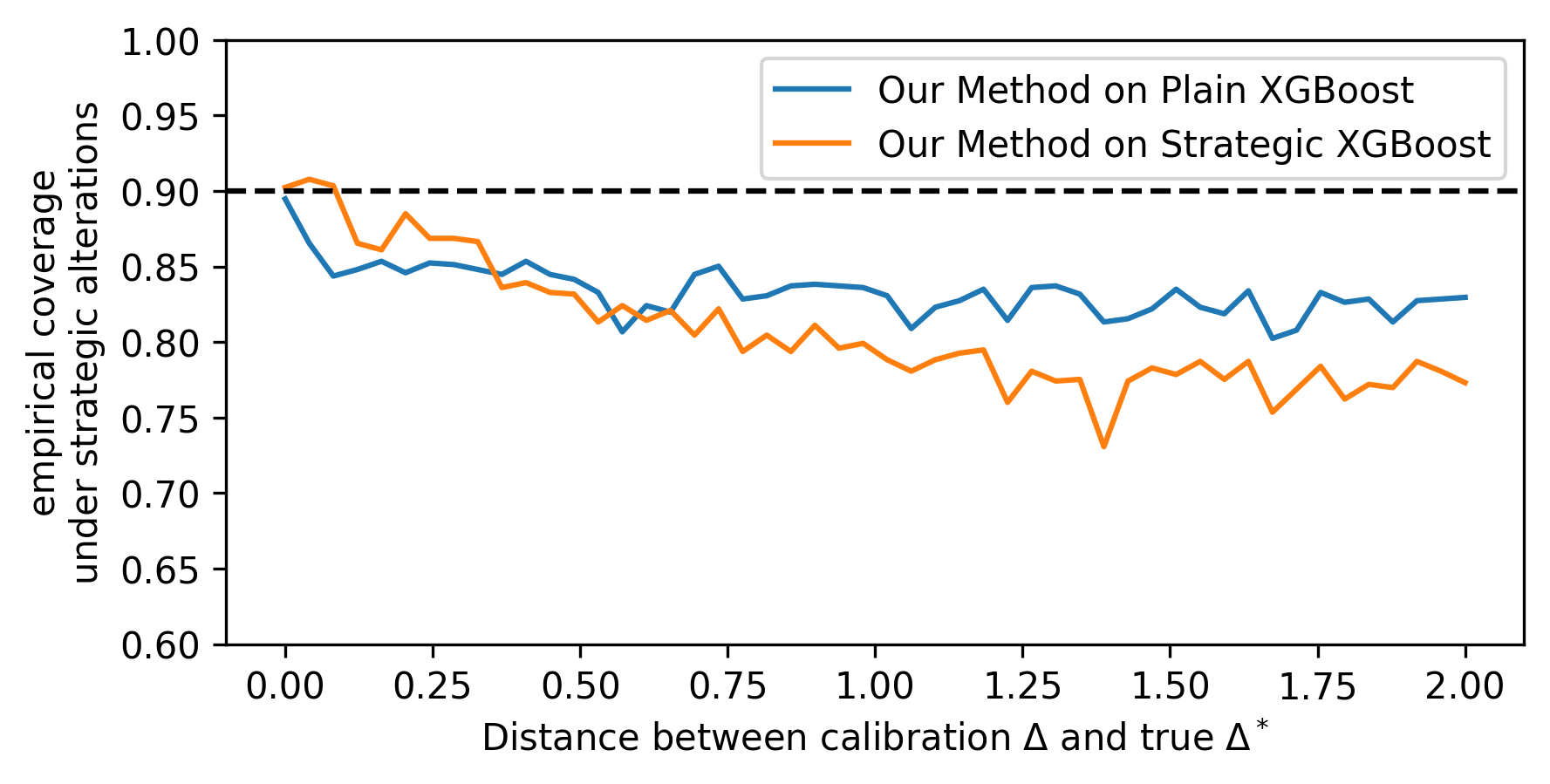}
    \caption{\texttt{spambase}}
\end{figure}

\begin{figure}[H]
    \centering
    \includegraphics[width=0.5\columnwidth]{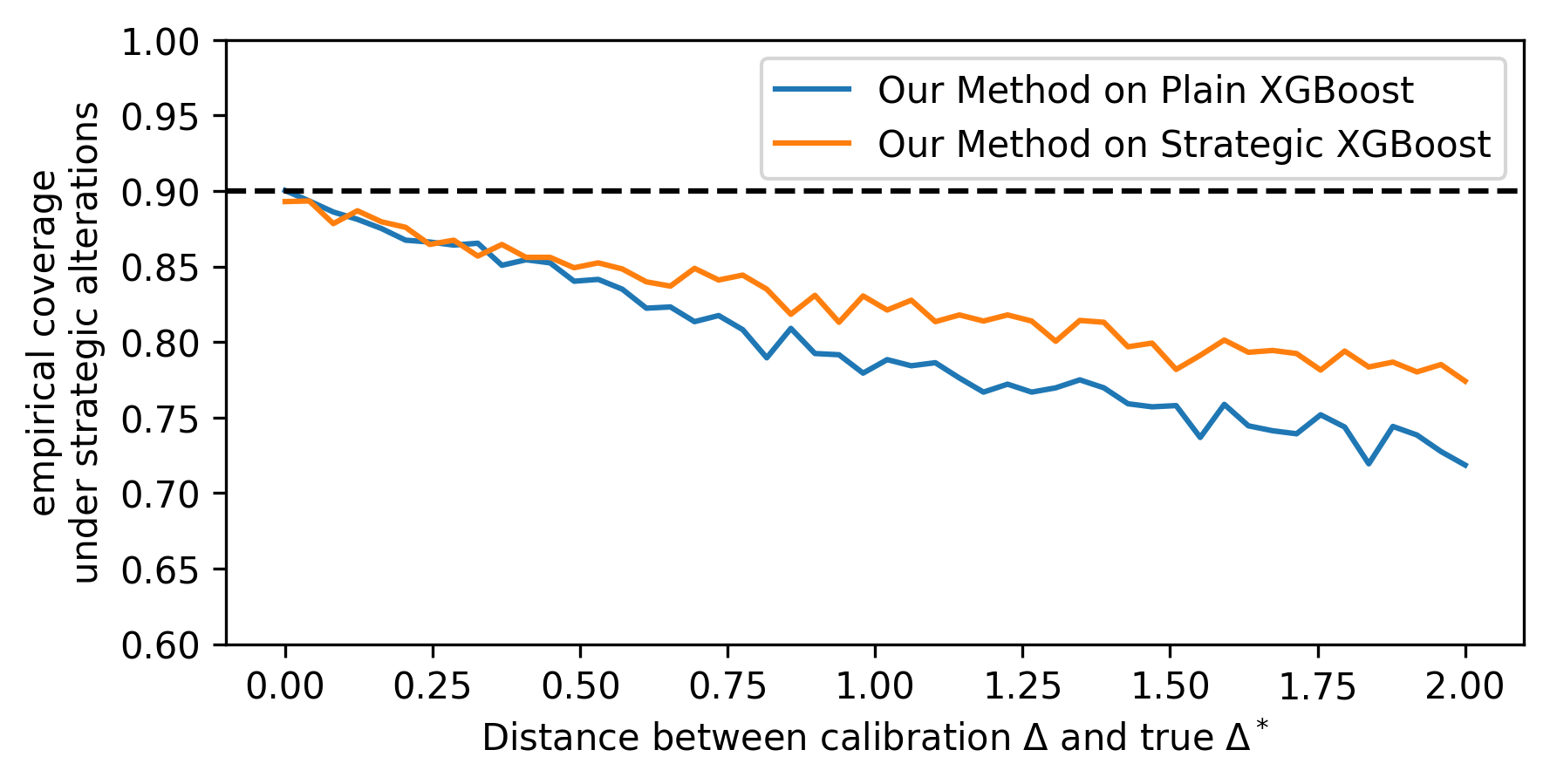}
    \caption{\texttt{shoppers}}
\end{figure}

\begin{figure}[H]
    \centering
    \includegraphics[width=0.5\columnwidth]{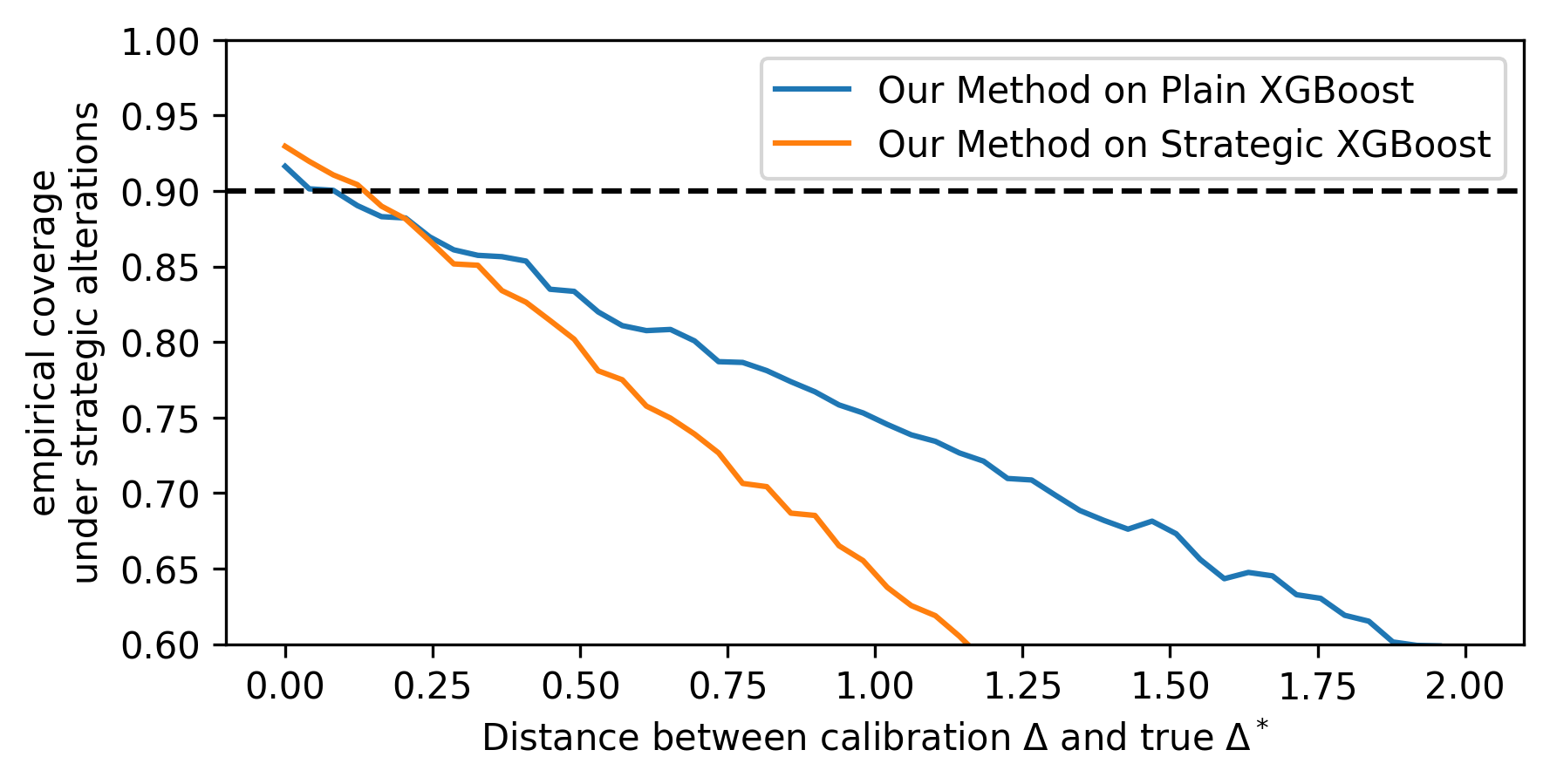}
    \caption{\texttt{news}}
\end{figure}

\begin{figure}[H]
    \centering
    \includegraphics[width=0.5\columnwidth]{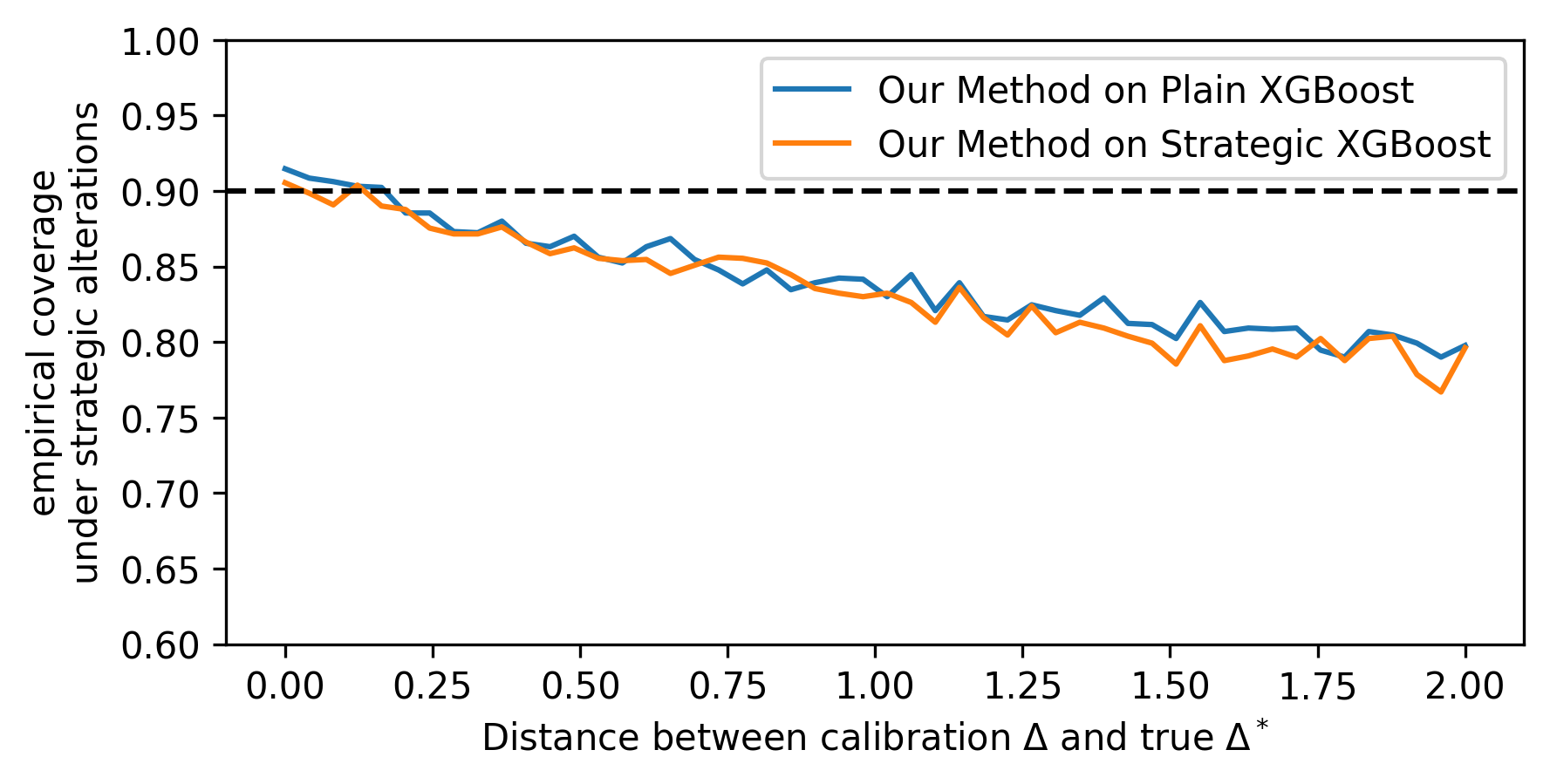}
    \caption{\texttt{wine}}
\end{figure}

\begin{figure}[H]
    \centering
    \includegraphics[width=0.5\columnwidth]{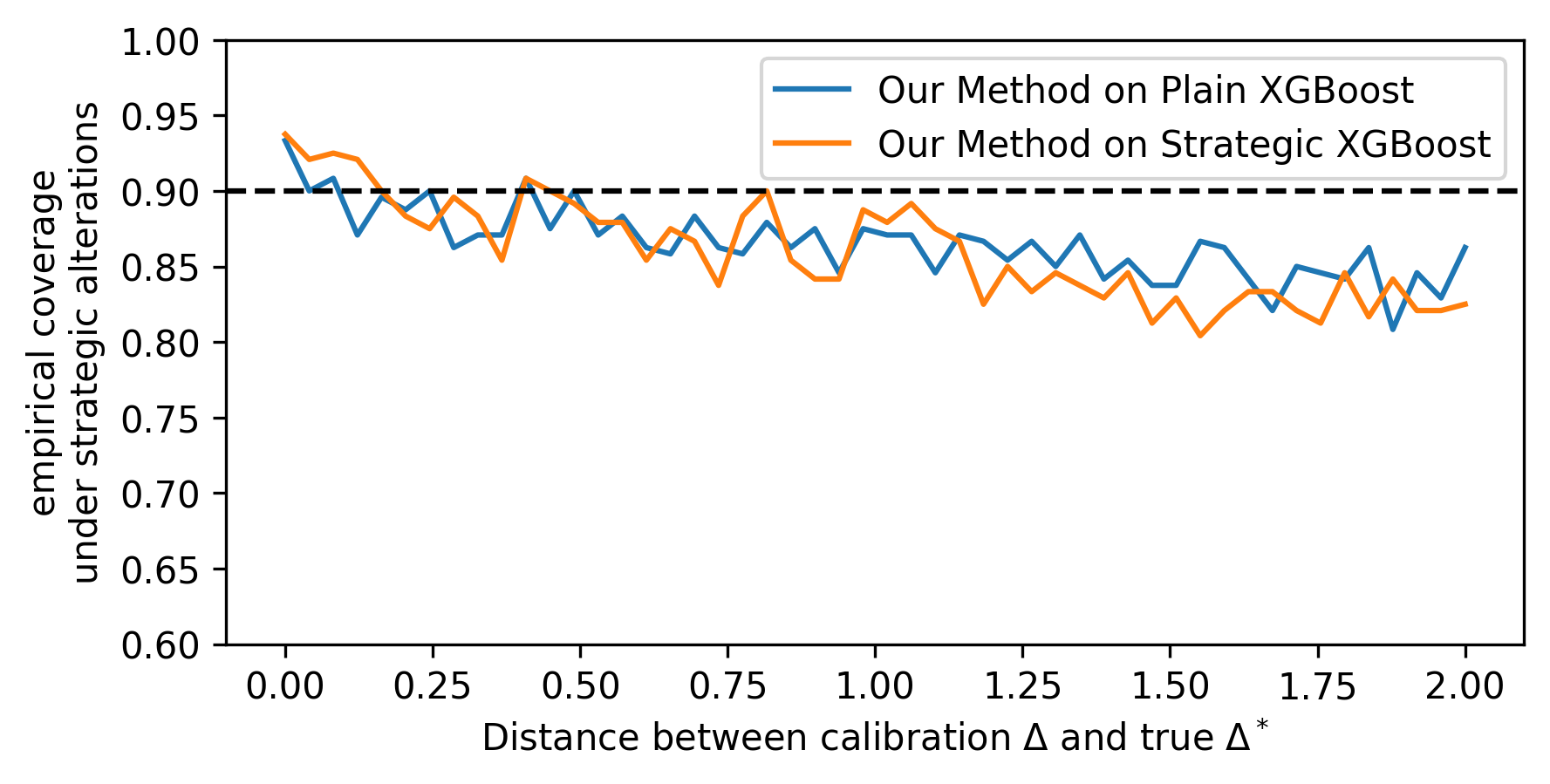}
    \caption{\texttt{productivity}}
\end{figure}

\begin{figure}[H]
    \centering
    \includegraphics[width=0.5\columnwidth]{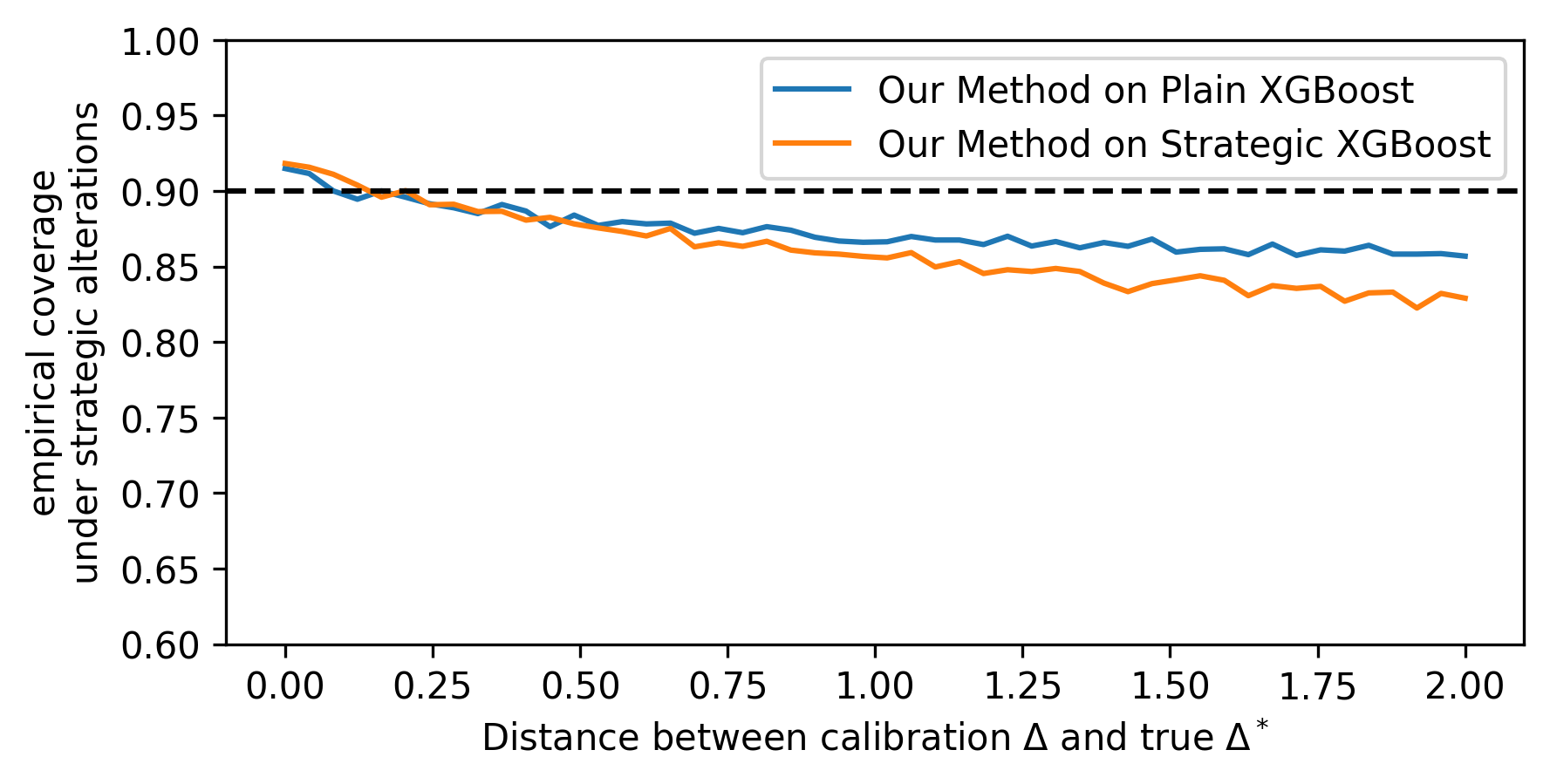}
    \caption{\texttt{taiwan}}
\end{figure}

\begin{figure}[H]
    \centering
    \includegraphics[width=0.5\columnwidth]{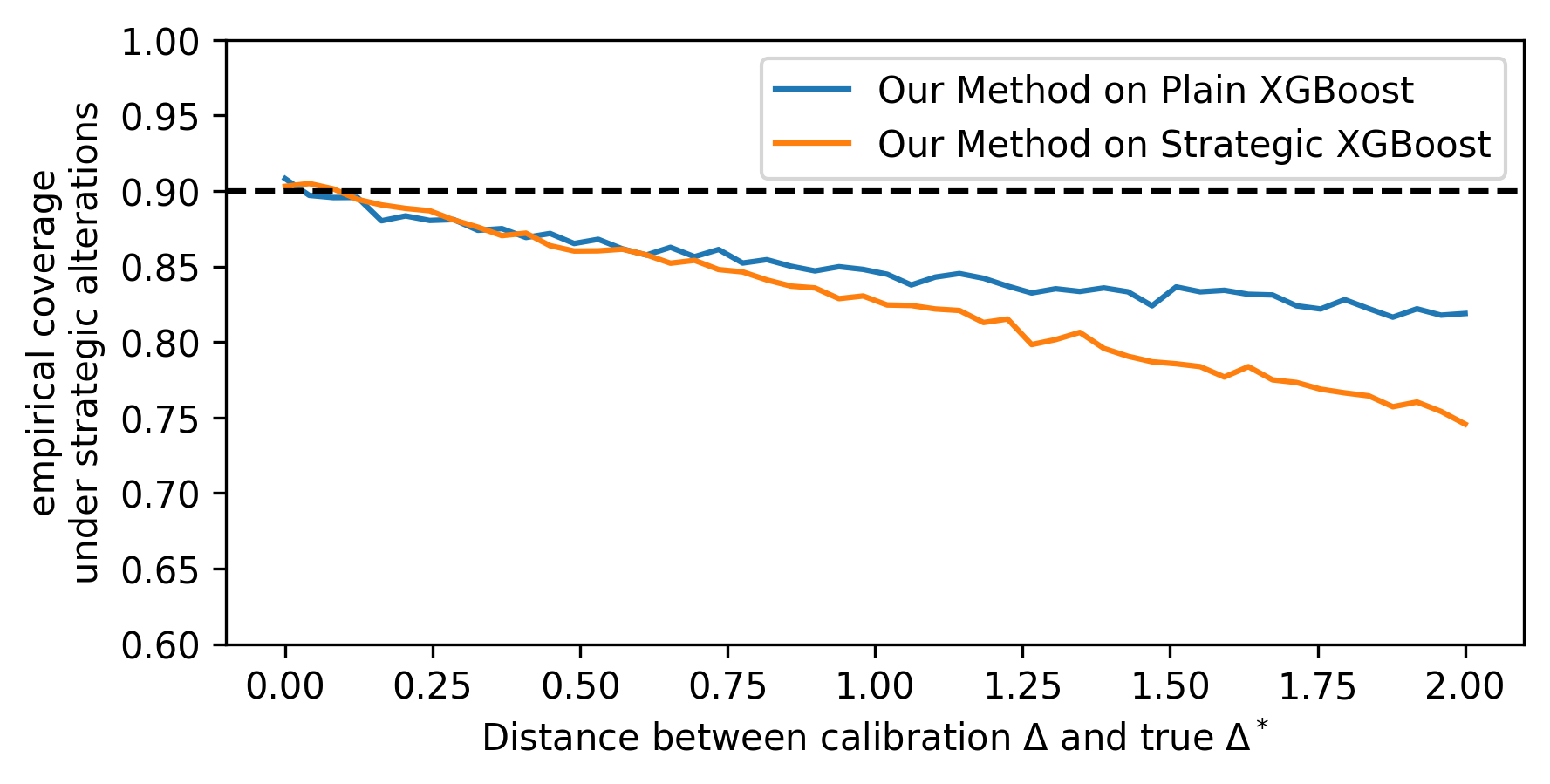}
    \caption{\texttt{bank-marketing}}
\end{figure}

\begin{figure}[H]
    \centering
    \includegraphics[width=0.5\columnwidth]{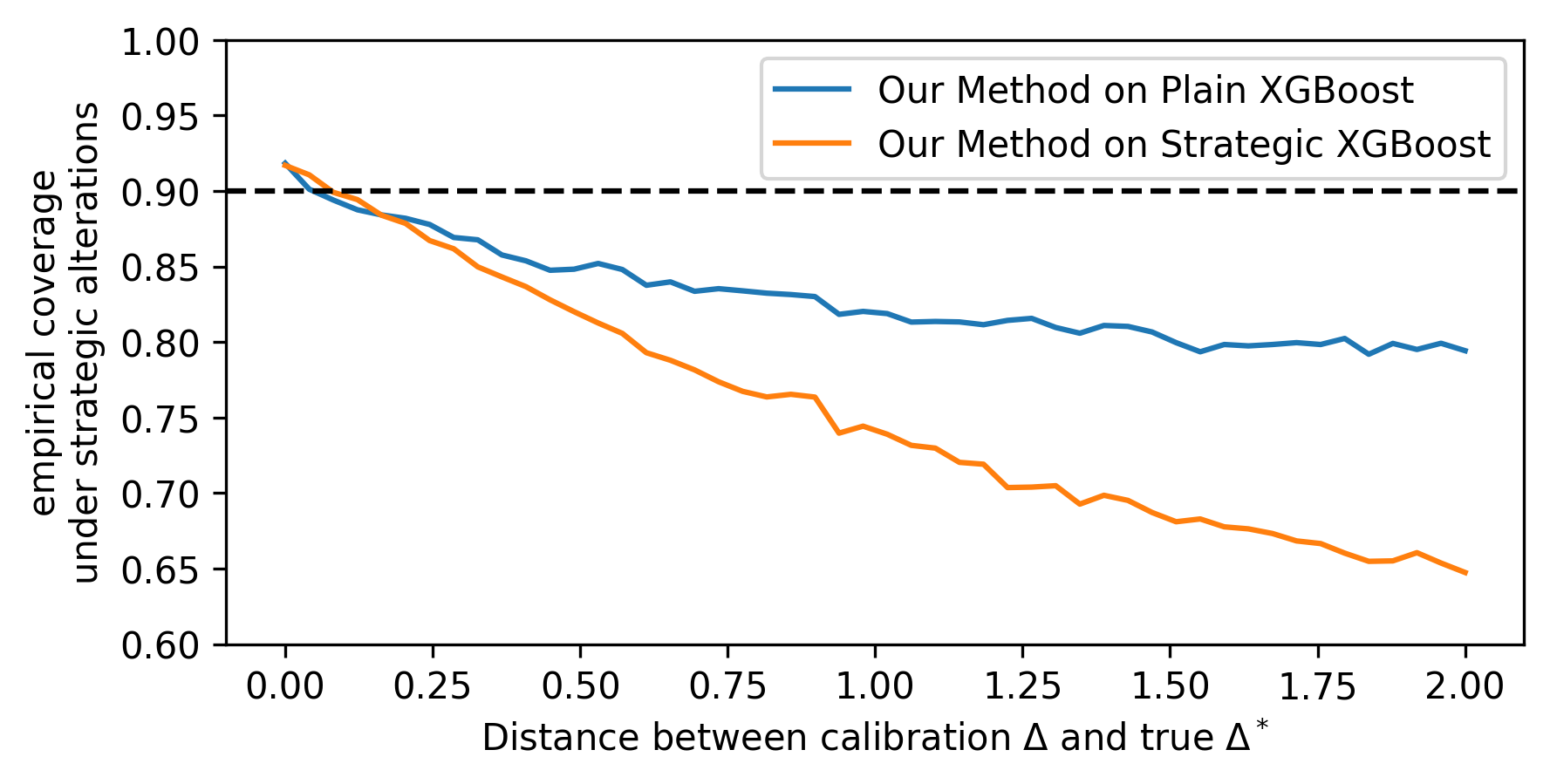}
    \caption{\texttt{census-income}}
\end{figure}

\end{document}